\documentclass{article}

\usepackage[nonatbib, final]{neurips_2022}

\usepackage[url=false, backend=biber, style=numeric-comp, maxcitenames=2, giveninits]{biblatex}
\DeclareNameAlias{sortname}{family-given}
\usepackage[hidelinks]{hyperref}
\usepackage[utf8]{inputenc} % allow utf-8 input
\usepackage[T1]{fontenc}    % use 8-bit T1 fonts
\usepackage{hyperref}       % hyperlinks
\usepackage{url}            % simple URL typesetting
\usepackage{booktabs}       % professional-quality tables
\usepackage{amsfonts}       % blackboard math symbols
\usepackage{nicefrac}       % compact symbols for 1/2, etc.
\usepackage{microtype}      % microtypography
\usepackage{xcolor}         % colors

\usepackage{bm}
\usepackage{bbm}
\usepackage{mathtools}
\usepackage{amssymb}
\usepackage{amsthm}
\usepackage{multirow}
\usepackage{subcaption}

\newtheorem{assume}{Assumption}
\newtheorem{alg}{Algorithm}
\newtheorem{lemma}{Lemma}
\newtheorem{theorem}{Theorem}
\newtheorem{prop}{Proposition}
\newtheorem{remark}{Remark}

\usepackage[shortlabels]{enumitem}

\newcommand{\iid}{\stackrel{\text{i.i.d}}{\sim}}

\DeclareMathOperator*{\argmin}{arg\,min}

\newcommand{\mcA}{\mathcal{A}}
\newcommand{\mcB}{\mathcal{B}}
\newcommand{\mcC}{\mathcal{C}}

\newcommand{\mcE}{\mathcal{E}}

\newcommand{\mcG}{\mathcal{G}}

\newcommand{\mcI}{\mathcal{I}}

\newcommand{\mcK}{\mathcal{K}}
\newcommand{\mcL}{\mathcal{L}}

\newcommand{\mcN}{\mathcal{N}}

\newcommand{\mcP}{\mathcal{P}}
\newcommand{\mcQ}{\mathcal{Q}}
\newcommand{\mcR}{\mathcal{R}}
\newcommand{\mcS}{\mathcal{S}}

\newcommand{\mcV}{\mathcal{V}}

\newcommand{\mcX}{\mathcal{X}}

\newcommand{\mcZ}{\mathcal{Z}}

\newcommand{\whomega}{\widehat{\omega}}
\newcommand{\bmomega}{\bm{\omega}_n}
\newcommand{\whbmomega}{\widehat{\bm{\omega}}_n}
\newcommand{\omegavec}{(\omega_1, \ldots, \omega_n)}
\newcommand{\whomegavec}{(\widehat{\omega}_1, \ldots, \widehat{\omega}_n)}

\newcommand{\fnone}{\tilde{f}_n(l, l', 1)}
\newcommand{\fnzero}{\tilde{f}_n(l, l', 0)}

\newcommand{\llp}{(l, l')}

\newcommand{\intsq}{\int_{[0, 1]^2}}

\newcommand{\losslipconst}{L_{\ell}}

\title{Asymptotics of $\ell_2$ Regularized Network Embeddings}

\author{%
  Andrew Davison\\
  Department of Statistics\\
  Columbia University\\
  New York, NY 10027 \\
  \texttt{ad3395@columbia.edu}
}

\begin{document}

\maketitle

\begin{abstract}
    A common approach to solving prediction tasks on large networks, such as node classification or link prediction, begins by learning a Euclidean embedding of the nodes of the network, from which traditional machine learning methods can then be applied. This includes methods such as DeepWalk and node2vec, which learn embeddings by optimizing stochastic losses formed over subsamples of the graph at each iteration of stochastic gradient descent. In this paper, we study the effects of adding an $\ell_2$ penalty of the embedding vectors to the training loss of these types of methods. We prove that, under some exchangeability assumptions on the graph, this asymptotically leads to learning a graphon with a nuclear-norm-type penalty, and give guarantees for the asymptotic distribution of the learned embedding vectors. In particular, the exact form of the penalty depends on the choice of subsampling method used as part of stochastic gradient descent. We also illustrate empirically that concatenating node covariates to $\ell_2$ regularized node2vec embeddings leads to comparable, when not superior, performance to methods which incorporate node covariates and the network structure in a non-linear manner.
\end{abstract}

\section{Introduction}
\label{sec:intro}
Network embedding methods \parencite[see e.g][]{belkin_laplacian_2003,grover_node2vec_2016,perozzi_deepwalk_2014,tang_line_2015,hamilton_inductive_2017,velickovic_deep_2018,hamilton_representation_2017,goyal_graph_2018} aim to find a latent representation of the nodes of the network within Euclidean space, in order to facilitate the solution of tasks such as node classification and link prediction, by using the produced embeddings as features for machine learning algorithms designed for Euclidean data. For example, they can be used for recommender systems in social networks, or to predict whether two proteins should be linked in a protein-protein interaction graph. Generally, such
methods obtain state-of-the-art performance on these type of tasks.

A classical approach to representation learning is to use the principal components of the Laplacian of the network \parencite{belkin_laplacian_2003}; however, this approach is computationally prohibitive for large datasets. In order to scale to large datasets, methods such as DeepWalk \parencite{perozzi_deepwalk_2014} and node2vec \parencite{grover_node2vec_2016} learn embeddings via optimizing a loss formed over stochastic subsamples of the graph. Letting $\mcG = (\mcV, \mcE)$ denote an undirected graph, and writing $\omega_u \in \mathbb{R}^d$ for the embedding of a vertex $u \in \mcV$ with
embedding dimension $d$ such that $d \ll | \mcV |$, these methods learn
embeddings by iterating over the following process: we take a random walk $(u_i)_{i \leq k}$ over the graph, form a loss 
\begin{align}
  \label{eq:n2v_loss}
  \mcL = - \sum_{i = 1}^{k} \sum_{ j \,:\, |j - i| \leq W } \log( \sigma\langle \omega_{u_i}, \omega_{u_j} \rangle ) - \sum_{i=1}^k \sum_{l = 1}^{L} \mathbb{E}_{v_l \sim \mcN(\cdot | u_i) }\Big[ \log(1 - \sigma( \langle \omega_{u_i}, \omega_{v_l} \rangle )) \Big],
\end{align}
and then perform gradient updates of the form $\omega_u \leftarrow \omega_u - \eta \nabla_{\omega_u} \mcL$ for some step size $\eta > 0$. Here $\langle \cdot, \cdot \rangle$ denotes the Euclidean inner product, $W$ is a window size, $\mcN(\cdot | v)$ a negative sampling distribution for vertex $v$ and $\sigma(y) := (1 + e^{-y})^{-1}$ the sigmoid function.

Other methods take variations on this approach - for example, LINE \parencite{tang_line_2015} replaces the random walk by sampling edges uniformly from the edge set of the graph. One can generalize \eqref{eq:n2v_loss} to capture both these cases, by considering embeddings learned via stochastic updates
\begin{align}
  \label{eq:emp_risk_step}
  \omega_u \leftarrow \omega_u - \eta \nabla_{\omega_u} \mcL \quad \text{ where } \quad 
  \mcL = \sum_{(i, j) \in \mcP } \ell_{\mcP}( \langle \omega_i, \omega_j \rangle ) + \sum_{(i, j) \in \mcN } \ell_{\mcN}( \langle \omega_i, \omega_j \rangle ),
\end{align}
$\mcP, \mcN \subseteq \mcV\times\mcV$ are random subsamples of 
the graph called
positive and negative samples respectively \parencite[see e.g][]{le-khac_contrastive_2020}, and $\ell_\mcP(\cdot)$
and $\ell_\mcN(\cdot)$ are chosen to force the embeddings of vertex pairs appearing
in $\mcP$ close together, and those in $\mcN$ far away. Usually one takes $\ell_\mcP(y) = -\log \sigma(y)$ and $\ell_{\mcN}(y) = -\log \sigma(-y)$. For example, in node2vec $\mcP$ is formed by taking concurrent edges in a random walk, and in LINE, $\mcP$ is formed by
uniform edge sampling; in both methods $\mcN$ is taken to be the same
negative sampling distribution $\mcN(u | v) \propto \mathrm{deg}(v)^{3/4}$ \parencite{mikolov_distributed_2013}. The scheme in \eqref{eq:emp_risk_step} attempts
to minimize
\begin{equation}
    \label{eq:emp_risk}
    \mcR := \sum_{i,j} \big\{ \mathbb{P}\big( (i, j) \in \mcP ) \ell_\mcP(\langle \omega_i, \omega_j \rangle ) + \mathbb{P}\big( (i, j) \in \mcN) \ell_{\mcN}(\langle \omega_i, \omega_j \rangle) \big\}
\end{equation}
obtained by averaging over the random process used to create $\mcP$ and $\mcN$ at each iteration of stochastic gradient descent \parencite{robbins_stochastic_1951,veitch_empirical_2019}; see Appendix~\ref{sec:app:sgd_erm} for a detailed derivation.

In the case where $\ell_\mcP(y) = -\log \sigma(y)$ and $\ell_{\mcN}(y) = -\log \sigma(-y)$, with $\mcP$ and $\mcN$ being random subsets of $\mcE$ and $(\mcV \times \mcV) \setminus \mcE$ (such as in LINE, or node2vec with a window length of $1$), then we can view
\eqref{eq:emp_risk} as the function obtained via trying to minimize
the negative log-likelihood (equivalent to maximizing the log-likelihood) of a probabilistic model
\begin{equation}
    \label{eq:bern_model}
    a_{uv} \,|\, \omega_u, \omega_v \sim \mathrm{Bernoulli}(\sigma( \langle \omega_u, \omega_v \rangle )) \text{ independently for $u < v$},
\end{equation}
($a_{uv} = 1$ indicating $(u, v)$ is an edge) using the stochastic gradient descent scheme specified in \eqref{eq:emp_risk_step} (see Appendix~\ref{sec:app:embed_as_graphon_learning}).
If it is assumed that the embedding vectors are drawn i.i.d from some latent
distribution, then this corresponds to implicitly fitting an exchangeable model to the graph \parencite{lovasz_large_2012} - see Appendix~\ref{sec:app:graphon} for a brief discussion on these models. Here the distribution of the
adjacency matrix $(a_{uv})_{u, v}$ is invariant to joint permutations of the vertices of
the network, or equivalently (by the Aldous-Hoover theorem \parencite{aldous_representations_1981}) arise from a probabilistic
model
\begin{equation}
    \lambda_u \sim \mathrm{Unif}([0, 1]) \text{ independently}, \qquad a_{uv} \,|\, \lambda_u, \lambda_v \sim \mathrm{Bernoulli}(W(\lambda_u, \lambda_v)) \text{ independently},
\end{equation}
for some symmetric measurable function $W: [0, 1]^2 \to [0, 1]$ called a \emph{graphon}.
We highlight that the model in \eqref{eq:bern_model} is
implicitly fitting a dense graph model to the data, \emph{even when the
observed graph data are sparse}, as is the case with real world networks.
One way of partially addressing this issue is to consider a \emph{sparsified exchangeable graph} with $W \to \rho_n W$ for some sparsifying sequence
$\rho_n \to 0$.

A natural example of a prior distribution on embedding vectors is $\omega_i \sim \mathrm{Normal}(0, (2\xi)^{-1/2} I_d)$ for some constant $\xi > 0$, so that the
contribution to the negative log-likelihood of each embedding is $ \xi \| \omega_i \|_2^2$. The full negative log-likelihood of \eqref{eq:bern_model} with
such a prior distribution is then given by
\begin{equation}
    \label{eq:bern_model_loglik}
    - \sum_{i, j} \{ a_{ij} \log \sigma( \langle \omega_i, \omega_j \rangle ) + (1 - a_{ij} ) \log \sigma( - \langle \omega_i, \omega_j \rangle) \} + \xi
    \sum_{i} \| \omega_i \|_2^2,
\end{equation}
which depends only on the matrix $G_{ij} = \langle \omega_i, \omega_j \rangle$; this loss can also arise by considering a weight decay optimization scheme 
\parencite[see e.g][]{krogh_simple_1991}.
In either case, letting $\Omega \in \mathbb{R}^{n \times d}$ be the matrix
whose rows are the $\omega_i$, so $G = \Omega \Omega^T$, we get that 
\begin{equation}
    \label{eq:intro:nuc_norm_reg}
    \sum_{i=1}^n \| \omega_i \|_2^2 = \sum_{i=1}^n \sum_{j=1}^d \omega_{ij}^2 = \sum_{i=1}^n \sum_{j=1}^d \omega_{ij} \omega_{ji} = \sum_{i=1}^n ( \Omega \Omega^T)_{ii} = \mathrm{tr}(\Omega \Omega^T) = \mathrm{tr}(G) = \| G \|_{*}
\end{equation}
where $\| G \|_{*}$ is the nuclear norm, which equals the sum of the singular values of $G$. 

As a result, we can view \eqref{eq:bern_model_loglik} as a regularized matrix factorization problem, where the nuclear-norm
penalty on the matrix $G$ is well known to shrink
the singular values of $G$ exactly towards zero, and to make
$G$ low rank \parencite[see e.g][]{bach_consistency_2008,recht_guaranteed_2010,koltchinskii_nuclear-norm_2011,chen_reduced_2013}. This consequently lowers the effective dimension of the embeddings. From a computation perspective, this is advantageous compared to treating the embedding dimension as a tunable hyperparameter, as warm-start procedures can be used to efficiently tune the regularization weight $\xi$ (and consequently the effective embedding dimension). Tuning the dimension optimality is also desirable, as generally networks have exact lower-dimensional factorizations than the embedding dimensions usually chosen in embedding models \parencite{chanpuriya_node_2020}.

In this paper, our interest is in studying the effects of such a
regularizer in the scenario where embeddings are learned via subsampling,
in which case the corresponding version of \eqref{eq:emp_risk} becomes
\begin{equation}
    \label{eq:emp_risk_reg}
    \sum_{i,j} \big\{ \mathbb{P}\big( (i, j) \in \mcP ) \ell_\mcP(\langle \omega_i, \omega_j \rangle ) + \mathbb{P}\big( (i, j) \in \mcN) \ell_{\mcN}(\langle \omega_i, \omega_j \rangle) \big\} + \sum_i \mathbb{P}\big( i \in \mcV(\mcP \cup \mcN) \big) \| \omega_i \|_2^2,
\end{equation}
where $\mcV(\mcP \cup \mcN)$ is the set of vertices which appear either
in $\mcP$ or $\mcN$ (see Appendix~\ref{sec:app:samp_reg}). The first part of the likelihood can still be thought of as corresponding to a matrix factorization term - see e.g \parencite{qiu_network_2018}. However, we note that for certain sampling schemes (e.g random walk samplers), the probability a vertex is sampled is not equiprobable across vertices, and so the regularizer will not be the same form as in \eqref{eq:intro:nuc_norm_reg}. Despite this, we still want to analyze the extent to which nuclear norm type regularization (and hence effective dimension reduction) may still arise. To do so, we study
minimizers $\whomegavec$ of \eqref{eq:emp_risk_reg} assuming the graph arises from a sparsified exchangeable graph, and obtain guarantees of the form
\begin{equation}
    \label{eq:guarantees}
    \frac{1}{n^2} \sum_{i,j} \big( \langle \whomega_i, \whomega_j \rangle - K(\lambda_i, \lambda_j) \big)^2 = o_p(1) \; \text{ and } \; \min_{Q \in O(d)} \frac{1}{n} \sum_{i} \big\| \whomega_i - \psi(\lambda_i) Q \big\|_2^2 = o_p(1)
\end{equation}
for some functions $K: [0, 1]^2 \to \mathbb{R}$ and $\psi: [0, 1] \to \mathbb{R}^d$ obtained through the process of minimizing the population objective (Section~\ref{sec:theory}). Our results allow us to recover the motivation above that the
regularization acts to reduce the effective dimension of the learned embedding
vectors. We also illustrate experimentally that using such regularization can give performance competitive to architectures
such as GraphSAGE (Section~\ref{sec:exps}).  We note that our theoretical results apply
\emph{in the regime $d \ll n$}, reflecting chosen embedding dimensions
in practice. This means that \eqref{eq:emp_risk_reg} is non-convex in the matrix $G_{ij} = \langle \omega_i, \omega_j \rangle$ due to rank constraints, which complicates the theoretical analysis (as compared to e.g \parencite{qiu_network_2018}.)

\subsection{Related works}

\textbf{Guarantees for embedding methods:} We highlight that there
is an extensive literature on the embeddings formed by the eigenvectors
of the adjacency or Laplacian matrices of
a network \parencite[e.g][]{lei_consistency_2015,rubin-delanchy_statistical_2017,tang_limit_2018,levin_limit_2021,athreya_statistical_2018,lei_network_2021}. Under various latent space
models for the network, these works
give guarantees on quantities of the form $\max_{i }\| \whomega_i - Q \psi_i\|_2^2$ for some orthogonal matrix $Q$ and vectors $\psi_i$, to  
discuss recovery of latent variables and/or obtain
exact recovery in a community detection task. Stronger bounds
are obtained in this setting as they are able to directly apply matrix eigenvector perturbation methods to study the embeddings, 
which we cannot with our approach; as a tradeoff, our approach 
allows us to study
embeddings learned via a variety of subsampling schemes,
which these works do not. We highlight that
the second bound in \eqref{eq:guarantees} can still be used to give guarantees
for weak recovery of community detection \parencite{lei_consistency_2015}.
There are a few works discussing random walk methods like node2vec, 
albeit circumventing the non-convexity in the problem; \parencite{qiu_network_2018} discusses the unconstrained minima of the loss \eqref{eq:n2v_loss} when $d = n$, with \parencite{zhang_consistency_2021} then
examining the best rank $r$ approximation to this matrix when the
generating model is a stochastic block model with $k$ communities 
and $r \leq k$. \parencite{davison_asymptotics_2021} covers the
non-convex regime $d \ll n$ in the case where $\xi_n = 0$, i.e
without regularization.

\textbf{Nuclear norm penalties and $\ell_2$ regularized embeddings:} In the context of matrix factorization (so the matrix factors are embedding matrices), the effects of
Frobenius norm penalties inducing nuclear norm penalties
are well known \parencite[e.g][]{recht_guaranteed_2007,udell_generalized_2015}; generally, there is an extensive literature on the effects of nuclear norm 
penalization in the finite-dimensional setting \parencite[e.g][]{bach_consistency_2008,recht_guaranteed_2010,koltchinskii_nuclear-norm_2011,chen_reduced_2013}. 
In \parencite{zhang_improve_2018}, it is also shown that $\ell_2$ regularized node2vec gives an improvement in performance on downstream tasks.

\textbf{Graphon estimation:} We mention that our
guarantees on the gram matrix formed by the learned embeddings are
similar to those obtained in the graphon estimation
literature \parencite[see e.g][]{borgs_private_2015-1,borgs_revealing_2018,wolfe_nonparametric_2013,gao_rate-optimal_2015,klopp_oracle_2017,latouche_variational_2016,chatterjee_matrix_2015,xu_rates_2018}. Depending
on the choice of sampling scheme and loss, it is possible for the limiting
matrix $K(\lambda_i, \lambda_j)$ to be an invertible transformation of
$W(\lambda_i, \lambda_j)$, and so we compare our rates of convergence
in such a scenario (see Remark~\ref{app:rmk:rates_of_convergence} in Appendix~\ref{sec:app:embed_conv}).
\section{Framework and assumptions of analysis}
\label{sec:framework}

Given a sequence of graphs $\mcG_n = (\mcV_n, \mcE_n)$, and writing $\bmomega = \omegavec$ with $\omega_u \in \mathbb{R}^d$ denoting the $d$-dimensional embedding
of vertex $u$, we study the regularized empirical risk function
\begin{equation}
    \label{eq:framework:emp_risk} 
\begin{aligned}
  \mcR_n(\bm{\omega}_n) + \xi_n \mcR^{\text{reg}}_n(\bm{\omega}_n) \;\text{ where }\;
  \mcR_n(\bm{\omega}_n) & := \sum_{\substack{i, j \in [n] \\ i \neq j }} \mathbb{P}\big( (i, j) \in S(\mcG_n) \,|\, \mcG_n \big) \ell( \langle \omega_i, \omega_j \rangle, a_{ij}), \\
  \mcR_n^{\text{reg}}(\bm{\omega}_n) & := \sum_{i \in [n] } \mathbb{P}\big( i \in \mcV(S(\mcG_n)) \,|\, \mcG_n \big) \| \omega_i \|_2^2,
\end{aligned}
\end{equation}
for $\xi_n \geq 0$. Here, we define a subsample $S(\mcG)$ of a graph $\mcG$ as a collection of vertices $\mcV(S(\mcG))$ and a symmetric subset of $\mcV(S(\mcG)) \times \mcV(S(\mcG))$. The sampling probabilities are conditional on $\mcG_n$ as we will soon assume that the $\mcG_n$ also arise from a probabilistic model. We note \eqref{eq:framework:emp_risk} arises from \eqref{eq:emp_risk_reg} whenever
$\mcP$ and $\mcN$ are random subsets of $\mcE_n$ and $\mcV \times \mcV \setminus \mcE_n$ respectively (see Appendix~\ref{sec:app:get_to_the_risk}). 

Throughout, $\ell(y, x)$ is either the cross-entropy loss
$\ell_{\sigma}(y, x) := -x \log\sigma(y) - (1-x) \log\sigma(-y)$ or the squared
loss $\ell_2(y, x) := (y-x)^2$. We now discuss our assumptions on the
generative model of the graph. Recall in
the introduction we argued embedding methods are implicitly fitting an
exchangeable graph model; consequently, as a first step to analysis, we will assume that the graph arises from a sparsified graphon (to account for the sparsity in graphs observed in the real world).

\begin{assume}
    \label{assume:graphon}
    We assume that the sequence of graphs $(\mcG_n)_{n \geq 1}$ have
    vertex sets $\mcV_n = [n]$ and arise
    from a graphon process with generating graphon $W_n = \rho_n W$ with $\rho_n \gg \log(n)/n$, so that
    \begin{equation*}
        \lambda_i \iid \mathrm{Unif}[0, 1] \text{ for } i \in [n], \quad 
        a_{ij} | \lambda_i, \lambda_j \stackrel{\text{indep.}}{\sim} \mathrm{Bernoulli}(\rho_n W(\lambda_i, \lambda_j)) \text{ for } i < j.
    \end{equation*}
    We moreover suppose that $W \in [c, 1-c]$ for some $c > 0$, and either a) $W$ is piecewise constant on a partition $\mcQ \times \mcQ$ where $\mcQ$ is a partition of $[0, 1]$ of size $\kappa$ (so the model is a stochastic block model), or b) $W$ is $\text{H\"{o}lder}([0, 1]^2, \beta_W, L_W)$ for some
    exponents and constants $\beta_W \in (0, 1]$ and $L_W < \infty$.
\end{assume}

We provide an introduction to graphon models in Appendix~\ref{sec:app:graphon}.

\begin{remark}
    \label{rmk:sparse}
    Here we use the canonical choice of uniform latent variables for the graphon as
    guaranteed by the Aldous-Hoover theorem for vertex exchangeable graphs \parencite{aldous_representations_1981}; in principle, 
    our results can extend to graphons on higher dimensional 
    latent spaces by the same style of arguments \parencite[see e.g][]{davison_asymptotics_2021}. We highlight that our assumptions
    are somewhat restrictive with regards to the boundedness and  
    and sparsity conditions; while
    it is common to allow $\rho_n \gg \log(n)/n$ \parencite[e.g][]{wolfe_nonparametric_2013,oono_graph_2021} - in which case the degree structure is regular, not necessarily realistic of real world networks - 
    it is also common to work in the regime where $\rho_n = \Theta(\log(n)/n)$ or smaller \parencite[e.g][]{borgs_revealing_2018,xu_rates_2018}. We
    highlight that in general graphons, regardless
    of the sparsity factor, tend to not give rise to graphs
    with power-law or heavy-tailed type degree structures, which is frequent
    with real world networks \parencite{broido_scale-free_2019,zhou_power-law_2020}.
\end{remark}

For subsampling schemes used in practice (such as random
walk and uniform edge samplers), the sampling probability
of vertices and edges depends only on local features of the graph.
Following \parencite{davison_asymptotics_2021}, we can formalize
this as follows:

\begin{assume}
  \label{assume:slc}
  There exist sequences of measurable functions $(f_n)_{n \geq 1}$ and $(\tilde{g}_n)_{n \geq 1}$ and a sequence $s_n = o(1)$, such that
  \begin{gather*}
    \max_{\substack{i, j \in [n] \\ i \neq j}} \Big| \frac{ n^2 \mathbb{P}\big( (i, j) \in S(\mcG_n) \,|\, \mcG_n\big)  }{ f_n(\lambda_i, \lambda_j, a_{ij} ) } - 1 \Big| = O_p(s_n), \; \max_{i \in [n]} \Big| \frac{ n \mathbb{P}\big( i \in \mcV(S(\mcG_n)) \,|\, \mcG_n) }{ \tilde{g}_n(\lambda_i ) } - 1 \Big| = O_p(s_n),
  \end{gather*}
  and moreover $\mathbb{E}[f_n(\lambda_1, \lambda_2, a_{12} ) ] = O(1)$, $\mathbb{E}[f_n(\lambda_1, \lambda_2, a_{12})^2] = O(\rho_n^{-1})$, $\mathbb{E}[\tilde{g}_n(\lambda_1)] = O_p(1)$.
\end{assume}

This assumption allows us to replace the sampling probabilities in the empirical risk by functions which depend on the latent variables and edge assignments in the model, from which the exchangeability in the model can be used to allow for a large sample analysis. Examples of sampling schemes satisfying this condition are given in
Section~\ref{sec:samp}. We additionally impose some regularity
conditions on the "averaged" versions of the above functions defined by
\begin{align}
    \fnone := f_n(l, l', 1) W_n(l, l'), \qquad \fnzero := f_n(l, l', 0) (1 - W_n(l, l')).
\end{align} 

\begin{assume}
  \label{assume:reg}
  We assume the the functions $\fnone$, $\fnzero$ and $\tilde{g}_n(l)$
  are uniformly bounded above by $M$ and away from zero by $M^{-1}$ for
  some constant $M \in (0, \infty)$. We also suppose that either a) 
  there exists a partition $\mcQ$ of $[0, 1]$ into $\kappa$ parts such that
  $\fnone$ and $\fnzero$ are piecewise constant on $\mcQ \times \mcQ$, and
  $\tilde{g}_n(l)$ is piecewise constant on $\mcQ$; or b) the $\fnone$ 
  and $\fnzero$ are H\"{o}lder$([0, 1]^2$, $\beta$, $L)$, and that
  the $\tilde{g}_n(l)$ are H\"{o}lder$([0, 1]$, $\beta$, $L)$.
\end{assume}

Assumption~\ref{assume:reg} will follow as a consequence of Assumptions~\ref{assume:graphon}~and~\ref{assume:slc}
for the sampling schemes discussed in Section~\ref{sec:samp}, with
$\beta$ depending on $\beta_W$ and the hyper-parameters of the sampling
scheme.

\section{Theoretical results}
\label{sec:theory}

Our theoretical results take the following flavour: we identify
the correct population versions (in reference to an infinite graphon on a vertex set $\mathbb{N}$) of $\mcR_n(\bmomega)$ and 
$\mcR_n^{\text{reg}}(\bmomega)$ in the large sample limit $n \to \infty$
to give a regularized population risk,
and then use this to give guarantees about any minimizers of
$\mcR_n(\bmomega) + \xi_n \mcR_n^{\text{reg}}(\bmomega)$ being
close (in some sense) to the unique minimizer of the regularized population
risk. As network embedding methods are used on very large networks,
such a large sample statistical analysis is appropriate. 

We introduce the population versions of $\mcR_n(\bmomega)$ and $\mcR_n^{\text{reg}}(\bmomega)$ respectively as
\begin{equation}
    \mcI_n[K] := \intsq \sum_{x \in \{0, 1\}} \tilde{f}_n(l, l', x) \ell(K(l, l'), x) \, dl dl', \qquad \mcI_n^{\text{reg}}[K] := \int_0^1 K(l, l) \tilde{g}_n(l) \, dl,
\end{equation}
defined over functions $K: [0, 1]^2 \to \mathbb{R}$. $\mcI_n[K]$ was first introduced in \parencite{davison_asymptotics_2021}.
The formula given for $\mcI_n^{\text{reg}}[K]$ holds only for $K$ whose
diagonal is well defined; in general, if $K$ admits a decomposition
\begin{equation}
    \label{eq:K_comp}
    K\llp = \sum_{i = 1}^{\infty} \mu_i(K) \psi_i(l) \psi_i(l') \quad \text{ where } \quad \int_0^1 \psi_i(l) \psi_j(l) \tilde{g}_n(l) \, dl = \begin{cases} 1 & \text{ if } i = j, \\ 0 & \text{ otherwise,} \end{cases}
\end{equation}
and $\mu_i(K) \geq 0$ for all $i$ (understood as a limit in $L^2([0, 1]^2)$), then we can extend the definition of the regularizer to be $\mcI_n^{\text{reg}}[K] := \sum_{i = 1}^{\infty} \mu_i(K)$. Consequently, the penalty corresponds to the trace of $K$,
when viewed as the kernel of a Hilbert-Schmidt operator.
This means that $\mcI_n^{\text{reg}}[K]$ should be
viewed as a nuclear-norm penalty on the kernel $K$, which encourages the $\mu_i(K)$ to be shrunk exactly towards zero, similar to the finite 
dimensional scenario; also see e.g. Theorem~\ref{thm:calc_minima}.

\subsection{Guarantees on the learned embedding vectors}
\label{sec:theory:embed}

We begin
with a result which guarantees that $\mcI_n[K] + \xi_n \mcI_n^{\text{reg}}[K]$,
once restricted to an appropriate domain, is the correct population version of $\mcR_n(\bmomega) + \xi_n \mcR_n^{\text{reg}}(\bmomega)$.

\begin{theorem}
  \label{thm:loss_conv}
  Suppose that Assumptions~\ref{assume:slc}~and~\ref{assume:reg} hold, and that $\xi_n = O(1)$. Define
  \begin{equation*}
    \mcZ_d^{\geq 0}(A) := \big\{ K \,:\, K\llp = \langle \eta(l), \eta(l') \rangle, \eta: [0, 1] \to [-A, A]^d \big\} \text{ for } d \in \mathbb{N}, A > 0.
  \end{equation*}
  Then we have that 
  \begin{align*}
    \Big| & \min_{ \bmomega \in ([-A, A]^d)^n } \big\{ \mcR_n(\bmomega) + \xi_n \mcR_n^{\text{reg}}(\bmomega) \big\} - \min_{K \in \mcZ_d^{\geq 0}(A) } \big\{ \mcI_n[K] + \xi_n \mcI_n^{\text{reg}}[K] \big\} \Big| = O_p(r_n)
  \end{align*}
  where $r_n = s_n + (d^p/n\rho_n)^{1/2} + t_n$, with $t_n = (\log \kappa/n)^{1/2}$ under Assumption~\ref{assume:reg}a), $t_n = (\log(n)/n^{2\beta/(1+2\beta)})^{1/2}$ under Assumption~\ref{assume:reg}b), $p = 3$ for the cross-entropy
  loss and $p = 5$ for the squared loss.
\end{theorem}

See Appendix~\ref{sec:app:loss_conv} for the proof and a discussion of the rates given; we note that it is necessary that $d \ll n$ in order for $r_n \to 0$. The rates can be improved to give $p = 1$, under additional restrictions on the parameter space (Remark~\ref{app:rmk:tighter_bounds}). The interpretation of the set $\mcZ_d^{\geq 0}(A)$ is as follows: if $\omega_u \in [-A, A]^d$
is the embedding of vertex $u$, in the population limit $\eta(\lambda_u) \in [-A, A]^d$
should give the embedding of a vertex with latent feature $\lambda_u$.
As \eqref{eq:framework:emp_risk} is parameterized through terms of the form $\langle \omega_u, \omega_v \rangle$, the population version of \eqref{eq:framework:emp_risk} should be parameterized through functions 
$K(\lambda_u, \lambda_v) = \langle \eta(\lambda_u), \eta(\lambda_v) \rangle$.

\begin{remark}
    We note that the assumption that the embedding vectors belong to a hypercube $[-A, A]^d$ is not restrictive; for example, in practice
    embedding vectors are usually initialized randomly and
    uniformly over $[-1, 1]^d$. Moreover, our results allow for $A$ to
    grow logarithmically with $n$, and only change the bound by a poly-log
    factor. We highlight that if $\xi_n \to \infty$ as $n \to \infty$, then the embedding vectors will shrink towards $0$ as $n \to \infty$ (as seen in Figure~\ref{fig:tsne_cora}). This is because the proof of Theorem~\ref{thm:loss_conv} shows that any minimizer must satisfy $n^{-1} \sum_{i=1}^n \| \omega_i \|_2^2 = O_p(\xi_n^{-1}) = o_p(1)$ in such a regime.
\end{remark}

We now give convergence guarantees for any sequence of embedding vectors
minimizing \eqref{eq:framework:emp_risk}.

\begin{theorem}
  \label{thm:embed_conv}
  Suppose that Assumptions~\ref{assume:slc}~and~\ref{assume:reg} hold and that $\xi_n = O(1)$. Then for each $n$, there exists a unique minimizer to the optimization problem
  \begin{equation*}
    K_n^* = \argmin_{K \in \mcZ^{\geq 0} } \big\{ \mcI_n[K] + \xi_n \mcI_n^{\text{reg}}[K] \big\} \quad \text{ where } \quad \mcZ^{\geq 0} := \mathrm{cl}\Big( \bigcup_{d \geq 1} \mcZ_d^{\geq 0}(A) \Big)
  \end{equation*}
  is free of $A > 0$ (see Proposition~\ref{app:min_exist} for details). Moreover, under some regularity conditions on the $K_n^*$ (see Theorem~\ref{app:thm:embed_conv_full} in Appendix~\ref{sec:app:embed_conv}), there exists $A' < \infty$ free of $n$ and a sequence of embedding dimensions $d = d(n) \ll n$ such that, for any sequence of minimizers
  \begin{equation*}
    \label{eq:embed_conv:embeds}
    \whbmomega \in \argmin_{\bmomega \in ([-A_1,A_1]^d)^n} \big\{ \mcR_n(\bmomega) + \xi_n \mcR_n^{\text{reg} }(\bmomega) \big\} \text{ satisfying } \max_{i,j} | \langle \whomega_i, \whomega_j \rangle | \leq A_2
  \end{equation*}
  with $A_1, A_2 \geq A'$, we have that 
  \begin{equation*}
    \frac{1}{n^2} \sum_{i, j \in [n]} \big( \langle \whomega_i, \whomega_j \rangle - K_n^*(\lambda_i, \lambda_j) \big)^2 = o_p(1).
  \end{equation*}
  Moreover, when Assumption~\ref{assume:reg}a) holds, $K_n^*$
  can be computed via a finite dimensional convex program, and is of rank $r \leq \kappa$, in that an expansion of the form \eqref{eq:K_comp} holds
  with $\mu_i(K_n^*) = 0$ for $i > r$.
\end{theorem}

The case where $\xi_n = 0$ is proven in \parencite{davison_asymptotics_2021}, which also verifies the convergence on simulated data. A proof is 
given in Appendix~\ref{sec:app:embed_conv}. Under certain
circumstances, this allows us to give guarantees about the
distribution of the embedding vectors themselves.

\begin{theorem}
    \label{thm:procrustes}
    Suppose that $K_n^*$ is regular in the sense of Theorem~\ref{thm:embed_conv}, with the conclusions of the theorem holding. Moreover suppose that $K_n^*$ is a kernel of rank $r < \infty$, has a decomposition of the form $K_n^*(l, l') = \sum_{i=1}^r \phi_{n,i}(l) \phi_{n,i}(l')$
    for some functions $\phi_{n,i}: [0, 1] \to \mathbb{R}$, and the dimension $d$
    of the embedding vectors is chosen to be equal to $r$. Writing $\phi_n(l) = (\phi_{n,i}(l))_{i=1}^r$, we have
    \begin{equation}
        \label{eq:procrustes}
        \min_{Q \in O(r)} \frac{1}{n} \sum_{i =1}^n \big\| \whomega_i - Q \phi_{n}(\lambda_i) \big\|_2^2 = o_p(1).
    \end{equation}
\end{theorem}

The assumption that $d = r$ in Theorem~\ref{thm:procrustes} is a 
restrictive one, given
that embedding dimensions in practice are usually chosen to be one of
either 128, 256 or 512. This can be alleviated by instead
giving a guarantee for the optimal $r$ dimensional projection of the embedding vectors. 

\begin{theorem}
    \label{thm:procrustes_2}
    If instead $d > r$ in Theorem~\ref{thm:procrustes}, let $\widetilde{G} \in \mathbb{R}^{n \times n}$ denote the best rank $r$ approximation to 
    the matrix $G_{ij} := (\langle \whomega_i, \whomega_j \rangle)_{ij}$,
    and write $\widetilde{G} = \tilde{\Omega} \tilde{\Omega}^T$ for some $\tilde{\Omega} \in
    \mathbb{R}^{n \times r}$. Then $n^{-2} \| \widetilde{G} - G \|_F^2 = o_p(1)$, and writing 
    $\widetilde{\omega}_i$ for the rows of $\widetilde{\Omega}$,
    the guarantee in \eqref{eq:procrustes} holds with $\widetilde{\omega}_i$
    replacing the $\whomega_i$.
\end{theorem}

Informally, this says that the embedding vectors approximately lie
on a $r$-dimensional subspace which contains some latent information
about the network, depending on the minimizing kernel $K_n^*$. 
We now highlight that the assumption that
$K_n^*$ is of finite rank $r < \infty$ is not restrictive
even when $W$ is not a SBM; this is a
consequence of the effect of 
the regularization penalty $\mcI_n^{\text{reg}}[K]$.

\begin{theorem}
    \label{thm:calc_minima}
    Let $\ell(y, x)$ be the squared loss, and suppose that $\rho_n = 1$, $\tilde{f}_n(l, l', 1) = \tilde{f}_n(l, l', 0) = c_1$ and $\tilde{g}_n(l) = c_2$ for some
    constants $c_1, c_2 > 0$ (see e.g. Algorithm~\ref{alg:psamp}
    in Section~\ref{sec:samp}). Then if $W$ is H\"{o}lder$([0, 1]^2, \beta, L)$, the minima of $\mcI_n[K] + \xi_n \mcI_n^{\text{reg}}[K]$ 
    over $\mcZ^{\geq 0}$ is of finite
    rank for any $\xi_n > 0$, and is also H\"{o}lder continuous of exponent $\beta$.
\end{theorem}

We highlight that this result also shows that the regularizer acts to shrink the singular values of a minimizer of $\mcI_n[K] + \xi_n \mcI_n^{\text{reg}}[K]$ exactly towards zero. See Appendix~\ref{sec:app:extra} for proofs of Theorems~\ref{thm:procrustes},~\ref{thm:procrustes_2}~and~\ref{thm:calc_minima}. 

\subsection{Sampling schemes satisfying Assumption~\ref{assume:slc}}
\label{sec:samp}

We now discuss some examples of frequently used sampling schemes which satisfy
Assumption~\ref{assume:slc}. Proofs of the results in this section
can be found in Appendix~\ref{sec:app:samp_proof}. We introduce the notation
\begin{equation}
    W(\lambda, \cdot) := \int_0^1 W(\lambda, y)\, dy, \quad 
    \mcE_W(\alpha) := \int_0^1 W(\lambda, \cdot)^{\alpha} \, d \lambda, \quad
    \mcE_W := \mcE_W(1). 
\end{equation}

\begin{alg}[Uniform vertex sampling] 
  \label{alg:psamp}
  Given a graph $\mcG_n$ and number of samples $k$, we select $k$ vertices from $\mcG_n$ uniformly and without replacement, and then return $S(\mcG_n)$ as the induced subgraph using these sampled vertices.
\end{alg}

\begin{lemma} \label{sampling:psamp_formula}
  For Algorithm~\ref{alg:psamp}, Assumption~\ref{assume:slc} holds with $     f_n(\lambda_i, \lambda_j, a_{ij} ) = k(k-1)$, $\tilde{g}_n(\lambda_i) = k$
  and $s_n = 1/n$.
\end{lemma}

\begin{alg}[Uniform edge sampling \parencite{tang_line_2015}]
  \label{alg:unifedge+ns}
  Given a graph $\mcG_n$, number of edges to sample $k$, and number of negative samples $l$ per positive sample,
  \begin{enumerate}[label=\roman*), leftmargin=*]
      \item We form a sample $S_0(\mcG_n)$ by sampling $k$ edges from $\mcG_n$ uniformly and without replacement;
      \item We form a sample set of negative samples $S_{ns}(\mcG_n)$ by drawing, for each $u \in \mcV(S_0(\mcG_n))$, $l$ vertices $v_1, \ldots, v_l$ i.i.d according to the unigram distribution 
      \begin{equation*}
          \mathrm{Ug}_{\alpha}\big( v \,|\, \mcG_n \big) \propto \mathbb{P}\big( v \in \mcV(S_0(\mcG_n)) \,|\, \mcG_n )^{\alpha} 
      \end{equation*}
      and then adjoining $(u, v_i) \to S_{ns}(\mcG_n)$ if $a_{u v_i} = 0$.
  \end{enumerate}
  We then return $S(\mcG_n)$ as the union of $S_0(\mcG_n)$ and $S_{ns}(\mcG_n)$. 
\end{alg}

\begin{lemma} \label{sampling:unif_edge_uni_formula}
  For Algorithm~\ref{alg:unifedge+ns}, Assumption~\ref{assume:slc} holds with
  $s_n = ( \log(n) / n \rho_n )^{-1/2}$,
  \begin{gather*}
      f_n(\lambda_i, \lambda_j, 1 ) = \frac{2k}{\mcE_W \rho_n}, \; f_n(\lambda_i, \lambda_j, 0) = \frac{ 2k l }{ \mcE_W \mcE_W(\alpha) } \big\{ W(\lambda_i, \cdot) W(\lambda_j, \cdot)^{\alpha} + W(\lambda_i, \cdot)^{\alpha} W(\lambda_j, \cdot) \big\}, \\
    \tilde{g}_n(\lambda_i) = \frac{2k W(\lambda_i, \cdot)}{\mcE_W } +  \frac{ 2 k l W(\lambda_i, \cdot)^{\alpha} }{ \mcE_W \mcE_W(\alpha)} \cdot \int_0^1 (1 - \rho_n W(\lambda_i, y) ) W(y, \cdot) \, dy.
  \end{gather*}
\end{lemma}

\begin{alg}[Random walk sampling \parencite{perozzi_deepwalk_2014,grover_node2vec_2016}]
  \label{alg:random_walk}
  Given a graph $\mcG_n$, a walk length $k$, number of negative samples $l$ per positively sampled vertex and unigram parameter $\alpha$, we
  \begin{enumerate}[label=\roman*), leftmargin=*]
      \item Perform a simple random walk on $\mcG_n$ of length $k$, beginning from its stationary distribution, to form a path $(\tilde{v}_i)_{i \leq k+1}$, and report $(\tilde{v}_i, \tilde{v}_{i+1})$ for $i \leq k$ as part of $S_0(\mcG_n)$;
      \item For each vertex $\tilde{v}_i$, we select $l$ vertices $(\eta_j)_{j \leq l}$ independently and identically according to the unigram distribution 
      \begin{equation*}
          \mathrm{Ug}_{\alpha}( v \,|\, \mcG_n) \propto \mathbb{P}\big( \tilde{v}_i = v \text{ for some } i \leq k \,|\, \mcG_n \big)^{\alpha}
      \end{equation*}
      and then form $S_{ns}(\mcG_n)$ as the collection of vertex pairs $(\tilde{v}_i, \eta_j)$ which are not an edge in $\mcG_n$.
  \end{enumerate}
  We then return $S(\mcG_n)$ as the union of $S_0(\mcG_n)$ and $S_{ns}(\mcG_n)$.
\end{alg}

\begin{lemma} \label{sampling:rw_uni_stat_formula}
  For Algorithm~\ref{alg:random_walk}, Assumption~\ref{assume:slc} holds with
  $s_n = (\log(n)/ n \rho_n)^{1/2}$,
  \begin{gather*}
      f_n(\lambda_i, \lambda_j, 1 ) = \frac{2k}{\mcE_W \rho_n}, \; f_n(\lambda_i, \lambda_j, 0) = \frac{ l(k+1) }{\mcE_W \mcE_W(\alpha) } \big\{ W(\lambda_i, \cdot) W(\lambda_j, \cdot)^{\alpha} + W(\lambda_i, \cdot)^{\alpha} W(\lambda_j, \cdot) \big\}, \\
  \tilde{g}_n(\lambda_i)  = \frac{ k W(\lambda_i, \cdot) }{\mcE_W} +  \frac{ (k+1)l W(\lambda_i, \cdot)^{\alpha} }{ \mcE_W(\alpha) \mcE_W } \cdot \int_0^1 (1 - \rho_n W(\lambda_i, y) ) W(y, \cdot) \, dy.
  \end{gather*}
\end{lemma}

Examining the formula for $\tilde{g}_n(\lambda)$ above, we see that for random walk samplers the shrinkage provided to the learned kernel will be greater for vertices with larger degrees. Indeed, as $K(\lambda, \lambda) = \| \eta(\lambda) \|_2^2$ for $K \in \mcZ_d^{\geq 0}(A)$, the larger $\tilde{g}_n(\lambda)$ is, $\| \eta(\lambda) \|_2^2$ will be forced closer towards zero.

\subsubsection{An illustrating example}
\label{sec:theory:illustrate}

We now give a brief illustration of our theoretical results under a simple graphon model. To do so, we consider a sparsified SBM with $\kappa$ communities, each equiprobable, and with probabilities $\rho_n p$ and $\rho_n q$ ($p > q$) denoting the within-community and between-community edge probabilities. Writing $A_i = [(i-1)/\kappa, i/\kappa)$ for $i \in [\kappa]$, this can be represented as a graphon model with graphon $W_n = \rho_n W$, where $W(u, v) = p$ if $(u, v) \in \prod_{i=1}^\kappa A_i \times A_i$ and $W(u, v) = q$ otherwise. 

\begin{theorem}
    \label{thm:example}
    Suppose that the graph arises from the model discussed above, we use a cross-entropy loss and the random walk sampling scheme as described in Algorithm~\ref{alg:random_walk}. Then Theorem~\ref{thm:embed_conv} holds such that, for any minimizer $\bmomega = (\whomega_1, \ldots, \whomega_n)$ satisfying \eqref{eq:embed_conv:embeds} of Theorem~\ref{thm:embed_conv}, we have that
    \begin{equation*}
        \frac{1}{n^2} \sum_{i, j} \big( \langle \whomega_i, \whomega_j \rangle - K_n(\lambda_i, \lambda_j) \big)^2 = o_p(1) \text{ where } K_n(u, v) = (\tilde{K}_n)_{i, j} \text{ if } (u, v) \in A_i \times A_j,
    \end{equation*}
    and $\tilde{K}_n \in \mathbb{R}^{\kappa \times \kappa}$ is defined as the unique positive semi-definite minimizer of the function 
    \begin{equation*}
        - \frac{1}{\kappa^2} \sum_{i, j} \big\{ 2k c_1 \cdot (p \delta_{ij} + q (1 - \delta_{ij} )) \log \sigma( \tilde{K}_{ij}) + 2l (k+1) \log \sigma(-\tilde{K}_{ij}) \big\} + \xi c_2 \| \tilde{K} \|_{*}
    \end{equation*}
    where we write $c_1 := ( (p + (\kappa - 1)q )/\kappa)^{-1}$ and $c_2 := k + l(k+1)(1 - \rho_n c_1^{-1})$. In particular, for the above example we see that the form of the regularizer is exactly the nuclear norm of $\tilde{K}$, and so as $\xi$ increases, the singular values of the minimizer will be shrunk towards zero.
\end{theorem}

\section{Experiments}
\label{sec:exps}

We now examine the performance in using regularized node2vec embeddings for
link prediction and node classification tasks, and illustrate comparable, when not superior, performance to more complicated encoders for network embeddings. We perform experiments on the Cora, CiteSeer and PubMedDiabetes citation network datasets (see Appendix~\ref{sec:app:exps} for more details), which we use as they are commonly used benchmark datasets - see e.g \parencite{hamilton_inductive_2017,velickovic_deep_2018,kipf_variational_2016,hassani_contrastive_2020}.

\begin{figure*}[t!]
  \begin{subfigure}[b]{0.25\linewidth}
  \centering \includegraphics[width=\textwidth]{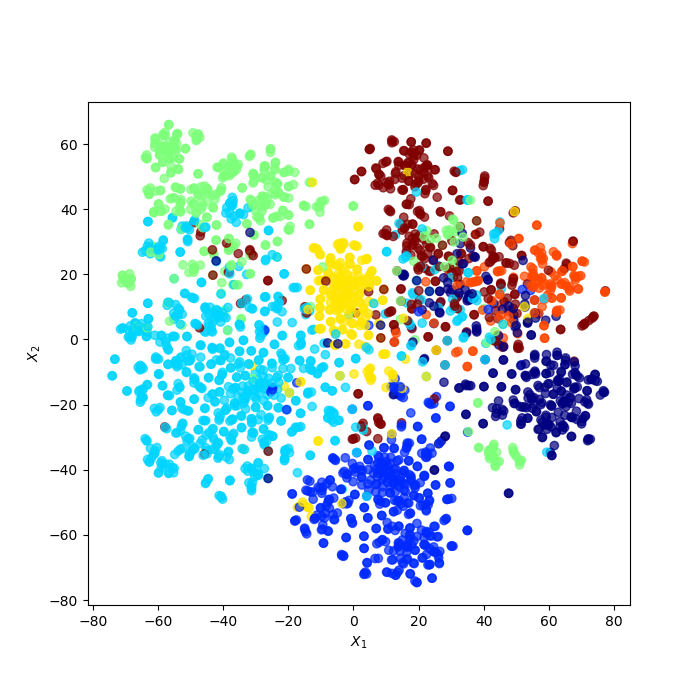}
  \caption{$\xi = 0$}
  \end{subfigure}\hfill%
  \begin{subfigure}[b]{0.25\linewidth}
  \centering \includegraphics[width=\textwidth]{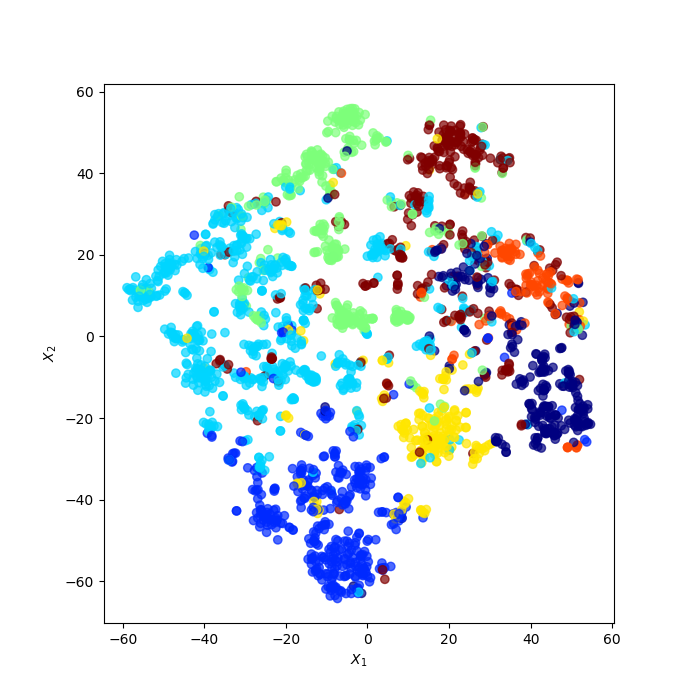}
  \caption{$\xi = 5 \times 10^{-7}$}
  \end{subfigure}\hfill%
  \begin{subfigure}[b]{0.25\linewidth}
    \centering \includegraphics[width=\textwidth]{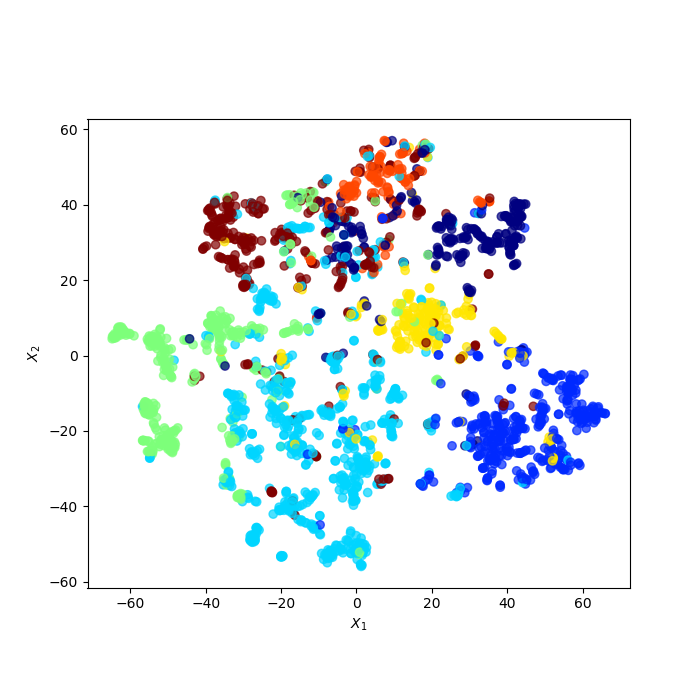}
    \caption{$\xi = 10^{-6}$}
    \end{subfigure}\hfill%
    \begin{subfigure}[b]{0.25\linewidth}
    \centering \includegraphics[width=\textwidth]{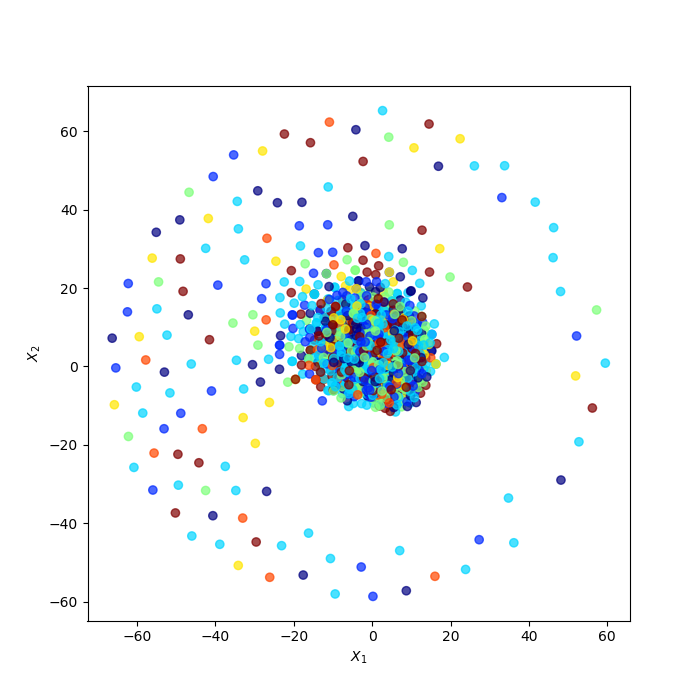}
    \caption{$\xi = 10^{-3}$}
  \end{subfigure}
\caption{TSNE visualizations of Cora network embeddings, learnt using node2vec for different regularization penalties $\xi$, with different colors representing different classes. As $\xi$ increases, the classes begin to cluster together and separate, and then eventually collapse towards the origin.}
\label{fig:tsne_cora}
\end{figure*}

\subsection{Methods used for comparison}
\label{sec:exps:methods}

For our experiments, we consider 128 dimensional node2vec embeddings learned with either no regularization, or a $\ell_2$ penalty on the embedding vectors with weight $\xi \in \{1, 5\} \times 10^{-\{3, 4, 5, 6, 7, 8\}}$; and with or without the node features concatenated. We compare these against methods which incorporate the network and covariate structure together in a non-linear fashion. We consider three methods - a two layer GCN \parencite{kipf_semi-supervised_2017} with 256 dimensional output embeddings trained in an unsupervised fashion through the node2vec loss, a two layer GraphSAGE architecture \parencite{hamilton_inductive_2017} with 256 dimensional output embeddings trained through the node2vec loss, and a single layer 256 dimensional GCN trained using Deep Graph Infomax (DGI) \parencite{velickovic_deep_2018}. We highlight that the GCN, GraphSAGE and DGI
always have access to the nodal features during training. All of our experiments used the Stellargraph \parencite{data61_stellargraph_2018} implementations for each method. Further experimental details are given in Appendix~\ref{sec:app:exps}.

\subsection{Link prediction experiments}

For the link prediction experiments, we create a training graph by removing 10\% of both the edges and non-edges within the network, and use this to learn an embedding of the network. We then form link embeddings by taking the entry-wise product of the corresponding node embeddings, use 10\% of the held-out edges to build a logistic classifier for the link categories, and then evaluate the performance on the remaining edges, repeating this process
50 times.

The PR AUC scores are given in Table~\ref{tab:link_results_pr} and visualized in Figure~\ref{fig:link_class_graphs}. Provided the regularization parameter is chosen optimally, we see an improvement in performance compared to using no regularization, with the jump in performance slightly greater when nodal features are not incorporated
into the embedding. The optimally regularized version of node2vec, even without including the node features, is competitive with the GCN and GraphSAGE (being equal or outperforming at least one of them across the three datasets), and that the version with concatenated node features is as good as the GCN trained using DGI. For all the datasets, we observe a sharp decrease in performance after the optimal weight, suggesting that it needs to be chosen carefully to avoid removing the informative network structure. As seen in Figure~\ref{fig:tsne_cora}, this occurs as the learned embeddings become randomly distributed at the origin once the regularization weight is too large.

\begin{table}
  \caption{PR AUC scores for link prediction experiments, and macro F1 scores for node classification experiments. "n2v" stands for node2vec without regularization; "rn2v" stands for regularized node2vec with the best score over the specified grid of penalty values; "NF" indicates that node features were concatenated to the learned node embeddings.}
  \vspace{5pt}
  \centering
  \label{tab:link_results_pr}
  \begin{tabular}{@{}ccccccc@{}}
  \toprule
  \multirow{2}{*}{Methods} & \multicolumn{3}{c}{PR AUC (link prediction)} 
  & \multicolumn{3}{c}{Macro F1 (node classification)}
  \\ \cmidrule(l){2-4} \cmidrule(l){5-7}
   & Cora & CiteSeer & PubMed & Cora & CiteSeer & PubMed\\ 
   \midrule
  n2v & 0.84 $\pm$ 0.01 & 	0.80 $\pm$ 0.02 & 0.86 $\pm$ 0.01 
  & 0.67 $\pm$ 0.04	& 0.48 $\pm$ 0.03	& 0.76 $\pm$ 0.00\\ 
  rn2v & 0.90 $\pm$ 0.01 &	0.88 $\pm$ 0.02 & 0.91 $\pm$ 0.00
  & 0.73 $\pm$ 0.03 &	0.55 $\pm$ 0.04	& 0.77 $\pm$ 0.00 \\[3pt]
  n2v+NF & 0.88 $\pm$ 0.01 &	0.90 $\pm$ 0.01 & 0.92 $\pm$ 0.00 
  & 0.70 $\pm$ 0.03 &	0.54 $\pm$ 0.03	& 0.79 $\pm$ 0.01\\
  rn2v+NF & 0.92 $\pm$ 0.01	& 0.93 $\pm$ 0.01 & 0.95 $\pm$ 0.00
  & 0.74 $\pm$ 0.03	& 0.58 $\pm$ 0.04	& 0.84 $\pm$ 0.00 \\[3pt]
  GCN & 0.90 $\pm$ 0.01 &	0.87 $\pm$ 0.02 & 0.94 $\pm$ 0.00 
  & 0.67 $\pm$ 0.04 & 0.48 $\pm$ 0.04 & 0.80 $\pm$ 0.01\\
  GSAGE & 0.90 $\pm$ 0.01 &	0.90 $\pm$ 0.01 & 0.88 $\pm$ 0.01 
  & 0.74 $\pm$ 0.04 &	0.56 $\pm$ 0.03 &	0.79 $\pm$ 0.00\\
  DGI & 0.91 $\pm$ 0.01 & 0.93 $\pm$ 0.01 & 0.95 $\pm$ 0.00 & 
  0.76 $\pm$ 0.03	& 0.60 $\pm$ 0.02	& 0.84 $\pm$ 0.00 \\ \bottomrule
  \end{tabular}
\end{table}

\begin{figure*}[t]
  \centering
  \begin{subfigure}[b]{0.30\linewidth}
  \centering \includegraphics[width=0.8\textwidth]{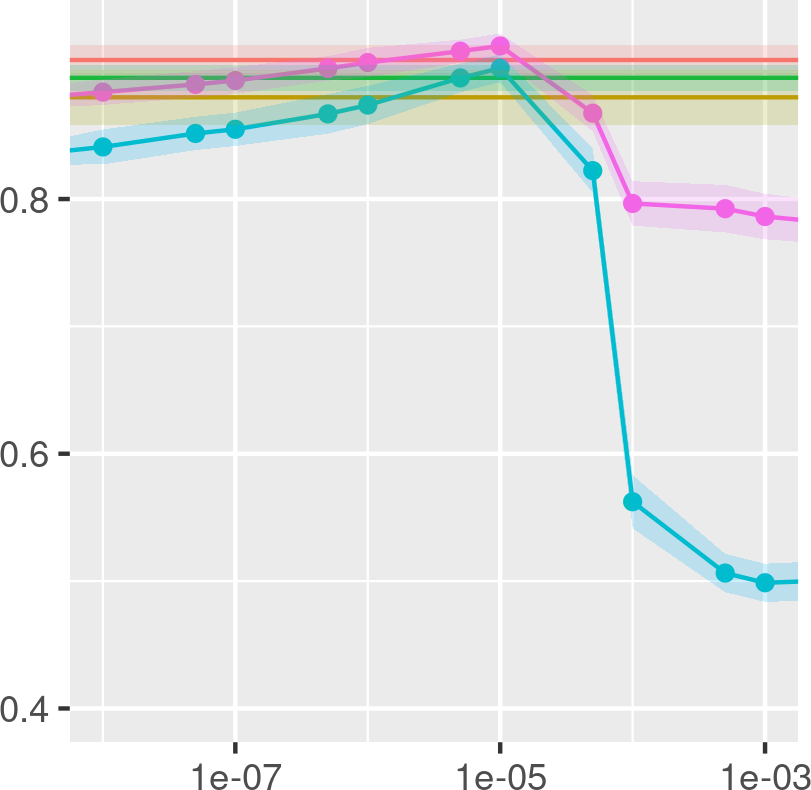}
  \caption{PR AUC scores on Cora}
  \end{subfigure}\hfill%
  \begin{subfigure}[b]{0.30\linewidth}
    \centering \includegraphics[width=0.8\textwidth]{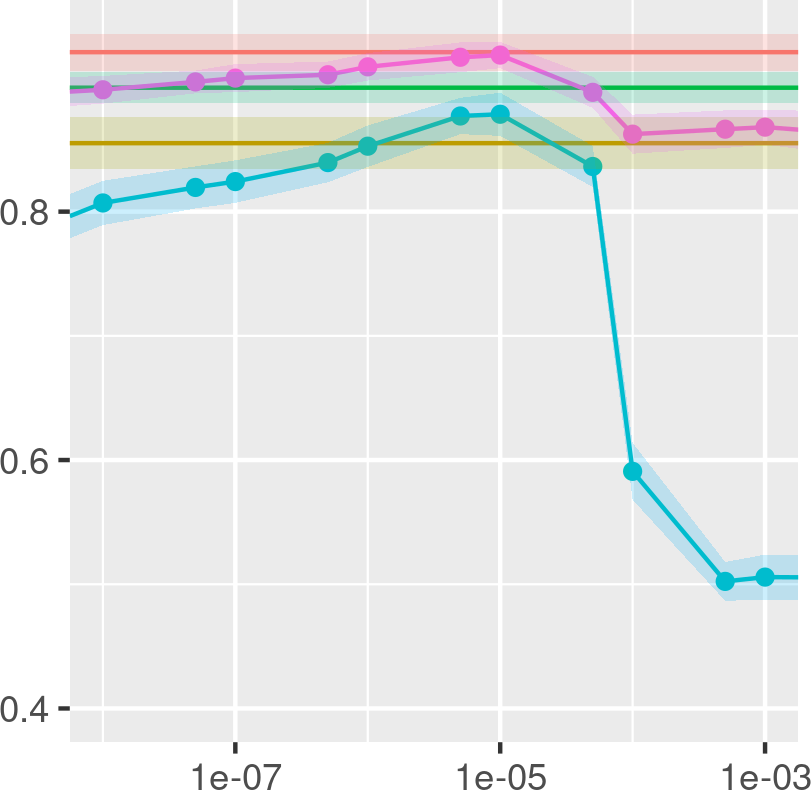}
    \caption{PR AUC scores on CiteSeer}
  \end{subfigure}\hfill%
  \begin{subfigure}[b]{0.30\linewidth}
      \centering \includegraphics[width=0.8\textwidth]{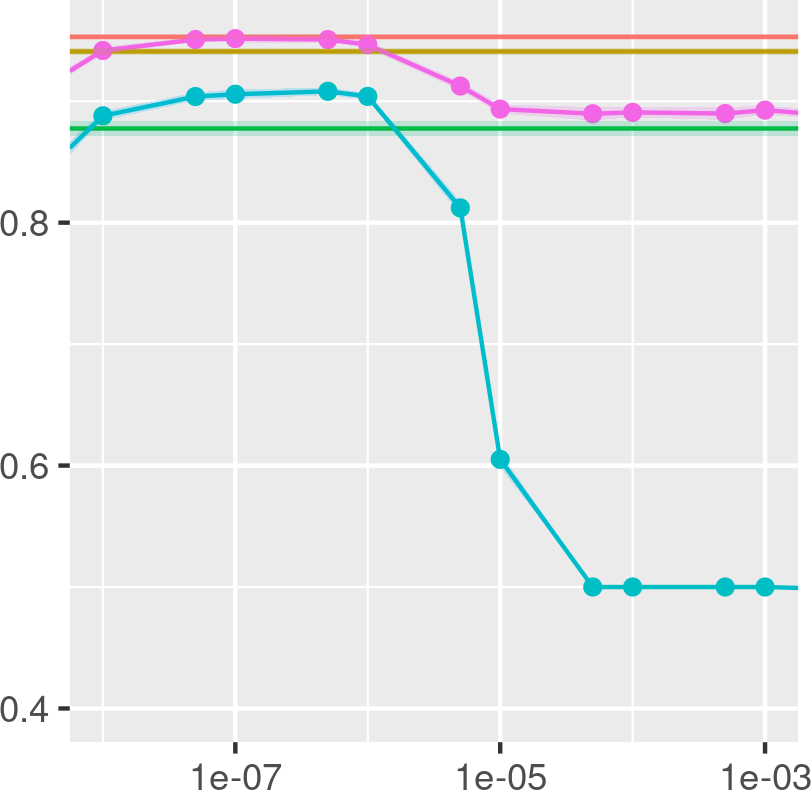}
      \caption{PR AUC scores on PubMed}
    \end{subfigure}%
    \\\vspace{5pt}%
    \begin{subfigure}[b]{0.30\linewidth}
        \centering \includegraphics[width=0.8\textwidth]{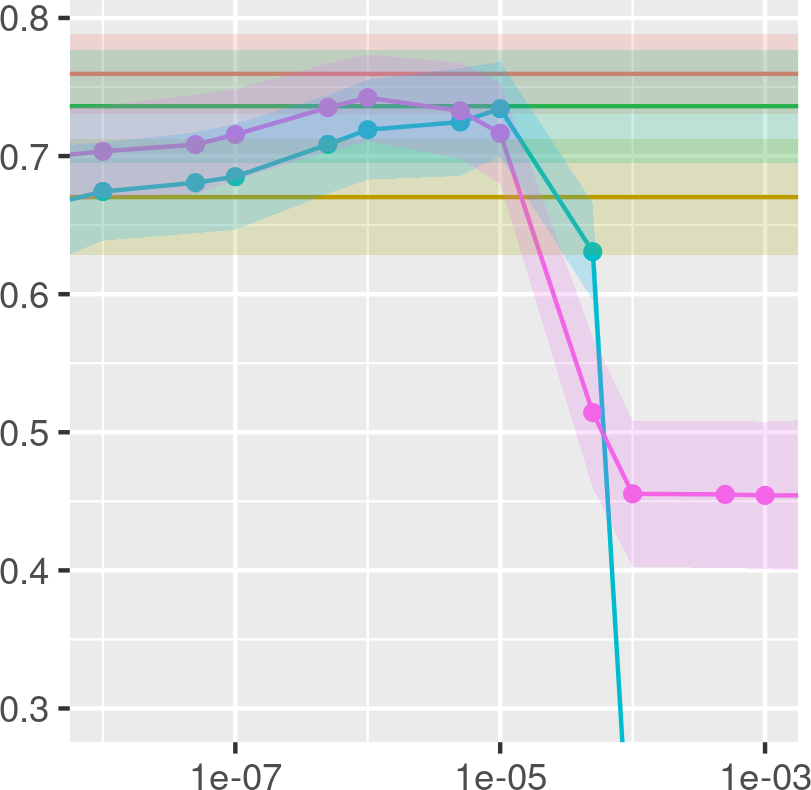}
        \caption{Macro F1 scores on Cora}
        \end{subfigure}\hfill%
        \begin{subfigure}[b]{0.30\linewidth}
          \centering \includegraphics[width=0.8\textwidth]{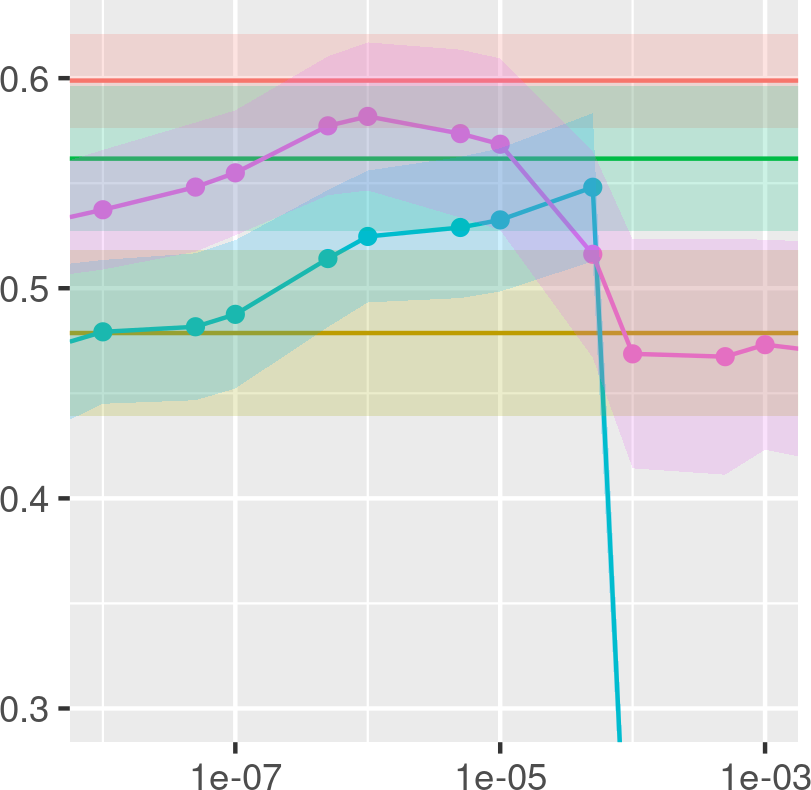}
          \caption{Macro F1 scores on CiteSeer}
        \end{subfigure}\hfill%
        \begin{subfigure}[b]{0.30\linewidth}
            \centering \includegraphics[width=0.8\textwidth]{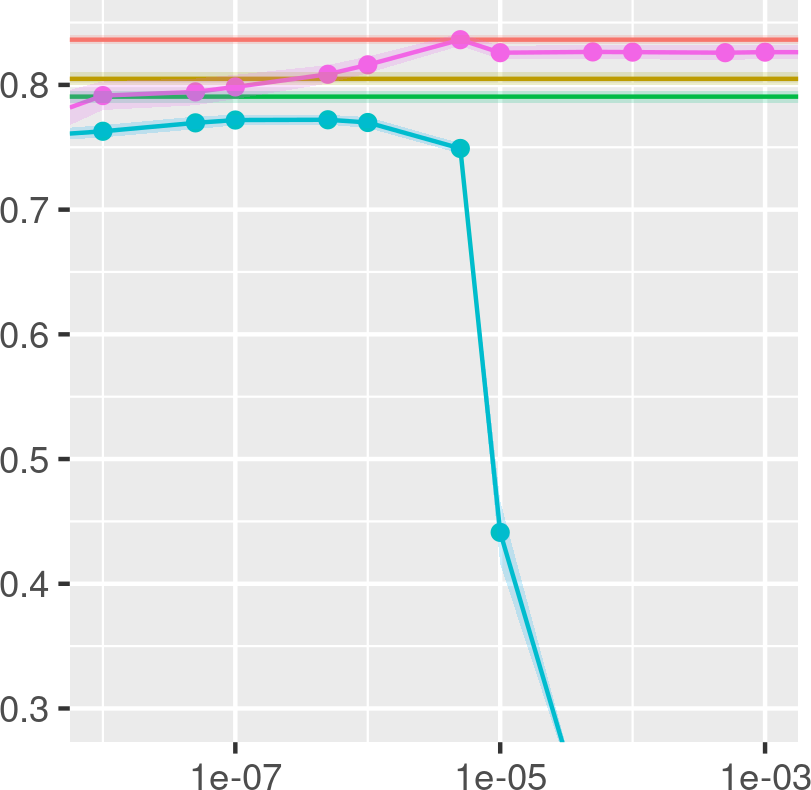}
            \caption{Macro F1 scores on PubMed}
          \end{subfigure}
\caption{Plots of average PR AUC scores for link prediction and macro F1 scores for node classification (bands representing plus/minus one standard deviation), against the regularization penalty used for node2vec with (purple) and without (blue) concatenated node features. The results from GCN, GraphSAGE and DGI are given as horizontal bars in yellow, green and red respectively.}
\label{fig:link_class_graphs}
\end{figure*}

\subsection{Node classification experiments}

To evaluate performance for the node classification task, we 
learn a network embedding without access to the node labels, and
then learn/evaluate a one-versus-rest multinomial node classifier
using 5\%/95\% stratified training/test splits of the node labels.
We repeat this over 25 training runs of the embeddings, and a
further 25 training/test splits for the node classifiers per embedding. Table~\ref{tab:link_results_pr} and Figure~\ref{fig:link_class_graphs} show the average macro F1 scores and their standard deviation for each method and dataset. Similar to the link prediction experiments, we see that the optimally regularized node2vec methods are competitive, if not outperforming, the GCN and GraphSAGE trained through the node2vec loss, and is outperformed by the GCN learned using DGI by no more than two percentage points. In these experiments, the standard deviations correspond partially to the randomness induced by the training/test splits of the node labels, and suggest that the regularized version of node2vec is no less robust to particular choices of training/test splits than the other methods.

Interestingly, we note that the optimal performance on PubMed is given by the regularized node2vec with node features. However, as the highest level of performance is performed when the regularization weight is so large that the learnt embeddings are uninformative (as illustrated by the massive decrease
in performance of the regularized method without node features), this suggests that the nodes can be classified using only the covariate information, and that the network features are not needed. As the dataset only has three distinct
classes corresponding to the academic field, and the node features are TF-IDF embeddings of the academic papers, this is not too surprising.

\section{Conclusion}
\label{sec:conc}

In this paper we have theoretically described the effects of performing $\ell_2$ regularization on network embedding vectors learned by schemes
such as node2vec, describing the asymptotic distribution of the
embedding vectors learned through such schemes, and showed that
the regularization helps to reduce the effective
dimensionality of the learned embeddings by penalizing the singular values
of the limiting distribution of the embeddings towards zero. We 
do so in the non-convex regime $d \ll n$, reflecting how
embedding dimensions are chosen in the real world. 
We moreover highlight empirically that using $\ell_2$ regularization
with node2vec
leads to competitive performance on downstream tasks, when compared to 
embeddings produced from more recent encoding and training architectures. 

We end with a brief discussion of some limitations of our works,
and directions for future work. Generally, graphons are not realistic
models for graphs; we suggest that our work could be
extended to frameworks such as graphexes \parencite{veitch_class_2015,borgs_sparse_2018-1,borgs_sampling_2019} which can produce
more realistic degree distributions for networks, but have
enough underlying latent exchangeability for our arguments
to go through. We ignore aspects of 
optimization, i.e whether the minima of \eqref{eq:framework:emp_risk} 
are obtained in practice, which we believe would be an interesting
area of future research. As for our experimental results, we note that
methods such as GraphSAGE are better than node2vec in that they provide
embeddings for data unobserved during training,
and also scale better to larger networks. Consequently, 
we believe our experiments should be used primarily as motivation to investigate better methods for incorporating
nodal covariates into network embedding models, and how to regularize
embeddings produced by encoder methods such as GCNs or GraphSAGE.

\begin{ack}
    We thank Morgane Austern for helpful discussions, 
    along with several anonymous reviewers for their comments on the current version of the paper, along with a prior version. We also acknowledge computing resources from Columbia University's Shared Research Computing Facility project, which is supported by NIH Research Facility Improvement Grant 1G20RR030893-01, and associated funds from the New York State Empire State Development, Division of Science Technology and Innovation (NYSTAR) Contract C090171, both awarded April 15, 2010. 
\end{ack}

\printbibliography

\section*{Checklist}

\begin{enumerate}

\item For all authors...
\begin{enumerate}
  \item Do the main claims made in the abstract and introduction accurately reflect the paper's contributions and scope?
    \answerYes{}
  \item Did you describe the limitations of your work?
    \answerYes{See the first paragraph in the related works section; Remark~\ref{rmk:sparse}; the paragraph between Theorem 3 and 4; the conclusion; Remark~\ref{app:rmk:rates_of_convergence} in the appendix re: our rates of convergence.}
  \item Did you discuss any potential negative societal impacts of your work?
    \answerNo{The paper discusses the theoretical effects of a particular type of regularization scheme, and does not advocate for any methodological changes or gives particular applications; I believe that making \emph{any} types of claims about the societal impacts of this work would be excessively speculative, to be frank.}
  \item Have you read the ethics review guidelines and ensured that your paper conforms to them?
    \answerYes{Yes (\url{https://neurips.cc/public/EthicsGuidelines}).}
\end{enumerate}

\item If you are including theoretical results...
\begin{enumerate}
  \item Did you state the full set of assumptions of all theoretical results?
    \answerYes{See Section~\ref{sec:framework} and the relevant theorem statements; the one theorem in the main body which does not state them fully (Theorem~\ref{thm:embed_conv}) mentions the corresponding version of the theorem in the appendix (Theorem~\ref{app:thm:embed_conv_full})
    which has the full regularity conditions.}
    \item Did you include complete proofs of all theoretical results?
    \answerYes{See the appendix.}
\end{enumerate}

\item If you ran experiments...
\begin{enumerate}
  \item Did you include the code, data, and instructions needed to reproduce the main experimental results (either in the supplemental material or as a URL)?
    \answerYes{See Appendix~\ref{sec:app:exps}.}
  \item Did you specify all the training details (e.g., data splits, hyperparameters, how they were chosen)?
    \answerYes{See Section~\ref{sec:exps} and Appendix~\ref{sec:app:exps}.}
        \item Did you report error bars (e.g., with respect to the random seed after running experiments multiple times)?
    \answerYes{See Table~\ref{tab:link_results_pr} and Figure~\ref{fig:link_class_graphs}.}
    \item Did you include the total amount of compute and the type of resources used (e.g., type of GPUs, internal cluster, or cloud provider)?
    \answerYes{See Appendix~\ref{sec:app:exps}.}
\end{enumerate}

\item If you are using existing assets (e.g., code, data, models) or curating/releasing new assets...
\begin{enumerate}
  \item If your work uses existing assets, did you cite the creators?
    \answerYes{See Section~\ref{sec:exps:methods} and Appendix~\ref{sec:app:exps}.}
  \item Did you mention the license of the assets?
    \answerYes{See Appendix~\ref{sec:app:exps}.}
  \item Did you include any new assets either in the supplemental material or as a URL?
    \answerNA{No new assets were created.}
  \item Did you discuss whether and how consent was obtained from people whose data you're using/curating?
    \answerNA{Standard datasets were used, which do not contain information about/curated from individuals.}
  \item Did you discuss whether the data you are using/curating contains personally identifiable information or offensive content?
    \answerNA{Datsets used were standard citation network datasets with no identifiable/textual information provided.}
\end{enumerate}

\item If you used crowdsourcing or conducted research with human subjects...
\begin{enumerate}
  \item Did you include the full text of instructions given to participants and screenshots, if applicable?
    \answerNA{No crowdsourcing/human subjects were used.}
  \item Did you describe any potential participant risks, with links to Institutional Review Board (IRB) approvals, if applicable?
    \answerNA{No crowdsourcing/human subjects were used.}
  \item Did you include the estimated hourly wage paid to participants and the total amount spent on participant compensation?
    \answerNA{No crowdsourcing/human subjects were used.}
\end{enumerate}

\end{enumerate}

\newpage

\appendix
\section{Introduction to exchangeable graph or graphon models}
\label{sec:app:graphon}

In the following discussion, we assume all graphs mentioned are undirected and have no self loops. A \emph{graphon model} refers to a probabilistic model on a graph on a countable set $\mcV \subseteq \mathbb{N}$, defined via a \emph{graphon}, which we define as a symmetric measurable function $W: [0, 1]^2 \to [0, 1]$. To define the law of the graph, for each vertex $u \in \mcV$, we assign an independent latent variable $\lambda_u \sim \mathrm{Unif}[0, 1]$, and then assign edges independently according to the law
\begin{equation}
    a_{uv} | \lambda_u, \lambda_v \sim \mathrm{Bernoulli}\big( W(\lambda_u, \lambda_v) \big)
\end{equation}
independently for $u < v$, and then setting $a_{vu} = a_{uv}$ for $u > v$. 

The Aldous-Hoover theorem \parencite{aldous_representations_1981} then gives the following equivalence between probabilistic models of a graph on a vertex set $\mcV = \mathbb{N}$:
\begin{enumerate}
    \item The law of the adjacency matrix is invariant to joint permutations of its rows and columns; in other words, for any permutation $\tau \in \mathrm{Sym}(\mcV)$ we have that $(a_{uv})_{u,v} \stackrel{d}{=} (a_{\tau(u) \tau(v) })_{u,v}$.
    \item There exists a graphon $W$ such that the law of the adjacency matrix is equivalent to a graphon model with graphon $W$.
\end{enumerate}
The presentation above choosing the latent distribution of the vertices to be uniform is a canonical one, but can be generalized. If we instead assign latent variables $\bm{\lambda}_u \iid Q$ for some probability distribution $Q \subseteq \mathbb{R}^p$ and assign edges $a_{uv} = 1$ independently with probability $\bm{W}(\bm{\lambda}_u, \bm{\lambda}_v)$ for some symmetric function $\bm{W}$, then the law of the graph is still exchangeable, and hence equivalent to a graphon model as presented above. 

% TODO: discuss stochastic block models 
One special case of a graphon model is known as a stochastic block model (SBM) \parencite{holland_stochastic_1983}. The typical formulation of a SBM defines a probabilistic model on a network given a number of communities $\kappa$, a probability distribution $(\pi_i)_{i \in [\kappa]}$ on $[\kappa]$, and a symmetric matrix $P \in [0, 1]^{\kappa \times \kappa}$. For $u \in \mcV$, we assign a community $C(u) \in [k]$ 
with probability
\begin{equation}
    \mathbb{P}\big( C(u) = j \big) = \pi_j \text{ for } j \in [\kappa]
\end{equation}
independently for each $u \in \mcV$. Conditional on these 
assignments, we then form the adjacency matrix of the network
via connecting vertices independently with probability
\begin{equation}
    \mathbb{P}\big( a_{uv} = 1 \,|\, C(u), C(v) \big) = P_{C(u), C(v)} = e_{C(u)}^T P e_{C(v)}
\end{equation}
where $e_i \in \mathbb{R}^\kappa$ denotes the $i$-th unit vector in $\mathbb{R}^\kappa$. This can be defined as a graphon model as follows: forming a partition of $[0, 1]$, say $(A_i)_{i \in [\kappa]}$, for
which $|A_i| = \pi_i$ for $i \in [k]$, then we can define a graphon model by using latent variables $\lambda_u \iid \mathrm{Unif}[0, 1]$ and a graphon function
\begin{equation}
    W(u, v) = P_{C(u), C(v)} \text{ for } u, v \in [0, 1] \qquad \text{ where }
    C(u) = j \text{ if } u \in A_j.
\end{equation}
The law of the above graphon model is then identical to that of the SBM defined with $\pi$ and $P$. Such graphons are sometimes referred to as \emph{stepfunctions}, which are
graphons $W$ which are piecewise constant on a partition $\mcP \times \mcP$
of $[0, 1]^2$, where $\mcP$ is a partition of $[0, 1]$.

As presented, graphon models have some shortcomings. For example, by taking expectations, if we have a graphon model on a vertex set $\mcV_n = [n]$, the average number of edges will be $n^2 \int_0^1 \int_0^1 W(x, y) \, dx dy$. This means that graphon models give rise to \emph{dense} graphs, which is not a realistic assumption for naturally occurring networks. One way of accounting for this, particularly when considering sequences of graphs, is to consider the sequence of graphs $\mcG_n$ on $\mcV_n = [n]$ where, for each $n$, the generating graphon used is $W_n = \rho_n W$ where $W$ is a graphon function, and $\rho_n$ is a sparsifying sequence for which $\rho_n \to 0$ as $n \to \infty$. Such graphs are referred to as \emph{sparisified graphon models}.
\section{Expanded derivations from Sections~\ref{sec:intro}~and~\ref{sec:framework}}
\label{sec:app:expand}

\subsection{Stochastic gradient descent and empirical risk minimization}
\label{sec:app:sgd_erm}

We begin with considering the gradient updates of the form 
\begin{align}
    \label{app:emp_risk_step}
    \omega_u \leftarrow \omega_u - \eta \nabla_{\omega_u} \mcL \quad \text{ where } \quad 
    \mcL = \sum_{(i, j) \in \mcP } \ell_{\mcP}( \langle \omega_i, \omega_j \rangle ) + \sum_{(i, j) \in \mcN } \ell_{\mcN}( \langle \omega_i, \omega_j \rangle )
\end{align}
performed at each iteration of stochastic gradient descent, where $\eta > 0$ is a sequence of step sizes, and $\mcP$ and $\mcN$ are random subsets of $\mcV \times \mcV$ formed at every iteration of stochastic gradient descent. Note that we can equivalently write this as
\begin{equation}
    \label{app:emp_risk_step_2}
    \mcL = \sum_{i,j} \big\{ \mathbbm{1}\big[ (i, j) \in \mcP \big] \ell_{\mcP}(\langle \omega_i, \omega_j \rangle ) + \mathbbm{1}\big[ (i, j) \in \mcN \big] \ell_{\mcN}( \langle \omega_i, \omega_j \rangle ) \big\}, 
\end{equation}
where we use indicator terms to allow for the summation to occur over all pairs $(i,j)$. 

Recall that stochastic gradient descent, as introduced in
\textcite{robbins_stochastic_1951}, works by the following principle: suppose
we have a function of the form
\begin{equation}
    F(\theta) := \mathbb{E}_{X \sim Q}\big[ f(X, \theta) \big]
\end{equation}
for some function $f : \mcX \times \Theta \to \mathbb{R}$ and distribution $Q$
on $\mcX$ according to a random variable $X$. Moreover, suppose that we have
access to an \emph{unbiased} estimator of the gradient $\nabla_{\theta} F$ of $F$, say $g(x, \theta)$, so that $\mathbb{E}_{X \sim Q}[ g(X, \theta)] = \nabla_{\theta} F$. One can then show in various settings \parencite[see e.g][]{robbins_stochastic_1951,bottou_approche_1991,bottou_optimization_2018,ghadimi_stochastic_2013,bertsekas_gradient_2000,li_convergence_2019}
that the optimization scheme
\begin{equation}
    \label{app:sgd_gradient_iterations}
    \theta_{t+1} := \theta_{t} - \eta_t g(x_t; \theta)
\end{equation}
where $x_t \iid P$ will converge to a local minima of $F(\theta)$, at
least under some conditions on the step sizes $\eta_t > 0$ and the 
curvature of $F(\theta)$ about its local minima.

Applying this to the scenario in \eqref{app:emp_risk_step_2}, we note that
at each iteration, we sample sets $\mcP \subseteq \mcV \times \mcV$ and
$\mcN \subseteq \mcV \times \mcV$ at each iteration independently across iterations, according
to some probability measures $Q_{\mcP}$ and $Q_{\mcN}$ over $\mcV \times \mcV$.
With these sets, we perform gradient updates as in \eqref{app:sgd_gradient_iterations}
\begin{equation}
    \nabla_{\omega} \mcL = \sum_{i,j} \big\{ \mathbbm{1}\big[ (i, j) \in \mcP \big] \nabla_{\omega} \ell_{\mcP}(\langle \omega_i, \omega_j \rangle ) + \mathbbm{1}\big[ (i, j) \in \mcN \big] \nabla_{\omega} \ell_{\mcN}( \langle \omega_i, \omega_j \rangle ) \big\}
\end{equation}
for each embedding vector $\omega$, which is an unbiased estimator
of $\nabla_{\omega} \mcR$ where
\begin{equation}
    \label{app:emp_risk}
    \mcR = \sum_{i, j} \big\{ \mathbb{P}\big( (i, j) \in \mcP \big) \ell_{\mcP}( \langle \omega_i, \omega_j \rangle) + \mathbb{P}\big( (i, j) \in \mcN \big) \ell_{\mcN}( \langle \omega_i, \omega_j \rangle) \big\}
\end{equation}
as a result of the fact that e.g. $\mathbb{E}[ \mathbbm{1}[(i, j) \in \mcP]] = \mathbb{P}( (i, j) \in \mcP)$. Consequently, the procedure described in
\eqref{app:emp_risk_step} attempts to minimize \eqref{app:emp_risk}.

\subsection{Embedding methods as implicit graphon learning}
\label{sec:app:embed_as_graphon_learning}

Write $\ell_P(y) = - \log \sigma(y)$ and $\ell_{\mcN}(y) = -\log \sigma(-y)$,
and moreover suppose that $\mcP$ and $\mcN$ are randomly drawn subsets
from the sets $\mcE$ and $\mcV \times \mcV \subseteq \mcE$, i.e that
\begin{equation}
    \mcP \subseteq \{ (u, v) \,:\, a_{uv} = 1 \}, \qquad 
    \mcN \subseteq \{ (u, v) \,:\, a_{uv} = 0 \}.
\end{equation}
Letting $\mcV = [n]$ for some integer $n$, note that in the model
\begin{equation}
    a_{uv} \,|\, \omega_u, \omega_v \sim \mathrm{Bernoulli}\big( \sigma( \langle \omega_u, \omega_v \rangle ) \big) \text{ independently for $u < v$}, 
\end{equation}
and setting $a_{vu} = a_{uv}$ for $v < u$, the contribution to the negative log-likelihood of a single edge $(u, v)$ is of the form
\begin{equation}
    \ell((u,v)) = - a_{uv} \log \sigma( \langle \omega_u, \omega_v \rangle) - (1 - a_{uv}) \log \{ 1 - \sigma( \langle \omega_u, \omega_v \rangle) \}.
\end{equation}
(Recall that $\sigma(-y) = 1 - \sigma(y)$ for $y \in \mathbb{R}$.) Note
that the $a_{uv}$ are jointly independent conditional on the collection
of embedding vectors. Now, if we let $\ell_\mcP(y) = -\log\sigma(y)$
and $\ell_\mcN(y) = - \log \sigma(-y)$, as $\mcP$ is a subset of $\mcE$
and $\mcN$ is a subset of $\mcV \times \mcV \setminus \mcE$, the contributions
to the stochastic loss take exactly the form specified in \eqref{eq:emp_risk_step}, as $\ell((u,v)) = \ell_\mcP(\langle \omega_u, \omega_v)$
when $a_{uv} = 1$, and $\ell((u,v)) = \ell_\mcN( \langle \omega_u, \omega_v)$
when $a_{uv} = 0$.

\subsection{Empirical risk when including regularization}
\label{sec:app:samp_reg}

We explain this using the weight decay formulation first, and then show
that one has similar reasoning when using the probabilistic modelling
approach. Note that when considering stochastic gradient iterations
to only update individual parameters at a time, weight decay is applied
per iteration to \emph{only} the parameters to be updated (as otherwise
all of the parameters will be shrunk towards zero while waiting for
the next bona-fide gradient update). Consequently, if we have a stochastic
loss 
\begin{equation}
    \mcL = \sum_{(i, j) \in \mcP } \ell_{\mcP}( \langle \omega_i, \omega_j \rangle ) + \sum_{(i, j) \in \mcN } \ell_{\mcN}( \langle \omega_i, \omega_j \rangle )
\end{equation}
and we write $\mcV(\mcP \cup \mcN)$ for the vertices which belong to either
$\mcP$ or $\mcN$, the gradient updates for any vertex $u \in \mcV(\mcP \cup \mcN)$ take the form 
\begin{equation}
    \omega_u \leftarrow \omega_u - \eta ( \nabla_{\omega_u} \mcL + 2 \xi \omega_u) = \omega_u - \eta \nabla_{\omega_u} \big[ \mcL +  \xi \| \omega_u \|_2^2 \big] 
\end{equation}
and otherwise $\omega_u$ is kept as-is, meaning the gradient updates
correspond to taking gradient updates with the stochastic loss
\begin{equation}
    \sum_{u, v} \Big\{ 1\big[ (u, v) \in \mcP \big] \ell_{\mcP}( \langle \omega_u, \omega_v \rangle ) + 1\big[ (u, v) \in \mcN \big] \ell_{\mcN}( \langle \omega_u, \omega_v \rangle ) \Big\} + \xi \sum_{u} 1\big[ u \in \mcV(\mcP \cup \mcN)\big].
\end{equation}
Consequently, following the same argument as in Appendix~\ref{sec:app:sgd_erm}
gives the form of \eqref{eq:emp_risk_reg}. In the probabilistic modelling formulation, we again note that the contributions to the negative log-likelihood in the subsampling regime should again only arise from vertices belonging to $\mcV(\mcP \cup \mcN)$, and consequently the same argument will give the form of \eqref{eq:emp_risk_reg}.

\subsection{Simplifying the risk for certain positive/negative sampling schemes}
\label{sec:app:get_to_the_risk}

We note that in the setting where $\mcP \subseteq \mcE_n$ and $\mcN \subseteq (\mcV_n \times \mcV_n) \setminus \mcE_n$, we can write $S(\mcG_n) = \mcP \cup \mcN$, and also
\begin{align*}
    \ell(y, a_{ij}) & = \ell_\mcP(y) 1[ a_{ij} = 1] + \ell_{\mcN}(y) 1[ a_{ij} = 0], \\ 
    \mathbb{P}\big( (i, j) \in S(\mcG_n) \,|\, \mcG_n) & = \mathbb{P}\big( (i, j) \in \mcP \,|\, \mcG_n) 1[ a_{ij} = 1] + \mathbb{P}\big( (i, j) \in \mcN \,|\, \mcG_n) 1[ a_{ij} = 0],
\end{align*}
and as a result, we end up obtaining \eqref{eq:framework:emp_risk}
from \eqref{eq:emp_risk_reg}.
\section{Proof of Theorem~\ref{thm:loss_conv}}
\label{sec:app:loss_conv}

We follow the style of argument given in Appendix~C of \parencite{davison_asymptotics_2021}, introducing various intermediate functions and chaining together uniform convergence bounds between these functions over sets containing the minima of both functions; consequently, we
break the proof up into several parts. Note that 
we implicitly let the embedding dimension $d$ 
depend on $n$ throughout. 

Before giving the proof, we state some results from \parencite{davison_asymptotics_2021} which will be used in the following proof.

\begin{prop}[Appendix~C of \parencite{davison_asymptotics_2021}]
    \label{thm:app:loss:bounds}
    Suppose that Assumptions~\ref{assume:slc}~and~\ref{assume:reg} hold. Then we have the following:
    \begin{enumerate}[label=\roman*), leftmargin=*]
        \item (Theorem~30 and Lemma~41 of \parencite{davison_asymptotics_2021}) For the functions
        \begin{align*}
            \widehat{\mcR}_n(\bmomega) &:= \frac{1}{n^2} \sum_{i, j \in [n], i \neq j} f_n(\lambda_i, \lambda_j, a_{ij}) \ell( \langle \omega_i, \omega_j \rangle , a_{ij} ),  \\ 
            \mathbb{E}[ \widehat{\mcR}_n(\bmomega) \,|\, \bm{\lambda}_n ] & := \frac{1}{n^2} \sum_{i, j \in [n], i \neq j} \sum_{x \in \{0, 1\} } \tilde{f}_n(\lambda_i, \lambda_j, x) \ell( \langle \omega_i, \omega_j \rangle, x),
        \end{align*}
        we have that 
        \begin{equation*}
                \sup_{ \bmomega \in ([-A, A]^{d } )^n } \Big| \widehat{\mcR}_n(\bmomega) - \mathbb{E}[ \widehat{\mcR}_n(\bmomega) \,|\, \bm{\lambda}_n ] \Big| = O_p\Big( \frac{ d^{q+1/2}    }{ (n\rho_n)^{1/2}   } \Big).               
        \end{equation*}
        \item (Lemma~37 of \parencite{davison_asymptotics_2021}) For the functions
        \begin{align*}
            \mathbb{E}[ \widehat{\mcR}_n^{\mcP_n}( \bmomega) \,|\, \bm{\lambda}_n  ] & := \frac{1}{n^2} \sum_{i, j \in [n], i \neq j } \sum_{x \in \{0, 1\}} \mcP_n^{\otimes 2}[ \tilde{f}_{n, x}](\lambda_i, \lambda_j) \ell( \langle \omega_i, \omega_j \rangle, x), \\ 
            \mathbb{E}[ \widehat{\mcR}^{\mcP_n}_{n, (1)}(\bmomega) \,|\, \bm{\lambda}_n ] & := \frac{1}{n^2} \sum_{i, j \in [n]} \sum_{x \in \{0, 1\}} \mcP_n^{\otimes 2}[ \tilde{f}_{n, x} ](\lambda_i, \lambda_j) \ell( \langle \omega_i, \omega_j \rangle, x)
        \end{align*}
        where $\mcP_n^{\otimes 2}[\cdot]$ is the stepping operator defined in Appendix~\ref{sec:app:loss:sbm}, we have that
        \begin{equation*}
            \sup_{\bmomega \in ([-A, A]^{d} )^n } \big|   \mathbb{E}[ \widehat{\mcR}^{\mcP_n}_{n}(\bmomega)  \,|\, \bm{\lambda}_n ] - \mathbb{E}[ \widehat{\mcR}^{\mcP_n}_{n, (1)}(\bmomega) \,|\, \bm{\lambda}_n ]  \big| = O_p\Big( \frac{d^q}{n} \Big).
        \end{equation*}
        \item (Lemma~42 and Proposition~44 of \parencite{davison_asymptotics_2021}) Suppose that $X \sim \mathrm{Multinomial}(n; p)$ where
        $p = (p_i)_{i \in [M]}$ for some $M > 0$, where the $p_i > 0$
        and $\sum_{i=1}^M p_i = 1$. Then we have that 
        \begin{equation*}
            \max_{l \in [M]} \big| \frac{n^{-1} X_l - p_l}{p_l} \big|, \max_{l, l' \in [M]} \big| \frac{n^{-2} X_l X_{l'}- p_l p_{l'}}{p_l p_{l'}} \big| = O_p\Big(  \Big(  \frac{ \log M}{n \min_{l} p_l}  \Big)^{1/2}  \Big). 
        \end{equation*}
    \end{enumerate}
\end{prop}

We also require the following lemmas, whose proof are deferred to
Appendix~\ref{sec:app:loss:useful_lemma}.

\begin{lemma} \label{app:loss:relative_convergence}
    For each $n \in \mathbb{N}$, suppose we have a compact set $\Theta_n \subseteq \mathbb{R}^{p(n)}$ for some $p(n)$ with $0 \in \Theta_n$. Moreover, suppose we have non-negative random variables $c_{ijx}^{(n)}$, $\tilde{c}_{ijx}^{(n)}$ for $i, j \in [k(n)]$ and $x \in \{0, 1\}$, and $c_{i}^{(n)}$, $\tilde{c}_{i}^{(n)}$ for $i \in [\kappa(n)]$, which satisfy the conditions
    \begin{gather*}
      \max_{i, j, x} \Big| \frac{ c_{ijx}^{(n)} - \tilde{c}_{ijx}^{(n)} }{ c_{ijx}^{(n)} } \Big| = O_p(r_n), \quad  \max_{i} \Big| \frac{ c_{i}^{(n)} - \tilde{c}_{i}^{(n)} }{ c_{i}^{(n)} } \Big| = O_p(r_n), \\ 
      \sum_{i, j, x} c_{ijx}^{(n)} = O_p(1), \quad \sum_i c_i^{(n)} = O_p(1),
    \end{gather*}
    for some non-negative sequence $r_n \to 0$, where in the above ratios we interpret $0/0 = 1$. Define non-negative continuous functions $\ell_{ijx}^{(n)}, \ell_i^{(n)} : \Theta_n \to \mathbb{R}$ such that $\ell_{ijx}^{(n)}(0), \ell_i^{(n)}(0) \leq C$ for each $i, j \in [k(n)]$, $x \in \{0, 1\}$ and $n \in \mathbb{N}$. Finally, define the functions 
    \begin{align*}
      G_n(\theta) = \sum_{i, j, x} c_{ijx}^{(n)} \ell^{(n)}_{ijx}(\theta)  + \sum_{i} c_{i}^{(n)} \ell_i^{(n)}(\theta), \qquad \widetilde{G}_n(\theta) = \sum_{i, j, x} \tilde{c}_{ijx}^{(n)} \ell^{(n)}_{ijx}(\theta)  + \sum_{i} \tilde{c}_{i}^{(n)} \ell_i^{(n)}(\theta)
    \end{align*}
    for $\theta \in \Theta_n$. Then there exists a sequence of non-empty random measurable sets $\Psi_n$ such that 
    \begin{equation*}
      \mathbb{P}\Big( \argmin_{\theta_n \in \Theta_n} G_n(\theta_n) \cup \argmin_{\theta_n \in \Theta_n} \widetilde{G}_n( \theta_n ) \subseteq \Psi_n \Big) \to 1, \qquad \sup_{ \theta_n \in \Psi_n} \big| G_n( \theta_n ) - \widetilde{G}_n( \theta_n ) \big| = O_p(r_n).
    \end{equation*}
  \end{lemma}
  
    We note that the condition that $c_{ijx}^{(n)} = (1 + O_p(r_n)) \tilde{c}_{ijx}^{(n)}$ holds uniformly over all $i, j, x$ implies that 
    \begin{equation*}
      \sum_{i, j, x} c_{ijx}^{(n)} = O_p(1) \implies \sum_{i, j, x} \tilde{c}_{ijx}^{(n)} = O_p(1) 
    \end{equation*}
    and so it suffices for either $\sum_{ijx} c_{ijx}^{(n)} = O_p(1)$ or $\sum_{ijx} \tilde{c}_{ijx}^{(n)} = O_p(1)$ to hold, and similarly either $\sum_i c_i^{(n)} = O_p(1)$ or $\sum_{i} \tilde{c}_i^{(n)} = O_p(1)$.  

  \begin{lemma} \label{app:loss:relative_convergence_integral}
    % Change to the bounded away from zero version? idk, figure out what the correct way of dealing with this is...
    For each $n \in \mathbb{N}$, suppose we have a compact set $\Theta_n \subseteq \mathbb{R}^{p(n)}$ for some $p(n)$ with $0 \in \Theta_n$. Moreover, suppose we have non-negative functions $a_{n, x}, \tilde{a}_{n, x} : [0, 1]^2 \to \mathbb{R}$ and $b_n, \tilde{b}_n : [0, 1] \to \mathbb{R}$ for $n \in \mathbb{N}$, $x \in \{0, 1\}$, such that 
    \begin{gather*}
      \max_{x} \Big\| \frac{ a_{n, x} - \tilde{a}_{n, x} }{  a_{n, x} } \Big\|_{\infty} = O(r_n), \quad \Big\| \frac{ b_n - \tilde{b}_n }{ b_n} \Big\|_{\infty} = O(r_n), \\ 
      \intsq a_{n, x} \, dl dl' = O(1), \quad \int_{[0, 1]} b_n \, dl = O(1) 
    \end{gather*}
    for some non-negative sequence $r_n \to 0$. Define non-negative continuous functions $\ell_{x} : \mathbb{R} \to \mathbb{R}$ for $x \in \{0, 1\}$, along with the functions 
    \begin{align*}
      G_n(\eta) = \intsq \sum_x a_{n, x}(l, l') \ell_x( \langle \eta(l), \eta(l') \rangle ) \, dl dl' + \int_{[0, 1]} b_n(l) \| \eta(l) \|_2^2 \, dl, \\
      \widetilde{G}_n(\eta) = \intsq \sum_x \tilde{a}_{n, x}(l, l') \ell_x( \langle \eta(l), \eta(l') \rangle ) \, dl dl' + \int_{[0, 1]} \tilde{b}_n(l) \| \eta(l) \|_2^2 \, dl
    \end{align*}
    defined over functions $\eta: [0, 1] \to \Theta_n$. For any fixed constant $C > 1$, define the set
    \begin{equation*}
      \Psi_n := \big\{ \eta \,:\, G_n(\eta) \leq C G_n(0) \text{ or } \widetilde{G}_n(\eta) \leq C \widetilde{G}_n(0) \big\}. 
    \end{equation*}
    Provided there exist minima to the functionals $G_n(\eta)$ and $\widetilde{G}_n(\eta)$, we have that
    \begin{equation*}
        \argmin_{\eta} G_n(\eta) \cup \argmin_{\eta} \widetilde{G}_n(\eta) \subseteq \Psi_n \qquad \text{ and } \qquad
      \sup_{\eta \in \Psi_n } \big| G_n(\eta) - \widetilde{G}_n(\eta) \big| = O(r_n).
    \end{equation*}
\end{lemma}

We now begin with the proof of Theorem~\ref{thm:loss_conv}. Throughout, we
understand that an exponent $p$ depends on the choice of loss function,
with $q = 1$ for the cross-entropy loss, and $q = 2$ for the
squared loss; these will then give rise to the values of $p$ in the
exponents within the theorem statement.

\subsection{Replacing the sampling probabilities}

To begin, let 
\begin{equation}
  \widehat{\mcR}_n(\bmomega) := \frac{1}{n^2} \sum_{i, j \in [n], i \neq j} f_n(\lambda_i, \lambda_j, a_{ij}) \ell( \langle \omega_i, \omega_j \rangle , a_{ij} ), \; \widehat{\mcR}_n^{\text{reg}}(\bmomega) := \frac{1}{n} \sum_{i \in [n] } \tilde{g}_n(\lambda_i) \| \omega_i \|_2^2.
\end{equation}
We then note that by applying Lemma~\ref{app:loss:relative_convergence} with
\begin{itemize}[leftmargin=*]
  \item $\Theta^{(n)} = ([-A, A]^{d })^n$, $\theta_n = (\omega_1, \ldots, \omega_n)$ with $\omega_i \in [-A, A]^{d}$;
  \item $c_{ijx}^{(n)} = \mathbb{P}( (i, j) \in S(\mcG_n) \,|\, \mcG_n) \cdot 1[ a_{ij} = x ]$ for $i \neq j$ and $0$ otherwise; $\tilde{c}_{ijx}^{(n)} = n^{-2} f_n(\lambda_i, \lambda_j, x ) 1[ a_{ij} = x]$ for $i \neq j$ and $0$ otherwise (so $\sum_{ijx} \tilde{c}_{ijx}^{(n)} = O_p(1)$ by Markov's inequality and Assumption~\ref{assume:slc});
  \item $c_i^{(n)} = \mathbb{P}( i \in \mcV(S(\mcG_n)) \,|\, \mcG_n)$, $\tilde{c}_i^{(n)} = n^{-1} \tilde{g}_n(\lambda_i)$ (so $\sum_i c_i^{(n)} = O_p(1)$ by Markov's inequality and Assumption~\ref{assume:slc});
  \item $\ell_{ijx} = \ell(\langle \omega_i, \omega_j \rangle, x)$, $r_n = s_n$;
\end{itemize}
and so there exists a sequence of sets $\Psi_n^{(1)}$, containing the minima of both $\mcR_n(\bmomega) + \xi_n \mcR_n^{\text{reg}}(\bmomega)$ and $\widehat{\mcR}_n(\bmomega) + \xi_n \widehat{\mcR}_n^{\text{reg}}(\bmomega)$ with asymptotic probability one, such that 
\begin{equation}
  \label{loss_thm:bound_1}
  \sup_{\bmomega \in \Psi_n^{(1)} } \Big| \big\{ \mcR_n(\bmomega) + \xi_n \mcR_n^{\text{reg}}(\bmomega) \big\} - \big\{ \widehat{\mcR}_n(\bmomega) + \xi_n \widehat{\mcR}_n^{\text{reg}}(\bmomega) \big\} \Big| = O_p( s_n ).
\end{equation}

\subsection{Averaging over the adjacency structure}

We now want to work with the version of the loss averaged over the realizations of the adjacency matrix of the graph $\mcG_n$, and so we introduce the function (writing $\bm{\lambda}_n = (\lambda_1, \ldots, \lambda_n)$)
\begin{equation}
  \mathbb{E}[ \widehat{\mcR}_n(\bmomega) \,|\, \bm{\lambda}_n ] := \frac{1}{n^2} \sum_{i, j \in [n], i \neq j} \sum_{x \in \{0, 1\} } \tilde{f}_n(\lambda_i, \lambda_j, x) \ell( \langle \omega_i, \omega_j \rangle, x).
\end{equation}
By Proposition~\ref{thm:app:loss:bounds}a), we have that 
\begin{equation}
  \label{loss_thm:bound_2}
  \begin{aligned}
  \sup_{ \bmomega \in ([-A, A]^{d } )^n } \Big| \big\{ \widehat{\mcR}_n(\bmomega)  & + \xi_n \widehat{\mcR}_n^{\text{reg}}(\bmomega) \big\} \\
  & - \big\{ \mathbb{E}[ \widehat{\mcR}_n(\bmomega) \,|\, \bm{\lambda}_n ] + \xi_n \widehat{\mcR}_n^{\text{reg}}(\bmomega) \big\} \Big| = O_p\Big( \frac{ d^{q+1/2}    }{ (n\rho_n)^{1/2}   } \Big).
  \end{aligned}
\end{equation}

\begin{remark}
    \label{app:rmk:tighter_bounds}
    This remark can be skipped on a first reading of the theorem proof.
    Here, we discuss how we can obtain tighter
    bounds when imposing the additional constraint
\begin{equation*}
    B^{\infty}_{n, d}(A_2) := \{ \bmomega \in ((\mathbb{R}^d)^n ) \,:\, \Omega_{ij} = B(\omega_i, \omega_j) \text{ satisfies } \| \Omega \|_{\infty} \leq A_2 \}
\end{equation*}
to the domain of optimization of the embedding vectors $\bmomega$ is
imposed. This is particularly natural when considering the squared
loss, which corresponds to optimizing the risk when averaging $(a_{ij} - \langle \omega_i, \omega_j \rangle)^2$ over all pairs $(i, j)$; as a graphon is bounded in $[0, 1]$, there is no
need for $\langle \omega_i, \omega_j \rangle$ to be outside of the
range $[0, 1]$ either. With the understanding that in this remark,
we write $\Omega_{ij} = \langle \omega_i, \omega_j \rangle$ for the
gram matrix of the embedding vectors, we define the sets 
\begin{align*}
    Z_{n, d}(A_1) & := \{ \Omega \in \mathbb{R}^{n \times n} \,:\, \Omega_{ij} = \langle \omega_i, \omega_j \rangle, \| \omega_i \|_{\infty} \leq A_1 \},\\
    Z_{n}^{\infty}(A_2) & := \{ \Omega \in \mathbb{R}^{n \times n} \,:\,
    \max_{i, j} | \Omega_{i,j} | \leq A_2 \}
\end{align*}
for the constraint set placed directly on the induced matrix $\Omega$.

We now highlight that in the proof of Theorem~30 of \parencite{davison_asymptotics_2021} (from which the bound just prior to the remark follows from), one looks to bound 
the variance term
\begin{align*}
    v(\bm{A}_n \,|\, \bm{\lambda}_n ) & \leq \frac{1}{n^2} \Bigg\{ \frac{1}{n^2} \sum_{i \neq j} f_n(\lambda_i, \lambda_j, a_{ij})^2 \big( \ell(\Omega_{ij}, a_{ij}) - \ell(\widetilde{\Omega}_{ij}, a_{ij})  \big)^2  \\ & \qquad \qquad \qquad +  \frac{1}{n^2} \sum_{i \neq j} \mathbb{E}\Big[ f_n(\lambda_i, \lambda_j, a_{ij} )^2 \big( \ell(\Omega_{ij}, a_{ij}) - \ell(\widetilde{\Omega}_{ij}, a_{ij})  \big)^2 \,|\, \bm{\lambda}_n \Big] \Bigg\} 
\end{align*}
by some metric distance between $\Omega$ and $\widetilde{\Omega}$.
To proceed, we use the alternative bound 
\begingroup
\allowdisplaybreaks
\begin{align*}
    v(\bm{A}_n \,|\, \bm{\lambda}_n ) & \leq \frac{1}{n^2} \Bigg\{ \frac{1}{n^2} \sum_{i \neq j} f_n(\lambda_i, \lambda_j, a_{ij})^2 \big( \ell(\Omega_{ij}, a_{ij}) - \ell(\widetilde{\Omega}_{ij}, a_{ij})  \big)^2  \\ & \qquad \qquad \qquad +  \frac{1}{n^2} \sum_{i \neq j} \mathbb{E}\Big[ f_n(\lambda_i, \lambda_j, a_{ij} )^2 \big( \ell(\Omega_{ij}, a_{ij}) - \ell(\widetilde{\Omega}_{ij}, a_{ij})  \big)^2 \,|\, \bm{\lambda}_n \Big] \Bigg\} \\ 
    & \leq \frac{2}{n^4} \Bigg\{ \max_{i, j} f_n(\lambda_i, \lambda_j, a_{ij})^2 \Bigg\} \cdot \losslipconst \max_{i,j} \{ |\Omega_{ij}|, | \widetilde{\Omega}_{ij} | \}^{p-1} \sum_{i,j} \big( \Omega_{ij} - \widetilde{\Omega}_{ij} \big)^2 \\ 
    & \leq \frac{ 2\losslipconst A_{2}^{q-1} }{n^4} \Bigg\{ \max_{i, j,x} f_n(\lambda_i, \lambda_j, x)^2 \Bigg\} \cdot \| \Omega - \widetilde{\Omega} \|_F^2 
\end{align*}
\endgroup
where $\| \cdot \|_F$ denotes the Frobenius norm of a matrix, and we note that for the cross-entropy loss (with $q = 1$) and the squared loss (with $q = 2$) we can write 
\begin{equation*}
    | \ell(y, x) - \ell(y', x) | \leq \losslipconst \max\{ |y|, |y'| \}^{q-1} |y - y' |
\end{equation*}
for a Lipschitz constant $\losslipconst$. We now note
that we can contain $Z_{n}^{\infty}(A_2) \cap Z_{n, d}(A_1)$ within the set
\begin{equation*}
    Z^F_{n, d}(n A_2) := \big\{  \Omega \in \mathbb{R}^{n \times n} \,:\, \Omega \text{ is of rank $\leq d$, } \| \Omega \|_F \leq n A_2 \}
\end{equation*}
We note that with respect to the Frobenius norm, this set has covering
number
\begin{equation*}
    N( Z^F_{n, d}(n A_2), \| \cdot \|_F, \epsilon) \leq \Big( \frac{ C n A_2 }{\epsilon}  \Big)^{2nd}
\end{equation*}
for some absolute constant $C > 0$, and therefore by a similar argument to Lemma~41 in \parencite{davison_asymptotics_2021}, it will be possible to
conclude that $\gamma_2( Z^F_{n, d}(n A_2), \| \cdot \|_F ) \leq C' n^{3/2} d^{1/2}$ for some constant $C'$ depending on $A_1$ and $A_2$, which can then be plugged
into the bound given in Theorem~30 of \parencite{davison_asymptotics_2021}.
For the sampling schemes we consider, $\max_{i,j,x} f_n(\lambda_i, \lambda_j, x) = O(\rho_n^{-2})$, and consequently the bound we obtain is of the order
$(d/n \rho_n^2)^{1/2}$, rather than $(d^3/n\rho_n)^{1/2}$. This bound
is particularly effective in the non-sparse regime; in the sparse regime, one
would hope for a bound of the form $(d/n\rho_n)^{1/2}$, but we are unaware
as to whether such a bound is achievable.
\end{remark}

\subsection{Using a SBM approximation}
\label{sec:app:loss:sbm}

We begin by working in the scenario where Assumption~\ref{assume:reg}b) holds.
Letting $\mcP$ be a partition of $[0, 1]$ into $\kappa$ parts, say $\mcP = (A_{1}, \ldots, A_{\kappa})$, we introduce the stepping operators defined by 
\begin{gather*}
  \mcP^{\otimes 2}[h](x, y) = \frac{1}{ |A_l| |A_{l'}| } \int_{A_l \times A_{l'} } h(x', y') \, dx' dy' \text{ if } (x, y) \in A_l \times A_{l'}, \\
  \mcP[h](x) = \frac{1}{|A_l|} \int_{A_l} h(x') \, dx' \text{ if } x \in A_l
\end{gather*}
for any symmetric measurable function $h: [0, 1]^2 \to \mathbb{R}$ and measurable function $h: [0, 1] \to \mathbb{R}$ respectively. With this, let $\mcP_n = (A_{n1}, \ldots, A_{n\kappa(n)})$ be a sequence of partitions containing $\kappa(n)$ intervals of size $|A_{nl}| \asymp n^{-\alpha}$ for some constant $\alpha > 0$, and then introduce the functions
\begin{align}
  \mathbb{E}[ \widehat{\mcR}_n^{\mcP_n}( \bmomega) \,|\, \bm{\lambda}_n  ] & := \frac{1}{n^2} \sum_{i, j \in [n], i \neq j } \sum_{x \in \{0, 1\}} \mcP_n^{\otimes 2}[ \tilde{f}_{n, x}](\lambda_i, \lambda_j) \ell( \langle \omega_i, \omega_j \rangle, x), \\ 
  \widehat{\mcR}_n^{\text{reg}, \mcP_n}(\bmomega) & := \frac{1}{n} \sum_{i \in [n] } \mcP_n[\tilde{g}_n](\lambda_i) \| \omega_i \|_2^2,
\end{align}
where we make the abbreviation $\tilde{f}_{n, x}(\lambda_i, \lambda_j) := \tilde{f}_n(\lambda_i, \lambda_j, x)$. We note that as $\tilde{f}_{n, x}$ and $\tilde{g}_n$ are uniformly bounded away from zero by $M^{-1}$, and because they are H\"{o}lder of exponent $\beta$, we can apply Lemma~C.6 of \parencite{wolfe_nonparametric_2013} to obtain that 
\begin{equation} 
  \label{loss_thm:holder_bounds}
  \Big\| \frac{ \tilde{f}_{n, x} - \mcP_n^{\otimes 2}[ \tilde{f}_{n, x} ] }{  \tilde{f}_{n, x}   } \Big\|_{\infty}  = O( n^{-\alpha \beta}), \qquad \Big\| \frac{ \tilde{g}_n - \mcP_n[\tilde{g}_n] }{  \tilde{g}_n  } \Big\|_{\infty} = O( n^{-\alpha \beta} ).
\end{equation}
This, along with the uniform boundedness conditions on the $\tilde{f}_{n, x}$ and $\tilde{g}_n$ given in Assumption~\ref{assume:reg}, allow us to apply Lemma~\ref{app:loss:relative_convergence} to find that there exists a sequence of sets $\Psi_n^{(2)}$ for which the minima of both $\mathbb{E}[ \widehat{\mcR}_n(\bmomega) \,|\, \bm{\lambda}_n ] + \xi_n \widehat{\mcR}_n^{\text{reg}}(\bmomega)$ and $\mathbb{E}[ \widehat{\mcR}^{\mcP_n}_n(\bmomega) \,|\, \bm{\lambda}_n ] + \xi_n \widehat{\mcR}_n^{\text{reg}, \mcP_n}(\bmomega)$ are contained within it with asymptotic probability $1$, and 
\begin{equation}
  \label{loss_thm:bound_3}
  \begin{aligned}
  \sup_{\bmomega \in \Psi_n^{(2)} } \Big| \big\{ \mathbb{E}[ \widehat{\mcR}_n(\bmomega) \,|\, \bm{\lambda}_n ] & + \xi_n \widehat{\mcR}_n^{\text{reg}}(\bmomega) \big\} \\ 
  & - \big\{ \mathbb{E}[ \widehat{\mcR}^{\mcP_n}_n(\bmomega) \,|\, \bm{\lambda}_n ] + \xi_n \widehat{\mcR}_n^{\text{reg}, \mcP_n}(\bmomega) \big\} \Big| = O_p( n^{-\alpha \beta }).
  \end{aligned}
\end{equation}
Note that in the case where Assumption~\ref{assume:reg}a) holds, this step
is not necessary, and so we can take the above bound to be equal to zero.

\subsection{Adding the contribution of the diagonal term}

We note that in the definition of $\mathbb{E}[ \widehat{\mcR}^{\mcP_n}_n(\bmomega) \,|\, \bm{\lambda}_n ]$, the summation does not include any $i = j$ terms; if we introduce 
\begin{equation}
  \label{loss_thm:add_diag}
  \mathbb{E}[ \widehat{\mcR}^{\mcP_n}_{n, (1)}(\bmomega) \,|\, \bm{\lambda}_n ] := \frac{1}{n^2} \sum_{i, j \in [n]} \sum_{x \in \{0, 1\}} \mcP_n^{\otimes 2}[ \tilde{f}_{n, x} ](\lambda_i, \lambda_j) \ell( \langle \omega_i, \omega_j \rangle, x),
\end{equation}
then by Proposition~\ref{thm:app:loss:bounds}b), we have that 
\begin{equation}
\begin{aligned}
  \label{loss_thm:bound_4}
  \sup_{\bmomega \in ([-A, A]^{d} )^n } \big|  \{   \mathbb{E}[ \widehat{\mcR}^{\mcP_n}_{n}(\bmomega) & \,|\, \bm{\lambda}_n ] + \xi_n \widehat{\mcR}_n^{\text{reg}, \mcP_n}(\bmomega) \}  \\ & - \{  \mathbb{E}[ \widehat{\mcR}^{\mcP_n}_{n, (1)}(\bmomega) \,|\, \bm{\lambda}_n ] + \xi_n \widehat{\mcR}_n^{\text{reg}, \mcP_n}(\bmomega) \} \big| = O_p\Big( \frac{d^q}{n} \Big).
\end{aligned}
\end{equation}

\subsection{Linking minimizing embeddings to minimizing kernels}
\label{sec:app:loss:min_kernel}

We now want to reason about the minima of the function $\mathbb{E}[ \widehat{\mcR}^{\mcP_n}_{n, (1)}(\bmomega) \,|\, \bm{\lambda}_n ] + \xi_n \widehat{\mcR}_n^{\text{reg}, \mcP_n}(\bmomega)$. We denote
\begin{gather*}
  p_n(l) := |A_{nl}|, \qquad \mcA_n(l) := \{ i \in [n] \,:\, \lambda_i \in A_{nl} \}, \qquad \widehat{p}_n(l) := n^{-1} | \mcA_n(l) |, \\
  c_{f, n}(l, l', x) := \frac{1}{p_n(l) p_n(l') } \int_{A_{nl} \times A_{nl'} } \tilde{f}_n(\lambda, \lambda', x) \, d\lambda d\lambda', \qquad c_{g, n}(l) := \frac{1}{p_n(l) } \int_{A_{nl}} \tilde{g}_n(\lambda) \, d\lambda.
\end{gather*}
Consider writing 
\begin{equation}
  \tilde{\omega}_i = \frac{1}{ | \mcA_n(l) | } \sum_{j \in \mcA_n(l) } \omega_j \text{ if } i \in \mcA_n(l), \qquad \widetilde{\bm{\omega}}_n = (\tilde{\omega}_i)_{i \in [n] },
\end{equation}
given any set of embedding vectors $\bmomega$. As $\ell(y, x)$ is strictly convex in $y \in \mathbb{R}$ for $x \in \{0, 1\}$ and $\| \cdot \|_2^2$ is also strictly convex, by Jensen's inequality we have that
\begingroup
\allowdisplaybreaks
\begin{align*}
  \mathbb{E}[ \widehat{\mcR}^{\mcP_n}_{n, (1)}&(\bmomega) \,|\, \bm{\lambda}_n ] + \xi_n \widehat{\mcR}_n^{\text{reg}, \mcP_n}(\bmomega) \\
  & = \sum_{l, l' \in [\kappa(n)] } \widehat{p}_n(l) \widehat{p}_n(l') \sum_x \Big\{  \frac{c_{f, n}(l, l', x) }{ | \mcA_n(l) | | \mcA_n(l') | } \sum_{ \substack{ i \in \mcA_n(l) \\ j \in \mcA_n(l') } } \ell( \langle \omega_i, \omega_j \rangle, x) \Big\} \\
  &\qquad \qquad \qquad + \sum_{l \in [\kappa(n) ] } \widehat{p}_n(l) \frac{c_{g, n}(l) }{ | \mcA_n(l) | } \sum_{ i \in \mcA_n(l) } \| \omega_i \|_2^2 \\ 
  & \geq \sum_{ l, l' \in [\kappa(n) ] } \widehat{p}_n(l) \widehat{p}_n(l') \sum_x c_{f, n}(l, l', x) \ell\Big(  \sum_{\substack{ i \in \mcA_n(l) \\ j \in \mcA_n(l') } }  \langle \omega_i, \omega_j \rangle, x \Big) \\ 
  & \qquad \qquad \qquad + \sum_{l \in [\kappa(n)] } \widehat{p}_n(l) c_{g, n}(l) \Big\| \frac{1}{| \mcA_n(l) | } \sum_{i \in \mcA_n(l) } \omega_i \Big\|_2^2 \\ 
  & = \sum_{ l, l' \in [\kappa(n) ] } \widehat{p}_n(l) \widehat{p}_n(l') \sum_x c_{f, n}(l, l', x) \ell\Big( \langle \sum_{i \in \mcA_n(l) } \frac{\omega_i}{  | \mcA_n(l) | }, \sum_{j \in \mcA_n(l') } \frac{\omega_j}{ | \mcA_n(l') | } \rangle, x \Big) \\
  & \qquad \qquad + \sum_{l \in [\kappa(n)] } \widehat{p}_n(l) c_{g, n}(l) \Big\| \frac{1}{| \mcA_n(l) | } \sum_{i \in \mcA_n(l) } \omega_i \Big\|_2^2 \\
  & = \mathbb{E}[ \widehat{\mcR}^{\mcP_n}_{n, (1)}( \widetilde{\bm{\omega}}_n ) \,|\, \bm{\lambda}_n ] + \xi_n \widehat{\mcR}_n^{\text{reg}, \mcP_n}( \widetilde{\bm{\omega}}_n ),
\end{align*}
\endgroup
with equality if and only if the $\omega_i$ are equal across each of the sets $\mcA_n(l)$. In particular, this means that to minimize $\mathbb{E}[ \widehat{\mcR}^{\mcP_n}_{n, (1)}(\bmomega) \,|\, \bm{\lambda}_n ] + \xi_n \widehat{\mcR}_n^{\text{reg}, \mcP_n}(\bmomega)$, if we define 
\begin{align*}
  \widehat{I}_n^{\mcP_n}(\tilde{\omega}_1, \ldots, \tilde{\omega}_{\kappa(n) } ) & := \sum_{l, l' \in [\kappa(n) ] }  \widehat{p}_n(l) \widehat{p}_n(l') \sum_x c_{f, n}(l, l', x) \ell( \langle \tilde{\omega}_l, \tilde{\omega}_{l'} \rangle, x ) \\ 
  \widehat{I}_n^{\text{reg}, \mcP_n}( \tilde{\omega}_1, \ldots, \tilde{\omega}_{\kappa(n) } ) & := \sum_{l \in [\kappa(n)] } \widehat{p}_n(l) c_{g, n}(l) \| \tilde{\omega}_l \|_2^2,
\end{align*}
then it suffices to minimize $\widehat{I}_n^{\mcP_n}(\tilde{\omega}_1, \ldots, \tilde{\omega}_{\kappa(n) } )  + \xi_n \widehat{I}_n^{\text{reg}, \mcP_n}(\tilde{\omega}_1, \ldots, \tilde{\omega}_{\kappa(n) } )$, as the $\tilde{\omega}_i$ are constant across $i \in \mcA_n(l)$. In other words, the above argument has just showed that 
\begin{equation}
  \label{loss_thm:bound_5}
  \begin{split}
  \min_{\bmomega \in ([-A, A]^{d } )^n } & \big\{ \mathbb{E}[ \widehat{\mcR}^{\mcP_n}_{n, (1)}(\bmomega)  \,|\, \bm{\lambda}_n ] + \xi_n \widehat{\mcR}_n^{\text{reg}, \mcP_n}(\bmomega) \big\} \\ 
  & = \min_{ (\tilde{\omega}_i) \in ([-A, A]^{d})^{\kappa(n)} } \big\{ \widehat{I}_n^{\mcP_n}(\tilde{\omega}_1, \ldots, \tilde{\omega}_{\kappa(n) } )  + \xi_n \widehat{I}_n^{\text{reg}, \mcP_n}(\tilde{\omega}_1, \ldots, \tilde{\omega}_{\kappa(n) } ) \big\}.
  \end{split}
\end{equation}
We note that $\widehat{I}_n^{\mcP_n}$ and $\widehat{I}_n^{\text{reg}, \mcP_n}$ are stochastic, as they depend on the random variables $\widehat{p}_n(l)$. To remove the stochasticity, we introduce the functions 
\begin{align*}
  I_n^{\mcP_n}(\tilde{\omega}_1, \ldots, \tilde{\omega}_{\kappa(n) } ) & := \sum_{l, l' \in [\kappa(n) ] }  p_n(l) p_n(l') \sum_x c_{f, n}(l, l', x) \ell( \langle \tilde{\omega}_l, \tilde{\omega}_{l'} \rangle, x ) \\ 
  I_n^{\text{reg}, \mcP_n}( \tilde{\omega}_1, \ldots, \tilde{\omega}_{\kappa(n) } ) & := \sum_{l \in [\kappa(n)] } p_n(l) c_{g, n}(l) \| \tilde{\omega}_l \|_2^2.
\end{align*}
As by Proposition~\ref{thm:app:loss:bounds}c) we have that 
\begin{gather*}
  \max_{l, l' \in [\kappa(n) ] } \Big| \frac{ \widehat{p}_n(l) \widehat{p}_n(l') - p_n(l) p_n(l') }{ p_n(l) p_n(l')  } \Big| = \begin{cases} 
    O_p\Big( \Big( \frac{ \log \kappa}{n} \Big)^{1/2} \Big) & \text{ under Assumption~\ref{assume:reg}a) } \\ O_p\Big(  \frac{ \sqrt{\log n} }{ n^{1/2 - \alpha/2} }  \Big) & \text{ under Assumption~\ref{assume:reg}b)}  \end{cases} \\ 
    \max_{l \in [\kappa(n) ]} \Big| \frac{ \widehat{p}_n(l) - p_n(l) }{ p_n(l) } \Big| = \begin{cases} 
        O_p\Big( \Big( \frac{ \log \kappa}{n} \Big)^{1/2} \Big) & \text{ under Assumption~\ref{assume:reg}a) } \\ O_p\Big(  \frac{ \sqrt{\log n} }{ n^{1/2 - \alpha/2} }  \Big) & \text{ under Assumption~\ref{assume:reg}b)}  \end{cases}
\end{gather*}
and moreover that the $\widehat{p}_n(l)$ and $p_n(l)$ sum to $1$, we can apply Lemma~\ref{app:loss:relative_convergence} to argue that there exists a sequence of sets $\Psi_n^{(3)}$ which contains both the minima of $\widehat{I}_n^{\mcP_n}( (\tilde{\omega}_i)_{i=1}^{\kappa(n) } )  + \xi_n \widehat{I}_n^{\text{reg}, \mcP_n}( (\tilde{\omega}_i)_{i=1}^{\kappa(n) } )$ and $I_n^{\mcP_n}( (\tilde{\omega}_i)_{i=1}^{\kappa(n) } )  + \xi_n I_n^{\text{reg}, \mcP_n}( (\tilde{\omega}_i)_{i=1}^{\kappa(n) } )$ with asymptotic probability $1$, and that 
\begin{align} \nonumber 
  \sup_{ (\tilde{\omega}_i )_i \in \Psi_n^{(3)} } \Big| & \big\{ \widehat{I}_n^{\mcP_n}( (\tilde{\omega}_i)_{i=1}^{\kappa(n) } )  + \xi_n \widehat{I}_n^{\text{reg}, \mcP_n}( (\tilde{\omega}_i)_{i=1}^{\kappa(n) } ) \big\} - \big\{  I_n^{\mcP_n}( (\tilde{\omega}_i)_{i=1}^{\kappa(n) } )  + \xi_n I_n^{\text{reg}, \mcP_n}( (\tilde{\omega}_i)_{i=1}^{\kappa(n) } ) \big\} \Big| \\
   & = \begin{cases}  
    O_p\Big( \Big( \frac{ \log \kappa}{n} \Big)^{1/2} \Big) & \text{ under Assumption~\ref{assume:reg}a) } \\ O_p\Big(  \frac{ \sqrt{\log n} }{ n^{1/2 - \alpha/2} }  \Big) & \text{ under Assumption~\ref{assume:reg}b)} 
  \end{cases} \label{loss_thm:bound_6}
\end{align}
To transition from embedding vectors to kernels $K(l, l') = \langle \eta(l), \eta(l') \rangle$ for $\eta: [0, 1] \to [-A, A]^{d}$, we note that as we can write 
\begin{align*}
  \mcI_n^{\mcP_n}[K] & = \sum_{l, l' \in [\kappa(n) ]} p_n(l) p_n(l') \sum_{x} \frac{ c_{f, n}(l, l', x) }{p_n(l) p_n(l') } \int_{A_{nl} \times A_{nl'} } \ell( \langle \eta(\lambda), \eta(\lambda') \rangle, x) \, d\lambda d\lambda', \\
  \mcI_n^{\text{reg}, \mcP_n}[K] & = \sum_{l \in [\kappa(n)] } p_n(l) \frac{ c_{g, n}(l) }{ p_n(l) } \int_{A_{nl} } \| \eta( \lambda ) \|_2^2 \, d \lambda,
\end{align*}
by the same Jensen's inequality argument used to obtain \eqref{loss_thm:bound_5}, we get that 
\begin{equation}
  \begin{split}
    \min_{ (\tilde{\omega}_i)_i \in ([-A, A]^{d})^{\kappa(n) } } \big\{  I_n^{\mcP_n}( (\tilde{\omega}_i)_{i=1}^{\kappa(n) } )  & + \xi_n I_n^{\text{reg}, \mcP_n}( (\tilde{\omega}_i)_{i=1}^{\kappa(n) } ) \big\} \\ 
    & = 
  \min_{K \in \mcZ_d^{\geq 0}(A) } \big\{ \mcI_n^{\mcP_n}[K] + \xi_n \mcI_n^{\text{reg}, \mcP_n}[K] \big\}, 
  \end{split}
\end{equation}
where the correspondence between the minimizing $K(l, l') = \langle \eta(l), \eta(l') \rangle$ and $\tilde{\omega}_i$ is given by $\eta(l) = \tilde{\omega}_i$ for $i \in A_{nl}$. 

The final step is to remove the approximation terms $\mcP_n^{\otimes 2}[\tilde{f}_{n, x}]$ and $\mcP_n[\tilde{g}_n]$ from $\mcI_n^{\mcP_n}[K]$ and $\mcI_n^{\text{reg}, \mcP_n}[K]$ in order to get to $\mcI_n[K]$ and $\mcI_n^{\text{reg}}[K]$. To do so, we can use \eqref{loss_thm:holder_bounds} and Lemma~\ref{app:loss:relative_convergence_integral} to obtain that there exists a set $\Psi_n^{(4)}$ containing both the minima of $\mcI_n[K] + \xi_n \mcI_n^{\text{reg}}[K]$ and $\mcI_n^{\mcP_n}[K] + \xi_n \mcI_n^{\text{reg}, \mcP_n}[K]$ (which exist by Proposition~\ref{app:min_exist}) and 
\begin{equation}
  \label{loss_thm:bound_7}
  \sup_{K \in \Psi_n^{(4)} } \Big| \big\{ \mcI_n^{\mcP_n}[K] + \xi_n \mcI_n^{\text{reg}, \mcP_n}[K] \big\} - \big\{ \mcI_n[K] + \xi_n \mcI_n^{\text{reg}}[K] \big\} \Big| = O( n^{-\alpha \beta} ).
\end{equation}

\subsection{Combining to obtain rates of convergence}

To conclude, we first note that given uniform convergence bounds of two functions on a set containing both of their minima, we can argue convergence of their minimal values; indeed if a set $A$ contains minima $x_f$ and $x_g$ to some functions $g$, then 
\begin{equation*}
  \min_x f(x) - \min_x g(x) = \min_x f(x) - g(x_g) \leq f(x_g) - g(x_g) \leq \sup_{x \in A} | f(x) - g(x) |,
\end{equation*}
and similarly so for $\min_x g(x) - \min_x f(x)$. (We note that Proposition~\ref{app:min_exist} argues that all the relevant infimal values of the minimizers of the $\mcI_n[K] + \xi_n \mcI_n^{\text{reg}}[K]$ are attained.) Therefore, using this fact and chaining together the bounds in \eqref{loss_thm:bound_1},~\eqref{loss_thm:bound_2},~\eqref{loss_thm:bound_3},~\eqref{loss_thm:bound_4},~\eqref{loss_thm:bound_5},~\eqref{loss_thm:bound_6}~and~\eqref{loss_thm:bound_7}, we get when Assumption~\ref{assume:reg}b) holds that
\begin{equation}
  \begin{split}
  \Big| \min_{\bmomega \in ([-A, A]^{d } )^n } \big\{ \mcR_n(\bmomega) + \xi_n \mcR_n^{\text{reg}}(\bmomega) \big\} & - \min_{K \in \mcZ_d^{\geq 0}(A) }  \big\{ \mcI_n[K] + \xi_n \mcI_n^{\text{reg}}[K] \big\} \Big| \\
  & = O_p\Big( s_n + \frac{ d^{3/2} \rho_n^{-1/2} }{n^{1/2} } + n^{-\alpha \beta} + \frac{ \sqrt{ \log n}}{n^{1/2 - \alpha/2} }   \Big).
  \end{split}
\end{equation}
(We note that the $O_p(d^q/n)$ term from \eqref{loss_thm:bound_4} is negligible.) To conclude, we simply pick an optimal choice of $\alpha$, which we take to be $\alpha = 1/(1 + 2\beta)$, which gives the stated bound. In the
case where Assumption~\ref{assume:reg}a) holds, the term from the SBM
approximation disappears and the $\sqrt{ \log n}/n^{1/2 - \alpha/2}$ term
becomes $(\log \kappa / n)^{1/2}$, giving the stated bound in this regime.

\subsection{Proofs of useful lemmata}
\label{sec:app:loss:useful_lemma}

\begin{proof}[Proof of Lemma~\ref{app:loss:relative_convergence}]
  We begin by noting that as each of the $G_n(\theta)$ and $\widetilde{G}_n(\theta)$ are continuous functions defined on compact sets, the minima sets of each of the functions are non-empty. We now define the sets
  \begin{align*}
    \Psi_n & := \Bigg\{ \theta_n \in \Theta^{(n)} \,:\, G_n( \theta_n ) \leq 2 C \sum_{i, j, x} c_{ijx}^{(n)} + 2 C \sum_i c_i^{(n)} \Bigg\}, \\ 
    \widetilde{\Psi}_n & := \Bigg\{ \theta_n \in \Theta^{(n)} \,:\, \widetilde{G}_n( \theta_n ) \leq C \sum_{i, j, x} \tilde{c}_{ijx}^{(n)} + C \sum_i \tilde{c}_i^{(n)} \Bigg\},
  \end{align*}
  and note that $0 \in \Psi_n$, $0 \in \widetilde{\Psi}_n$ for each $n$, and therefore we also have that $\argmin_{\theta_n \in \Theta^{(n)}} G_n(\theta) \subseteq \Psi_n$ and $\argmin_{\theta_n \in \Theta^{(n)}} \widetilde{G}_n(\theta) \subseteq \widetilde{\Psi}_n$. We now want to argue that $\mathbb{P}( \widetilde{\Psi}_n \subseteq \Psi_n ) \to 1$ as $n \to \infty$. Note that for any $\theta_n \in \widetilde{\Psi}_n$, we have that 
  \begin{align*}
    G_n(\theta_n) & = \sum_{i, j, x} \frac{ c_{ijx}^{(n)} }{ \tilde{c}_{ijx}^{(n)} } \tilde{c}_{ijx}^{(n)} \ell^{(n)}_{ijx}(\theta) + \sum_i \frac{ c_i^{(n)} }{ \tilde{c}_i^{(n) } } \tilde{c}_i^{(n)} \ell_i^{(n)}(\theta) \\ 
    & \leq ( 1 + O_p(r_n) ) \widetilde{G}_n(\theta_n) \leq C (1 + O_p(r_n) ) \Big\{ \sum_{i, j, x} \tilde{c}_{ijx}^{(n)} + \sum_i  \tilde{c}_i^{(n)} \Big\}.
  \end{align*}
  As by Cauchy's third inequality we have that 
  \begin{equation*}
    \frac{ \sum_{i, j, x} \tilde{c}_{ijx}^{(n)}   }{ \sum_{i, j, x} c_{ijx}^{(n) }     } \leq \max_{i, j, x} \frac{   \tilde{c}_{ijx}^{(n)}  }{ c_{ijx}^{(n)}    } = 1 + O_p(r_n),
  \end{equation*}
  and similarly $\sum_{i} \tilde{c}_i^{(n)} \leq (1 + O_p(r_n)) \sum_i c_i^{(n)}$, it follows that 
  \begin{equation*}
    G_n(\theta_n) \leq C (1 + O_p(r_n) ) \Big\{ \sum_{i, j, x} \tilde{c}_{ijx}^{(n)} + \sum_i  \tilde{c}_i^{(n)} \Big\} \stackrel{w.h.p}{ \leq } 2 C  \Big\{ \sum_{i, j, x} c_{ijx}^{(n)} + \sum_i  c_i^{(n)} \Big\}
  \end{equation*}
  once $n$ is sufficiently large, and therefore $\theta_n \in \widetilde{\Psi}_n$ for $n$ sufficiently large. In particular, as the above argument holds freely of the choice of $\theta_n$, we have that $\widetilde{\Psi}_n \subseteq \Psi_n$ with asymptotic probability one. With this, we now note that $\sup_{\theta_n \in \Psi_n} G_n(\theta_n) = O_p(1)$ (due to the condition on the sum of the $c_{ijx}^{(n)}$ and $c_i^{(n)}$), and consequently we have that for all $\theta_n \in \Psi_n$
  \begin{align*}
    \big| G_n(\theta_n) - \widetilde{G}_n(\theta_n) \big| & \leq \max_{i, j, x } \Big| \frac{ c_{ijx}^{(n)} - \tilde{c}_{ijx}^{(n)} }{ c_{ijx}^{(n)} } \Big| \cdot \sum_{i, j, x} c_{ijx}^{(n)} \ell^{(n)}_{ijx}(\theta) + \max_{i} \Big| \frac{ c_i^{(n)} - \tilde{c}_i^{(n) } }{ c_i^{(n)} } \Big| \cdot \sum_i c_i^{(n)} \ell^{(n)}_i(\theta) \\
    & \leq O_p(r_n) G_n(\theta_n) \leq O_p(r_n)
  \end{align*}
  with the bound holding uniformly over the choice of $\theta_n$, giving the stated conclusion.
\end{proof}

\begin{proof}[Proof of Lemma~\ref{app:loss:relative_convergence_integral}]
  The proof follows the exact same style of argument as in Lemma~\ref{app:loss:relative_convergence}, so we skip repeating the details.
\end{proof}
\section{Proof of Theorem~\ref{thm:embed_conv}}
\label{sec:app:embed_conv}

Before proving any results, we introduce some useful facts from functional analysis; the terminology and basic properties used below can be found in standard references such as e.g. \parencite{barbu_convexity_2012,brezis_functional_2011,peypouquet_convex_2015}. Throughout, we will write $\mu_n$ to refer to the measure $\mu_n(A) := \int_A \tilde{g}_n d\mu$, define for all Borel sets of $[0, 1]$, where $\mu$ is the regular Lebesgue measure on $[0, 1]$, and write e.g. $L^2([0, 1], \mu_n)$ or $L^2(\mu_n)$ for the associated Lebesgue space of square integrable random variables. We note that as it assumed that the $\tilde{g}_n$ are uniformly bounded away from zero and uniformly bounded above by Assumption~\ref{assume:reg}, $h \in L^2(\mu_n)$ iff $h \in L^2(\mu)$. For any function $K \in L^2([0, 1]^2, \mu_n^{\otimes 2})$ (where we write $\mu_n^{\otimes 2}$ for the product measure of $\mu_n$ with itself), we introduce the associated operator
\begin{equation}
  \label{app:min:k2op}
  T_K : L^2(\mu_n) \to L^2(\mu_n), \qquad T_K[f](x) = \int_0^1 K(x, y) f(y) d\mu_n(y).
\end{equation}
The above operator is Hilbert-Schmidt, where all Hilbert-Schmidt operators $L^2(\mu_n) \to L^2(\mu_n)$ can be written in the above form for some kernel $K \in L^2([0, 1]^2, \mu_n^{\otimes 2})$; moreover $T_K$ is self-adjoint (so $T_K^* = T_K$) iff $K$ is symmetric. The above identification corresponds to an isometric isomorphism between the Hilbert spaces $L^2(\mu_n)$ and the Hilbert-Schmidt operators, via \parencite[e.g][Theorem~8.4.8]{heil_metrics_2018} the formula
\begin{equation}
  \mathrm{Tr}(T_K^* T_L) = \langle K, L \rangle_{L^2([0, 1]^2, \mu_n^{\otimes 2})} = \int_{[0, 1]^2} \overline{K(y, x)} L(x, y) \, d\mu_n(x) d\mu_n(y),
\end{equation}
which gives rise to the corresponding norm formula $\|T_K \|_{HS}^2 = \| K \|_{L^2(\mu_n)}$. Writing $\mcS(L^2(\mu_n))$ for the space of linear operators $L^2(\mu_n) \to L^2(\mu_n)$ with finite trace or nuclear norm $\| T \|_1 < \infty$ (referred to as the space of trace class operators), $\mcK(L^2(\mu_n))$ for the space of compact linear operators $L^2(\mu_n) \to L^2(\mu_n)$, and $\mcB(L^2(\mu_n))$ for the space of bounded linear operators $L^2(\mu_n) \to L^2(\mu_n)$ with norm $\| \cdot \|_{\text{op}}$, we have that \parencite[e.g][Theorem~3.3.9]{sunder_operators_2016}
\begin{itemize}
  \item $\mcS(L^2(\mu_n)) \cong ( \mcK(L^2(\mu_n)))^*$ via the mapping $T \in \mcS(L^2(\mu_n)) \mapsto [ A \mapsto \mathrm{Tr}(AT)]$;
  \item $\mcB(L^2(\mu_n)) \cong ( \mcS(L^2(\mu_n)) )^*$ via the mapping $A \in \mcB(L^2(\mu_n)) \mapsto [ T \mapsto \mathrm{Tr}(AT) ]$.
\end{itemize}
Consequently, this allows us to argue that the trace norm $\| \cdot \|_1$ is weak* lower semi-continuous on $\mcS(L^2(\mu_n))$, and that its closed level sets are weak* compact by Banach-Alaoglu. We also note that we have the inclusions 
\begin{equation*}
  \{\text{finite rank}\} \subset \{\text{trace class}\} \subset \{\text{Hilbert-Schmidt}\} \subset \{\text{compact operators}\} \subset \{\text{bounded operators}\}.
\end{equation*}
Operators which satisify $\langle T_K[f], f \rangle \geq 0$ for all $f \in L^2(\mu_n)$ are called positive operators\footnote{We note that unlike in finite dimensions, we usually do not distingush between operators which are positive definite as compared to being only non-negative definite.}; for positive operators we have that the trace norm is equal simply to the trace. With this, we now are in a position to prove the results needed to talk about minimizers of $\mcI_n[K] + \xi_n \mcI_n^{\text{reg}}[K]$ over various sets of functions $K$.

\begin{prop}
  \label{app:min_exist}
  For $K \in \mcZ_{\text{fr}}^{\geq 0}(A) := \cup_{d \geq 1} \mcZ_d^{\geq 0}(A)$, writing $K\llp = \sum_{i=1}^d \eta_i(l) \eta_i(l')$ for some $d$ and functions $\eta_i : [0, 1] \to [-A, A]$, we define 
  \begin{equation*}
    \mcI_n[K] = \intsq \sum_{x \in \{0, 1\}} \tilde{f}_n(l, l', x) \ell(K\llp, x) \, dl dl', \qquad \mcI_n^{\text{reg}}[K] = \int_{[0, 1]} \| \eta_i(l) \|_2^2 \cdot \tilde{g}_n(l) \, dl,
  \end{equation*}
  where we recall that $\tilde{f}_n$ and $\tilde{g}_n$ are as given in Assumptions~\ref{assume:slc}~and~\ref{assume:reg}, and $\ell(y, x)$ is either the cross-entropy loss or the squared loss function; we introduce a variable
  $q$ for which $q = 1$ applies to the cross-entropy loss, and $q=2$ for the squared loss. Treat $n$ as fixed. Write $\mu_n$ for the measure $\mu_n(A) := \int_A \tilde{g}_n \, d\mu$ where $\mu$ is the Lebesgue measure on $[0, 1]$. Then we have the following: 
  \begin{enumerate}[label=\roman*), leftmargin=*]
    \item For $K \in \mcZ_{\text{fr}}^{\geq 0}(A)$, $\mcI_n^{\text{reg}}[K] = \mathrm{Tr}[T_K]$ where 
    \begin{equation*}
      T_K : L^2(\mu_n) \to L^2(\mu_n), \qquad T_K[f](x) = \int_0^1 K(x, y) f(y) \tilde{g}_n(y) \, dy.
    \end{equation*}
    \item The set $\mcZ_{\text{fr}}^{\geq 0}(A)$ is free of $A$, and 
    so we can let $Z^{\geq 0}$ denote the weak* closure of $\mcZ_{\text{fr}}^{\geq 0}(A)$ in $\mcS(L^2(\mu_n))$. 
    \item $\mcI_n^{\text{reg}}[K]$ extends uniquely to a weak* lower semi-continuous function, namely the trace, on $\mcZ^{\geq 0}$, and to
    the larger domain of the positive trace-class operators on $L^2(\mu_n)$. Consequently, we write $\mcI_n^{\text{reg}}[K] = \mathrm{Tr}[T_K]$ for $K \in \mcZ^{\geq 0}$, or more generally any symmetric function $K$ for which $T_K$ is positive. 
    \item $\mcI_n[K]$ is finite for all symmetric functions $K$ for which $T_K$ is a positive operator and $\mcI_n^{\text{reg}}[K] < \infty$; $\mcI_n[K]$ is strictly convex in $K$; and $\mcI_n[K]$ is weak* lower semi-continuous with respect to the topology on $\mcS(L^2(\mu_n))$. 
    \item We have the local Lipschitz property
    \begin{align*}
        \big| \mcI_n[K] - \mcI_n[L] \big| & \leq 2 M^3\Big( \| K \|_{L^2(\mu_n^{\otimes 2})} + \| L \|_{L^2(\mu_n^{\otimes 2})} \Big)^{q-1} \| K - L \|_{L^2(\mu_n^{\otimes 2})} \\
        & = 2 M^3\Big( \| T_K \|_{\text{HS}} + \| T_L \|_{\text{HS}} \Big)^{q-1} \| T_K - T_L \|_{\text{HS}}.
    \end{align*}
    \item For any $\xi_n \geq 0$, we have that $\mcI_n[K] + \xi_n \mcI_n^{\text{reg}}[K]$ is a strictly convex function in $K$, which is weak* lower semi-continuous with respect to the topology on $\mcS(L^2(\mu_n))$.
    \item For each $d$, there exists at least one minimizer of $\mcI_n[K] + \xi_n \mcI_n^{\text{reg}}[K]$ over $\mcZ_d^{\geq 0}(A)$, and there exists a unique minimizer to $\mcI_n[K] + \xi_n \mcI_n^{\text{reg}}[K]$ over $\mcZ^{\geq 0}$. 
    \item When Assumption~\ref{assume:reg}a) holds, the minima of $\mcI_n[K] + \xi_n \mcI_n^{\text{reg}}[K]$ over $\mcZ^{\geq 0}$ can be determined via a finite dimensional convex program; write $K_n^*$ for such a minima. Moreover, there exists some $r = r(n) \leq \kappa$ such that $K_n^*$ is of rank $r(n)$, and moreover as soon as $d \geq r(n)$ and $A \geq (\kappa - 1) \| K_n^* \|_{\infty}$, we have that the minima of $\mcI_n[K] + \xi_n \mcI_n^{\text{reg}}[K]$ over $\mcZ_d^{\geq 0}(A)$ is unique and equals $K_n^*$.
  \end{enumerate}
\end{prop}

\begin{proof}[Proof of Proposition~\ref{app:min_exist}]
  For i), this follows simply by using the fact that if $K\llp = \langle \eta(l), \eta(l') \rangle$ for some functions $\eta = (\eta_1, \ldots, \eta_d)$, then we have that 
  \begin{equation*}
    T_K[f](x) = \sum_{i=1}^d \eta_i(x) \int_0^1 \eta_i(y) f(y) \tilde{g}_n(y) \, dy = \sum_{i=1}^d (\eta_i) \otimes (\eta_i)^* 
  \end{equation*}
  and consequently as $\mathrm{Tr}[ \nu \otimes \nu^* ] = \nu^*(\nu)$ and the trace is linear, we have that
  \begin{equation*}
    \mathrm{Tr}[T_K] = \sum_{i=1}^d \int_{[0, 1]} \eta_i(y) \eta_i(y) \tilde{g}_n(y) \, dy =  \mcI_n^{\text{reg}}[K].
  \end{equation*}

  Part ii) follows as $Z_{fr}^{\geq 0}(A)$ is free of $A$ as a result of Lemma~52 \parencite{davison_asymptotics_2021}. 
  
  For iii), as $\mcI_n^{\text{reg}}[K]$ is simply the trace of the operator $T_K$, this will continuously extend to giving the trace on $\mcZ^{\geq 0}$, and more generally the positive trace-class operators on $L^2(\mu_n)$. This function is weak* lower semi-continuous as explained above. 
  
  To handle part iv), we note that if $\mcI_n^{\text{reg}}[K] < \infty$, then $T_K$ is trace-class, and consequently the operator $T_K$ is also Hilbert-Schmidt, implying that $K \in L^2(\mu_n^{\otimes 2})$. We note that we have
  \begin{equation}
  \begin{gathered}
    \label{app:exist:fun_facts}
        0 < M^{-1} \leq \tilde{f}_n(l, l', x) \leq M < \infty, \qquad 0 < M^{-1} \leq \tilde{g}_n(l) \leq M < \infty, \\
        | \ell(y, x) - \ell(y', x) | \leq \losslipconst \max\{ |y|, |y'| \}^{q-1} |y - y'|
  \end{gathered}
\end{equation}
  for all $l, l' \in [0, 1]$, $y, y' \in \mathbb{R}$ and $x \in \{0, 1\}$ (where $q = 1$ for the cross-entropy loss, and $q = 2$ for the squared loss),
  for some constants $M, \losslipconst \in (0, \infty)$. It consequently therefore follows that for the cross-entropy loss we have that
  \begin{equation*}
    \mcI_n[K] \leq 2 M^3 \int_{[0, 1]^2} ( \log(2) + | K\llp | ) \, \tilde{g}_n(l) \tilde{g}_n(l') dl dl' \ll \| K \|_{L^1(\mu_n^{\otimes 2})} \leq  \| K \|_{L^2(\mu_n^{\otimes 2})} < \infty.
  \end{equation*}
  A similar argument holds for the squared loss function, after noting that
  $\ell(y, 0) = y^2$ and $\ell(y, 1) \leq 2(2+y^2)$ for all $y \in \mathbb{R}$.
  For the strict convexity, we note that this follows by the strict convexity of the loss functions $\ell(y, x)$, the positivity of the $\tilde{f}_{n}(l, l', x)$, and the fact that multiplying the $\ell(y, x)$ by $\fnone$ and $\fnzero$, integrating, and then adding the two inequalities, will preserve the strict convexity. 
  
  By using the properties stated above in \eqref{app:exist:fun_facts} we also can argue continuity of $\mcI_n[K]$, in that (recalling that $q = 1$ handles
  the cross-entropy loss, and $q =2$ handles the squared loss)
  \begin{align*}
    \big| \mcI_n[K] & - \mcI_n[L] \big|  \leq M \losslipconst \intsq \max\{ |K\llp|, |L\llp| \}^{p-1} | K\llp - L\llp | \, dl dl' \nonumber \\ 
    & \leq 2 M^3 \intsq \big( |K\llp| + |L\llp| \big)^{q-1} \big| K\llp - L\llp \big| \, \tilde{g}_n(l) \tilde{g}_n(l') \, dl dl' \\ 
    & \leq 2 M^3 \Big( \| K \|_{L^q(\mu_n^{\otimes 2})} + \| L \|_{L^q(\mu_n^{\otimes 2})} \Big)^{q-1} \| K - L \|_{L^q(\mu_n^{\otimes 2})} \nonumber \\
    & \leq 2 M^3\Big( \| K \|_{L^2(\mu_n^{\otimes 2})} + \| L \|_{L^2(\mu_n^{\otimes 2})} \Big)^{q-1} \| K - L \|_{L^2(\mu_n^{\otimes 2})} \\
    & = 2 M^3\Big( \| T_K \|_{\text{HS}} + \| T_L \|_{\text{HS}} \Big)^{q-1} \| T_K - T_L \|_{\text{HS}} \\
    & \leq 2 M^3\Big( \| T_K \|_{1} + \| T_L \|_{1} \Big)^{q-1} \| T_K - T_L \|_{1},
  \end{align*}
  which also gives us part v); this is obtained by using \eqref{app:exist:fun_facts} in the first line, the second by using the fact that $\tilde{g}_n$ is bounded below and that $\max\{|a|, |b| \} \leq |a| + |b|$; the third line by H\"{o}lder's inequality and the triangle inequality; the fourth line by Jensen's inequality; the fifth line by the identification between the $L^2$ norms of kernels and the Hilbert-Schmidt norm of their associated operators, and the last line by the fact that the trace norm upper bounds the Hilbert-Schmidt norm. In particular, $\mcI_n[K]$ is norm-continuous with respect to the norm of $L^2(\mu_n^{\otimes 2})$. This plus convexity implies that $\mcI_n[K]$ is weakly lower semi-continuous, in the sense of the weak topology on $L^2(\mu_n^{\otimes 2})$. The restriction of this topology to the trace-class operators is coarser than the weak* topology (by the definition of the weak topology), and therefore $\mcI_n[K]$ is also weak* lower semi-continuous, concluding the arguments for part iv). 
  
  For vi), this follows by using the above parts, the fact that the trace is linear over positive trace-class operators, and that the sum of convex and lower semi-continuous functions remain convex and lower semi-continuous respectively. 
  
  For vii), we first need to discuss some of the properties of the sets $\mcZ_d^{\geq 0}(A)$, $\mcZ_{\text{fr}}^{\geq 0}(A)$ and $\mcZ^{\geq 0}(A)$. We note that by the same argument in Proposition~47 of \parencite{davison_asymptotics_2021} that $\mcZ_d^{\geq 0}(A)$ is weak* closed, and that because of the facts a) $t \mcZ_d^{\geq 0}(A) \subset \mcZ_d^{\geq 0}(A)$ and b) $\mcZ_r^{\geq 0}(A) + \mcZ_s^{\geq 0}(A) = \mcZ_{r+s}^{\geq 0}(A)$, we can conclude that $\mcZ_{\text{fr}}^{\geq 0}(A) = \mcZ_{\text{fr}}^{\geq 0}$ - recall part ii) - is convex. As closures of convex sets are convex, it consequently follows that $\mcZ^{\geq 0}$ is convex and weak* closed. Noting that each of these sets contain $0$, any minimizer $K$ must satisfy
  \begin{equation*}
    \xi_n \mathrm{Tr}[T_K] \leq \mcI_n[K] + \xi_n \mcI_n^{\text{reg}}[K] \leq \mcI_n[0] + \xi_n \mcI_n^{\text{reg}}[0] = \mcI_n[0] \implies \mathrm{Tr}[T_K] \leq \xi_n^{-1} \mcI_n[0].
  \end{equation*}
  As the set $\mcB := \{ K : \mathrm{Tr}[T_K] \leq \xi_n^{-1} \mcI_n[0] \}$ is weak* compact, it therefore follows that when minimizing over $\mcZ_d^{\geq 0}(A)$ and $\mcZ^{\geq 0}$, it suffices to minimize over the weak* compact sets $\mcZ_d^{\geq 0}(A) \cap \mcB$ and $\mcZ^{\geq 0} \cap \mcB$ respectively, and so by Weierstrass' theorem a minimizer must exist. As $\mcI_n[K] + \xi_n \mcI_n^{\text{reg}}[K]$ is strictly convex and $\mcZ^{\geq 0}$ is convex, we therefore also know that the minimizer over this set is unique.

  To end with part viii), we highlight that in Appendix~\ref{sec:app:loss:min_kernel}, it is shown that when $\fnone$, $\fnzero$
  and $\tilde{g}_n(l)$ are piecewise constant, one can relate the minimization
  problem of minimizing $\mcI_n[K] + \xi_n \mcI_n^\text{reg}[K]$ over
  $\mcZ_d^{\geq 0}(A)$ to that of minimizing the function 
  \begin{equation*}
    \sum_{l, l' \in [\kappa ] }  p_n(l) p_n(l') \sum_x c_{f, n}(l, l', x) \ell( \langle \tilde{\omega}_l, \tilde{\omega}_{l'} \rangle, x ) + \sum_{l \in [\kappa] } p_n(l) c_{g, n}(l) \| \tilde{\omega}_l \|_2^2
  \end{equation*}
  over $\tilde{\omega}_l$ for $l \in [\kappa]$ with $ \| \tilde{\omega}_l \|_{\infty} \leq A$ for all $A$ (see Appendix~\ref{sec:app:loss:min_kernel} for a reminder of the relevant notation). In particular, in the case where we allow
  $d = \kappa$, and we relax the constraint on the $\tilde{\omega}_l$, if we write $\tilde{K}_{ll'} = \langle \tilde{\omega}_l, \tilde{\omega}_{l'} \rangle$, then we can write the above function as 
  \begin{equation*}
    \sum_{l, l' \in [\kappa(n) ] }  p_n(l) p_n(l') \sum_x c_{f, n}(l, l', x) \ell( \tilde{K}_{ll'}, x ) + \sum_{l \in [\kappa(n)] } p_n(l) c_{g, n}(l) \tilde{K}_{ll},
  \end{equation*}
  which is a strictly convex function in the matrix $K_{ll'}$, and consequently
  has a unique minimizer over the cone of positive semi-definite matrices; call this matrix $\tilde{K}_n^*$. Supposing that $\tilde{K}_n^*$ is of rank $r(n) \leq \kappa$ (as the matrix is $\kappa \times \kappa$ dimensional and the rank is trivially less than the matrix dimension), if we write $\tilde{K}_n^* = \sum_{i=1}^r(n) \mu_i \phi_i \phi_i^T$ for some eigenvalues $\mu_i > 0$ and orthonormal eigenvectors $\phi_i \in \mathbb{R}^{\kappa}$, then we can identify $K_n^*$ with $\tilde{K}_n^*$
  via letting $K_n^* = \sum_{i=1}^{r(n)} \mu_i \psi_i(l) \psi_i(l')$ where $\psi_i(l) = \phi_{ij}$ for $l \in A_j$. We now highlight that
  one trivially has that $\| \phi_i \|_{\infty} \leq 1$ for all $i$, 
  and moreover that as every row and column sum (ignoring the diagonal)
  is bounded above by $(\kappa - 1) \|\tilde{K}_n^*\|_{\infty}$, by the
  Gershgorin circle theorem the eigenvalues are bounded above by 
  $(\kappa - 1) \|\tilde{K}_n^*\|_{\infty}$ also. Consequently, 
  as soon as $d \geq r(n)$ and $A \geq (\kappa - 1) \|\tilde{K}_n^*\|_{\infty}$,
  $K_n^* \in \mcZ_d^{\geq 0}(A)$, and as a result we have that 
  the minima of $\mcI_n[K] + \xi_n \mcI_n^\text{reg}[K]$ over $ \mcZ_d^{\geq 0}(A)$ is unique and equals $K_n^*$.
\end{proof}

As the above theorem shows that $\mcI_n[K] + \xi_n \mcI_n^{\text{reg}}[K]$ is a strictly convex function, well defined for all symmetric kernels $K$ corresponding to positive, self-adjoint, trace class operators $L^2(\mu_n) \to L^2(\mu_n)$ via the identification $K \to T_K$ given in \eqref{app:min:k2op}, we briefly discuss here the corresponding KKT conditions for constrained minimization. 

\begin{prop}
  \label{app:kkt}
  Let $\mcC$ be a weak* closed set of positive, symmetric, trace class kernels. Then $L$ is the unique minima of $\mcI_n[K] + \xi_n \mcI_n^{\text{reg}}[K]$ over $\mcC$ if and only if there exists some $V \in \mcB(L^2(\mu_n))$ such that 
  \begin{equation*}
    \mathrm{Tr}(V T_{L} ) = \mcI_n^{\text{reg}}[K], \qquad \| V \|_{\text{op}} \leq 1, \qquad \mathrm{Tr}\big( ( T_{\nabla} + \xi_n V)( T_K - T_{L} ) \big) \geq 0 \text{ for all } K \in \mcC,
  \end{equation*}
  where we identify symmetric kernels $K \in L^2(\mu_n^{\otimes 2})$ with operators $L^2(\mu_n) \to L^2(\mu_n)$ as in \eqref{app:min:k2op}, and write $T_{\nabla}$ for the bounded operator $L^2(\mu_n) \to L^2(\mu_n)$ with kernel
  \begin{equation*}
    \nabla \mcI_n[K] = \sum_{x \in \{0, 1\} } \frac{ \tilde{f}_n(l, l', x) \ell'( K\llp, x)  }{ \tilde{g}_n(l) \tilde{g}_n(l') },
  \end{equation*}
  where $\ell'(y, x)$ is the derivative of $\ell(y, x)$ with respect to $y$.
\end{prop}

\begin{proof}[Proof of Proposition~\ref{app:kkt}]
  We begin by deriving the subgradient for both $\mcI_n[K]$ and $\mcI_n^{\text{reg}}[K]$, and then use the rules of subgradient calculus to obtain the KKT conditions. For $\mcI_n[K]$, note that we can write 
  \begin{equation}
    \mcI_n[K] := \intsq \sum_{x \in \{0, 1\} } \frac{ \tilde{f}_n(l, l', x) \ell( K\llp, x) }{ \tilde{g}_n(l) \tilde{g}_n(l') } \, \tilde{g}_n(l) \tilde{g}_n(l') \, dl dl'
  \end{equation}
  and so the subgradient (in terms of the operator) is a singleton, say $T_{\nabla}$, whose sole element is the operator with kernel given by the Fr\'{e}chet derivative of $\mcI_n[K]$
  \begin{equation}
    \nabla \mcI_n[K]\llp = \sum_{x \in \{0, 1\} } \frac{ \tilde{f}_n(l, l', x) \ell'( K\llp, x)  }{ \tilde{g}_n(l) \tilde{g}_n(l') }
  \end{equation}
  \parencite[e.g][Proposition~2.53]{barbu_convexity_2012}. As for $\mcI_n^{\text{reg}}[K]$, we recall that this equals $\mathrm{Tr}[T_K]$, i.e the trace norm of $T_K$, as $K$ is positive. Because the dual space of $S(L^2(\mu_n))$ is the space of bounded operators $L^2(\mu_n) \to L^2(\mu_n)$ equipped with norm $\| \cdot \|_{\text{op}}$, we have that 
  \begin{equation}
    \partial \mcI_n^{\text{reg}}[K] = \big\{ V \in \mcB(L^2(\mu_n)) \,:\, \mathrm{Tr}(V T_K) = \mcI_n^{\text{reg}}[K], \| V \|_{\text{op}} \leq 1 \big\}
  \end{equation}
  \parencite[e.g][Theorem~7.57]{aliprantis_infinite_2006}. Combining the two subgradients together says that $L$ is an optimizer to $\mcI_n[K] + \xi_n \mcI_n^{\text{reg}}[K]$ over $\mcC$ if and only if there exists some $V \in \mcB(L^2(\mu_n))$ such that 
  \begin{equation}
    \mathrm{Tr}(V T_{L} ) = \mcI_n^{\text{reg}}[K], \qquad \| V \|_{\text{op}} \leq 1, \qquad \mathrm{Tr}\big( ( T_{\nabla} + \xi_n V)( T_K - T_{L} ) \big) \geq 0 \text{ for all } K \in \mcC
  \end{equation}
  as stated.
\end{proof}

With this, we now state the full version of Theorem~\ref{thm:embed_conv}, complete with regularity conditions.  

\begin{theorem}
  \label{app:thm:embed_conv_full}
  Suppose that Assumptions~\ref{assume:slc}~and~\ref{assume:reg} hold and that $\xi_n = O(1)$. Write $\mcZ^{\geq 0} = \mathrm{cl}\big( \cup_{d \geq 1} \mcZ_d^{\geq 0}(A) \big)$ for the closure of the union of the $\mcZ_d^{\geq 0}(A)$ with respect to the weak* topology on the trace-class operators $L^2(\mu_n) \to L^2(\mu_n)$ as described in Proposition~\ref{app:min_exist}. For each $n$, let $K_n^*$ denote the unique minimizer to the optimization problem
  \begin{equation*}
    \min_{K \in \mcZ^{\geq 0}(A) } \big\{ \mcI_n[K] + \xi_n \mcI_n^{\text{reg}}[K] \big\},
  \end{equation*}
  and assume that the $K_n^*$ are uniformly bounded in $L^{\infty}([0, 1]^2)$.
  Moreover, suppose that either
  \begin{enumerate}[label=(\Roman*), leftmargin=*]
    \item on the same partition $\mcQ$ as given in Assumption~\ref{assume:reg}a), we have that $K_n^*$ is piecewise constant on $\mcQ \times \mcQ$;  
    \item the $K_n^*$ are all H\"{o}lder($[0,1]$, $\beta^*$, $L^*$) for some constants $\beta^*$ and $L^*$.
  \end{enumerate}
  Then there exists $A'$ (see Lemma~\ref{app:thm:eigenbound} and Lemma~\ref{app:thm:curvature}) such that whenever $A_1, A_2 \geq A'$, for any sequence of minimizers
  \begin{equation*}
    \whbmomega \in \argmin_{\bmomega \in ([-A_1,A_1]^d)^n} \big\{ \mcR_n(\bmomega) + \xi_n \mcR_n^{\text{reg} }(\bmomega) \big\} \text{ such that } \max_{i, j} \big| \langle \whomega_i, \whomega_j \rangle \big| \leq A_2
  \end{equation*}
    we have that under condition (II) that
  \begin{equation*}
    \frac{1}{n^2} \sum_{i, j \in [n]} \Big( \langle \whomega_i, \whomega_j \rangle - K_n^*(\lambda_i, \lambda_j) \Big)^2 = O_p\Big( r_n + d^{-\beta*} + \Big( \frac{ \log(n) }{n} \Big)^{\min\{ \beta, \beta^* \}/2 } \Big),
  \end{equation*}
  where $r_n$ is the relevant rate of convergence in Theorem~\ref{thm:loss_conv}. In particular, there exists a sequence of embedding dimensions $d = d(n)$ such that the above bound is $o_p(1)$. Under condition (I), the above
  rate of convergence can be improved as follows: there exists some
  constant $r \leq \kappa$ such that, as soon as $d \geq r$, we have that
  the above bound is of the order $O_p(r_n)$ only. In particular, as soon as $d \geq r$, the above bound is $o_p(1)$. 
\end{theorem}

\begin{remark}
  The conditions on $K_n^*$ are given in order to give explicit rates of convergence; in order to only argue that we obtain consistency of the bound given above, it suffices to have that the $K_n^*$ are equicontinuous for each $n$. Moreover, this is only necessary in order to relate the minimal values of the $\langle \whomega_i, \whomega_j \rangle$ directly to the values of $K_n^*(\lambda_i, \lambda_j)$; we can still obtain weaker notions of consistency (see e.g. \eqref{eq:weak_consistency}) if we do not impose any continuity requirements. With regards to the assumption that the infinity norm of the matrix $\langle \whomega_i, \whomega_j \rangle$ is bounded with
  $n$, this could be imposed as a constraint in Theorem~\ref{thm:loss_conv}
  to guarantee such a pair of minimizers; as highlighted in Remark~\ref{app:rmk:tighter_bounds}, this can lead to improved dependence on the
  dimension $d$. We highlight that as under the given assumptions on the
  $\fnone$ and $\fnzero$, the unconstrained minimizer when $\xi_n = 0$
  is uniformly bounded in $L^{\infty}([0, 1]^2)$, and so we do not consider
  these assumptions (both on $K_n^*$ and the gram matrix of the embedding vectors) to be restrictive.
\end{remark}

\begin{remark}
    We highlight that we usually expect $\beta = \beta^*$; for example, see Theorem~\ref{thm:calc_minima} for an example with the squared loss.
\end{remark}

\begin{remark}
    \label{app:rmk:rates_of_convergence}
    We briefly discuss the rates of convergence of the above estimator
    when in the dense regime and using the squared loss, as in this
    setting the bound we obtain naturally corresponds to the
    guarantees given in the graphon estimation literature. In particular,
    when $\fnone$, $\fnzero$ and $\tilde{g}_n(l)$ are constant (i.e, free of $l$), Theorem~\ref{thm:calc_minima} guarantees us that the minima
    of $\mcI_n[K] + \xi_n \mcI_n^{\text{reg}}[K]$ corresponds to a version
    of the original generating graphon $W$ whose singular values have been
    subject to a soft-thresholding operator, and we can take $\beta^* = \beta$ also.

    In such a scenario, we then note that if we also take Remark~\ref{app:rmk:tighter_bounds} into account, then the rate of convergence equals
    \begin{equation*}
        s_n + \Big( \frac{d}{n} \Big)^{1/2} + \Big(  \frac{ \log n}{ n^{2\beta/(1 + 2\beta)}} \Big)^{1/2} + d^{-\beta} + \Big( \frac{\log n}{n}  \Big)^{\beta / 2}.
    \end{equation*}
    By choosing the embedding dimension $d$ optimally so that $d = O(n^{1/(1+2\beta)})$, and noting that the $\big( \log n/ n^{2\beta/(1 + 2\beta)} \big)^{1/2}$ term is of a slower order than the $\big( \log n/n \big)^{\beta / 2}$ term, we end up with a rate of convergence 
    \begin{equation*}
        s_n + \Big(  \frac{ \log n}{ n^{2\beta/(1 + 2\beta)}} \Big)^{1/2}.
    \end{equation*}
    Up to logarithmic factors and the sampling term, this is a square root
    of the rate of convergence of the UVST procedure \parencite{xu_rates_2018},
    which is itself a square root of the minimax rates of estimation
    \parencite{gao_rate-optimal_2015}. We suspect that the difference 
    with the rates achieved in \parencite{xu_rates_2018} occurs due to our approach of looking at the rates of convergence
    between the empirical and population risks, rather than being able to
    work directly with the original objective at all times. It would be
    interesting to see whether the rates of convergence can be improved so that,
    up to the sampling term, we end up with the same rates of convergence as in \parencite{xu_rates_2018}. 
\end{remark}

\begin{proof}[Proof of Theorem~\ref{app:thm:embed_conv_full}]
  The idea of the proof is to associate a kernel $\widehat{K}$ to a minimizer $\whbmomega$ of $\mcR_n(\bmomega) + \xi_n \mcR_n^{\text{reg}}(\bmomega)$ over $([-A, A]^d)^n$, and then argue from the uniform convergence results developed in the proof of Theorem~\ref{thm:loss_conv} that this requires $\widehat{K}$ to be close to the minimizer of $\mcI_n[K] + \xi_n \mcI_n^{\text{reg}}[K]$. Consequently, we can then use the curvature of this function about its minima to derive consistency guarantees. 

  To associate a kernel $K$ to a collection of embedding vectors $\bmomega$, we begin by writing $\lambda_{n, (i)}$ for the associated order statistics of $\bm{\lambda}_n = (\lambda_1, \ldots, \lambda_n)$, and let $\pi_n$ be the mapping which sends $i$ to the rank of $\lambda_i$. We then define the sets
  \begin{equation*}
    A_{n, i} = \Big[ \frac{i - 1/2}{n+1}, \frac{i + 1/2}{n+1} \Big] \qquad \text{ for } i \in [n],
  \end{equation*}
  and define the sequence of functions 
  \begin{equation*}
    \widehat{K}_n(l, l') = \langle \widehat{\eta}(l), \widehat{\eta}(l') \rangle \qquad \text{ where } \qquad \widehat{\eta}(l) = \begin{cases} \whomega_i & \text{ if } l \in A_{n, \pi_n(i)}, \\ 
      0 & \text{ otherwise.} \end{cases}
  \end{equation*}
  for any sequence $\whbmomega$ of minimizers to $\mcR_n(\bmomega) + \xi_n \mcR_n^{\text{reg}}(\bmomega)$. The idea of the proof is to then focus on upper and lower bounding the quantity \begin{equation*} 
    \big\{ \mcI_n[\widehat{K}_n]+\xi_n \mcI_n^{\text{reg}}[\widehat{K}_n] \big\} - \big\{ \mcI_n[K_n^*]+\xi_n \mcI_n^{\text{reg}}[K_n^*] \big\},
  \end{equation*}
  where $K_n^*$ is the minimizer of $\mcI_n[K]+\xi_n \mcI_n^{\text{reg}}[K]$ over $\mcZ^{\geq 0}(A)$. 

  \textbf{Step 1: Bounding from above.} Begin by noting from the triangle inequality we have that 
  \begin{align*}
    \big\{ \mcI_n[\widehat{K}_n] & +\xi_n \mcI_n^{\text{reg}}[\widehat{K}_n] \big\} - \big\{ \mcI_n[K_n^*]+\xi_n \mcI_n^{\text{reg}}[K_n^*] \big\} \\
    & \leq \Big| \big\{ \mcI_n[\widehat{K}_n]+\xi_n \mcI_n^{\text{reg}}[\widehat{K}_n] \big\} - \big\{ \mcR_n(\whbmomega) +\xi_n \mcR_n^{\text{reg}}(\whbmomega) \big\} \Big| \tag{I} \\
    & + \Big| \big\{ \mcR_n[\whbmomega]+\xi_n \mcR_n^{\text{reg}}(\whbmomega) \big\} - \min_{K \in \mcZ_d^{\geq 0}(A) } \big\{ \mcI_n[K]+\xi_n \mcI_n^{\text{reg}}[K] \big\} \Big| \tag{II} \\
    & + \Big| \min_{K \in \mcZ_d^{\geq 0}(A)} \big\{ \mcI_n[K]+\xi_n \mcI_n^{\text{reg}}[K] \big\} - \min_{K \in \mcZ^{\geq 0}(A) } \big\{ \mcI_n[K]+\xi_n \mcI_n^{\text{reg}}[K] \big\} \Big|. \tag{III}
  \end{align*}
  We want to bound each of the terms (I), (II) and (III). By using Lemma~\ref{app:thm:embed_to_kernel}, Theorem~\ref{thm:loss_conv} and Lemma~\ref{app:thm:eigenbound} respectively, we end up being able to bound the above quantity by $O_p(q_n)$, where
  \begin{equation*}
    q_n = \begin{cases} r_n & \text{ if (I) holds} \\ 
        r_n + \big( \log(n)/n \big)^{\beta/2} + d^{-\beta^*} & \text{ if (II) holds.}
    \end{cases}
  \end{equation*}

  \textbf{Step 2: Bounding from below.} Let $q_n$ denote the upper bound on the rate of convergence of $\big\{ \mcI_n[\widehat{K}_n]+\xi_n \mcI_n^{\text{reg}}[\widehat{K}_n] \big\} - \big\{ \mcI_n[K_n^*]+\xi_n \mcI_n^{\text{reg}}[K_n^*] \big\}$ as developed above. Then by Lemma~\ref{app:thm:curvature}, we have that 
  \begin{equation*}
    \label{eq:weak_consistency}
    \intsq \Big( \widehat{K}_n\llp - K_n^*\llp \Big)^2 \, dl dl' = O_p(q_n).
  \end{equation*}
    If we then define the function 
  \begin{equation*}
    \overline{K}_n^*\llp = \begin{cases} K_n^*(\lambda_i, \lambda_j) & \text{ if } (l, l') \in A_{n, \pi_n(i) } \times A_{n, \pi_n(j) }, \\
      0 & \text{ otherwise} 
    \end{cases}
  \end{equation*}
  then by the same arguments as in the proof of Lemma~\ref{app:thm:embed_to_kernel} we get that 
  \begin{equation*}
    \intsq \big( \overline{K}_n^*\llp - K_n^*\llp \big)^2 \, dl dl' = \begin{cases} O_p(n^{-1/2}) & \text{ if (I) holds} \\ 
        O_p\big( (\log(n)/n)^{\beta^*/2} \big) & \text{ if (II) holds.}

    \end{cases}
  \end{equation*}
  Consequently, as a result of the triangle inequality we get that 
  \begin{align*}
    \frac{1}{(n+1)^2} & \sum_{i, j \in [n] } \big( K_n^*(\lambda_i, \lambda_j)  - \langle \whomega_i, \whomega_j \rangle \big)^2 \\ 
    & = \intsq \big( \overline{K}_n^*\llp - \widehat{K}_n\llp \big)^2 \, dl dl' = \begin{cases} 
        O_p(q_n) & \text{ if (I) holds } \\
        O_p\big( q_n + (\log(n)/n)^{\beta^*/2} \big) & \text{ if (II) holds,} \end{cases}
  \end{align*}
  giving the desired result.
\end{proof}
 
\subsection{Additional lemmata}

\begin{lemma}
  \label{app:thm:embed_to_kernel}
  Under the assumptions and notation of Theorem~\ref{app:thm:embed_conv_full}, we have that
  \begin{equation*}
    \Big| \big\{ \mcI_n[\widehat{K}_n]+\xi_n \mcI_n^{\text{reg}}[\widehat{K}_n] \big\} - \big\{ \mcR_n(\whbmomega) +\xi_n \mcR_n^{\text{reg}}(\whbmomega) \big\} \Big| = O_p\big( r_n + (\log(n) / n)^{\beta/2} \big),
  \end{equation*}
  where $r_n$ is the convergence rate in Theorem~\ref{thm:loss_conv} when
  condition (II) holds. When condition (I) holds, the rate of
  convergence can be improved to be simply $O_p(r_n)$. 
\end{lemma}

\begin{proof}[Proof of Lemma~\ref{app:thm:embed_to_kernel}]
    We begin by handling what occurs when condition (II) of Theorem~\ref{app:thm:embed_conv_full} holds, and then detail what changes
    when condition (I) holds instead.
  
  Begin by defining the quantities 
  \begin{gather*}
    \tilde{c}_{f, n}(i, j, x) := \frac{1}{ | A_{n, \pi_n(i)} | |A_{n, \pi_n(j)} |} \int_{ A_{n, \pi_n(i)} \times A_{n, \pi_n(j)} } \tilde{f}_n(l, l', x) \, dl dl', \\ \tilde{c}_{g,n}(l) = \frac{1}{ |A_{n, \pi_n(i)} | }  \int_{A_{n, \pi_n(i) }} \tilde{g}_n(l) \, dl,
  \end{gather*}
  and note that as 
  \begin{equation*}
    \max_{ i \in [n] } \Big| \lambda_{n, (i) } - \frac{i}{n+1} \Big| = O_p\Big( \Big( \frac{\log(2n) }{n} \Big)^{1/2} \Big)
  \end{equation*}
  \parencite[by e.g][Theorem~2.1]{marchal_sub-gaussianity_2017}, we get that 
  \begin{align*}
    \big| \tilde{c}_{f, n}&(i, j, x) - \tilde{f}_n(\lambda_i, \lambda_j, x) \big| \\
    & \leq \frac{1}{ |A_{n, \pi_n(i) } | |A_{n, \pi_n(j) } | } \int_{A_{n, \pi_n(i)} \times A_{n, \pi_n(j)} } \big| \tilde{f}_n(l, l', x) - \tilde{f}_n(\lambda_{n, ( \pi_n(i) ) }, \lambda_{n, ( \pi_n(j) ) }, x ) \big| \, dl dl'  \\ 
    & \leq L \sup_{ (l, l') \in A_{n, \pi_n(i)} \times A_{n, \pi_n(j)} } \| (l, l') - ( \lambda_{n, ( \pi_n(i)) }, \lambda_{n, (\pi_n(j)) } ) \|_2^{\beta} \\ 
    & \leq L 2^{\beta/2} \Big(  \frac{1}{2n} + \max_{i \in [n] } \big| \lambda_{n, (i) } - \frac{i}{n+1} \big| \Big)^{\beta} = O_p\Big( \Big( \frac{\log n}{n} \Big)^{\beta/2} \Big),
  \end{align*}
  uniformly for all $i, j$, and similarly 
  \begin{equation*}
    \big| \tilde{c}_n(i) - \tilde{g}_n(\lambda_i) \big| = O_p\Big( \Big( \frac{\log n}{n} \Big)^{\beta/2} \Big)
  \end{equation*}
  uniformly over $i$. Using the fact that $\widehat{K}_n$ is piecewise constant, we can then write 
  \begin{align*}
    \mcI_n&[\widehat{K}_n]+\xi_n \mcI_n^{\text{reg}}[\widehat{K}_n] \\
    & = \frac{1}{(n+1)^2} \sum_{(i, j) \in [n] } \sum_{x \in \{0, 1\} } \ell( \langle \whomega_i, \whomega_j \rangle, x) \tilde{c}_{f, n}(i, j, x) + \xi_n \sum_{i \in [n] } \| \whomega_i \|_2^2 \tilde{c}_{g, n}(i) + \frac{2 (n-1) c_{\ell}}{(n+1)^2},
  \end{align*}
  where $c_{\ell}$ is a constant which depends on the choice of the 
  loss function.
  Introducing the function (compare with $\mathbb{E}[ \widehat{\mcR}^{\mcP_n}_{n, (1) }(\bmomega) \,|\, \bm{\lambda}_n ]$ from \eqref{loss_thm:add_diag})
  \begin{equation*}
    \mathbb{E}[ \widehat{\mcR}_{n, (1) }(\bmomega) \,|\, \bm{\lambda}_n ] := \frac{1}{n^2} \sum_{i, j \in [n] } \sum_{x \in \{0, 1\} } \tilde{f}_{n}(\lambda_i, \lambda_j, x) \ell( \langle \omega_i, \omega_j \rangle, x),
  \end{equation*}
  it follows that
  \begingroup 
  \allowdisplaybreaks 
  \begin{align*}
    \big| \big\{ \mcI_n[\widehat{K}_n] &+\xi_n \mcI_n^{\text{reg}}[\widehat{K}_n] \big\} - \big\{ \mcR_n(\whbmomega) +\xi_n \mcR_n^{\text{reg}}(\whbmomega) \big\} \big| \\
    & \leq \Big| \frac{1}{(n+1)^2} \sum_{i, j \in [n] } \sum_{x \in \{0, 1\} } \ell(\langle \whomega_i, \whomega_j \rangle, x) \big\{ \tilde{c}_{f, n}(i, j, x) - \tilde{f}_n(\lambda_i, \lambda_j, x) \big\} \nonumber \\
    & \qquad \qquad + \frac{\xi_n}{n+1} \sum_{i \in [n]} \| \whomega_i\|_2^2 \big\{ \tilde{c}_{g, n}(l) - \tilde{g}_n(\lambda_i) \big\} \Big| \tag{A} \\ 
     & + \big| \big\{ \mathbb{E}[ \widehat{\mcR}_{n, (1) }(\whbmomega) \,|\, \bm{\lambda}_n ] + \xi_n \widehat{\mcR}_n^{\text{reg}}(\whbmomega) \big\} - \big\{ \mcR_n(\whbmomega) + \xi_n \mcR_n^{\text{reg}}(\whbmomega) \big\} \big| \tag{B} \\ 
     & + O(n^{-1}) \big\{ \mathbb{E}[ \widehat{\mcR}_{n, (1) }(\whbmomega) \,|\, \bm{\lambda}_n ] + \xi_n \widehat{\mcR}_n^{\text{reg}}(\whbmomega) \big\} \tag{C} + O(n^{-1}).
  \end{align*}
  \endgroup
  From the proof\footnote{We note that the step where the `diagonal term' of including/excluding the sums of $\ell(\langle \omega_i, \omega_j \rangle, x)$ can be carried out before or after the stepping approximation step.} of Theorem~\ref{thm:loss_conv}, we know that the (B) term is of the order $O_p(r_n)$, and consequently via the uniform convergence bounds developed throughout the proof, this will also imply that $\big\{ \mathbb{E}[ \widehat{\mcR}_{n, (1) }(\whbmomega) \,|\, \bm{\lambda}_n ] + \xi_n \widehat{\mcR}_n^{\text{reg}}(\whbmomega) \big\} = O_p(1)$ and consequently the term in (C) will be of the order $O_p(n^{-1})$. For (A), we begin by noting that (A) can be bounded via the triangle inequality and the observations above by 
  \begin{equation*}
    \text{(A)} \leq \Big( \frac{1}{n^2} \sum_{i, j \in [n] } \sum_{x \in \{0, 1\} } \ell( \langle \whomega_i, \whomega_j \rangle, x) + \frac{1}{n} \sum_{i \in [n] } \| \whomega_i \|_2^2 \Big) \cdot O_p\Big( \Big( \frac{\log n}{n} \Big)^{\beta/2} \Big).
  \end{equation*} 
  To conclude, we just need to argue that 
  \begin{equation*}
    \frac{1}{n^2} \sum_{i, j \in [n] } \sum_{x \in \{0, 1\} } \ell( \langle \whomega_i, \whomega_j \rangle, x) + \frac{1}{n} \sum_{i \in [n] } \| \whomega_i \|_2^2  = O_p(1).
  \end{equation*}
  To see this, we note that this simply follows by using the fact that $\big\{ \mathbb{E}[ \widehat{\mcR}_{n, (1) }(\whbmomega) \,|\, \bm{\lambda}_n ] + \xi_n \widehat{\mcR}_n^{\text{reg}}(\whbmomega) \big\} = O_p(1)$ (as argued above) and the fact that the $\fnone$, $\fnzero$ and $\tilde{g}_n$ are assumed to be uniformly bounded below by $M^{-1}$. 

  When condition $(I)$ holds instead, we need to change the style
  of argument. Note that when $\mcQ = (A_1, \ldots, A_{\kappa})$,
  if we define the sets
  \begin{gather*}
    N_{\lambda, n, k} := \{ j \,:\, \lambda_j \in A_k \}, \quad N_{A, n, k} = \{ j \,:\, A_{n, \pi_n(j)} \; A_k \} \\
    M_{n, k} = N_{\lambda, n, k} \cap N_{A, n, k}, \quad M_n = \bigcup_{k = 1}^{\kappa} M_{n, k},
  \end{gather*}
  then by Theorem~63 of \parencite{davison_asymptotics_2021}, we have that
  $|M_n| \geq n - O_p(\sqrt{n})$, $|M_n^c| \leq O_p(\sqrt{n})$. To make
  use of this, note that
  \begin{equation*}
    \big| \tilde{c}_{f, n} (i, j, x) - \tilde{f}_n(\lambda_i, \lambda_j, x) \big| = \begin{cases} 0 & \text{ if } i, j \in M_n \\ M & \text{ otherwise,} \end{cases}
  \end{equation*}
  and also that  
  \begin{equation*}
    \big| \tilde{c}_n(i) - \tilde{g}_n(\lambda_i) \big| = \begin{cases} 0 & \text{ if } i \in M_n \\ M & \text{ otherwise.} \end{cases}
  \end{equation*}
  Writing $c_{\ell, 2} = \max\{ \ell(A_2, 1), \ell(A_2, 0), \ell(-A_2, 0), \ell(-A_2, 1) \}$, the bound in (A) is replaced by
  \begin{align*}
    \text{(A)} & \leq M  \Big( c_{\ell, 2} \frac{ |M_n^c|^2 + 2|M_n||M_n^c|}{(n+1)^2 } + \frac{ \xi_n |M_n^c|}{n+1} \Big) = O_p( n^{-1/2}),
  \end{align*}
  and so the argument progresses through as before, except that we can drop 
  the $(\log n/n)^{\beta/2}$ term in the overall rate of convergence.
\end{proof}

\begin{lemma} 
  \label{app:thm:eigenbound}
  Under the assumptions and notation of Theorem~\ref{app:thm:embed_conv_full}, there exists $A'$ such that whenever $A \geq A'$, we have, under condition (II) of Theorem~\ref{app:thm:embed_conv_full}, that
  \begin{equation*}
    \sup_{n \geq 1} \Big| \min_{K \in \mcZ_d^{\geq 0}(A)} \big\{ \mcI_n[K]+\xi_n \mcI_n^{\text{reg}}[K] \big\} - \min_{K \in \mcZ^{\geq 0}} \big\{ \mcI_n[K]+\xi_n \mcI_n^{\text{reg}}[K] \big\} \Big| = O(d^{-\beta^*}).
  \end{equation*}
  When condition (I) holds instead, then there exists $r \leq \kappa$
  and $A' < \infty$ such that, as soon as $d \geq r$ and $A \geq A'$, we have that
  \begin{equation*}
    \sup_{n \geq 1} \Big| \min_{K \in \mcZ_d^{\geq 0}(A)} \big\{ \mcI_n[K]+\xi_n \mcI_n^{\text{reg}}[K] \big\} - \min_{K \in \mcZ^{\geq 0}} \big\{ \mcI_n[K]+\xi_n \mcI_n^{\text{reg}}[K] \big\} \Big| = 0.
  \end{equation*}
\end{lemma}

\begin{proof}[Proof of Lemma~\ref{app:thm:eigenbound}]
    We begin with the argument under condition (II) first,
    and then highlight how the details change when condition (I) holds instead.
    Note that by the spectral theorem for self-adjoint compact operators, we can write for each $n$ the eigendecomposition
  \begin{equation*}
    K_n^*(l, l') = \sum_{i=1}^{\infty} \lambda_i(K_n^*) \psi_{n,i}(l) \psi_{n,i}(l')
  \end{equation*}
  where the $\lambda_i(K_n^*)$ are non-negative, monotone decreasing in $i$ for each $n$, and satisfy the bound $\lambda_i(K_n^*) = O( d^{-(1 + \beta^*)} )$ \parencite{reade_eigenvalues_1983}, and are uniformly bounded above by $\| K_n^* \|_{L^2(\mu_n^{\otimes 2})} \leq M^2 \sup_{n \geq 1} \| K_n^* \|_{\infty}$. As for the eigenfunctions, we note that they are orthonormal in that $\langle \psi_{n, i}, \psi_{n, j} \rangle_{L^2(\mu_n)} = \delta_{ij}$. Moreover, as the image of the operator $T_{K_n^*}$ under the unit $L^2(\mu_n)$ ball lies within the class of H\"{o}lder($[0, 1]$, $\beta^*$, $L^*$) functions, the $\psi_{n,i}$ are each H\"{o}lder($[0, 1]$, $\beta^*$, $L^*$), and as they are uniformly bounded in $L^2(\mu_n)$, they will also be uniformly bounded (across $i$ and $n$) in $L^{\infty}([0, 1])$ too
  (see e.g. Lemma~\ref{app:other_results:holder_eigenfunction}). Consequently, writing 
  \begin{equation*}
    K_{n, d}^*(l, l') = \sum_{i=1}^d \lambda_i(K_n^*) \psi_{n, i}(l) \psi_{n, i}(l')
  \end{equation*}
  for the best rank $d$ approximation to $K_n^*$, it follows that $K_{n, d}^* \in \mcZ_d^{\geq 0}(A)$ for any $A \geq A' = M \sqrt{ \sup_{n \geq 1} \| K_n^*\|_{\infty} } \cdot \sup_{n, i} \| \psi_{n, i} \|_{\infty}$. As a result, we have that 
  \begin{align*}
    \min_{K \in \mcZ_d^{\geq 0}(A)} \big\{ \mcI_n[K]+\xi_n \mcI_n^{\text{reg}}[K] \big\} & - \min_{K \in \mcZ^{\geq 0}} \big\{ \mcI_n[K]+\xi_n \mcI_n^{\text{reg}}[K] \big\} \\
    & \leq \big\{ \mcI_n[K_{n,d}^*]+\xi_n \mcI_n^{\text{reg}}[K_{n,d}^*] \big\} - \big\{ \mcI_n[K_n^*]+\xi_n \mcI_n^{\text{reg}}[K_n^*] \big\}.
  \end{align*}
  In order to obtain the final bound, we then note that by the local-Lipschitz property derived in Proposition~\ref{app:min_exist}v), in addition to the fact that the trace is linear and equals the sum of the eigenvalues of the operator, we get that (where we use $\lesssim$ to hide unimportant constants)
  \begin{align*}
    \big\{ \mcI_n[K_{n,d}^*] & +\xi_n \mcI_n^{\text{reg}}[K_{n,d}^*] \big\} - \big\{ \mcI_n[K_n^*]+\xi_n \mcI_n^{\text{reg}}[K_n^*] \big\} \\ 
    & \lesssim (2 \|K_n^* \|_{L^2(\mu_n^{\otimes 2})} )^{q-1} \cdot \| K_{n, d}^* - K_n^* \|_{L^2(\mu_n^{\otimes 2})} + \xi_n \big| \mathrm{Tr}[ T_{K_n^*} - T_{K_{n, d}^* } ] \big| \\ 
    & = O\Big( \sup_{n \geq 1} \| K_n^* \|_{\infty}^{q-1} \Big( \sum_{i=d+1}^{\infty} d^{-2(1+\beta^*)} \Big)^{1/2} + \sum_{i=d+1}^{\infty} d^{-(1+\beta^*)} \Big) = O( d^{-\beta^*} ),
  \end{align*}
  as desired, noting that the bound on the RHS holds uniformly in $n$. 

  We highlight that in the case where condition (I) holds,
  we know by the last part of Proposition~\ref{app:min_exist} that, for each $n$, there exists $r(n) \leq \kappa$ such that once $d \geq r(n)$ 
  and $A \geq (\kappa - 1) \| K_n^* \|_{\infty}$, the minima of $\mcI_n[K]+\xi_n \mcI_n^{\text{reg}}[K]$ over 
  $\mcZ_d^{\geq 0}(A)$ equals the minima over the set $\mcZ^{\geq 0}$.
  Consequently, under the assumptions stated, letting $r = \tilde{r} = \sup_{n \geq 1} r(n) \leq \kappa$ and $A' = (\kappa - 1) \sup_{n \geq 1} \|K_n^* \|_{\infty}$ gives the stated result.
\end{proof}

\begin{lemma}
  \label{app:thm:curvature}
  When $\ell(y, x)$ is the cross-entropy loss, under the assumptions and notation of Theorem~\ref{app:thm:embed_conv_full}, for any $K \in \mcZ^{\geq 0}$ such that $\| K \|_{\infty} < \infty$, we have that 
  \begin{align*}
    \big\{ \mcI_n[K]+\xi_n \mcI_n^{\text{reg}}[K] \big\}  - \big\{ \mcI_n[K_n^*]+\xi_n \mcI_n^{\text{reg}}[K_n^*] \big\}  \geq \frac{1}{2M} \intsq \big(  K_n^*\llp - K\llp \big)^2 \, dl dl'
  \end{align*}
  for some constant $M$ which depends on $C_M := \max\{ \| K \|_{\infty}, \sup_n \| K_n^* \|_{\infty} \} < \infty$; in particular one
  can take $M = (e^{C_M}/(1+e^{C_M})^2)^{-1}$. When $\ell(y, x)$ is the squared
  loss, one can relax the requirement that $\| K \|_{\infty} < \infty$,
  and can take $M = 1/2$ instead. 
\end{lemma}

\begin{proof}[Proof of Lemma~\ref{app:thm:curvature}]
    We give the details for the cross-entropy loss, as the argument
    for the squared loss is the same, except for that the requirement
    that $\| K \|_{\infty} < \infty$ can be dropped. To begin, we note
    that for any $y$ and $y'$ such that $|y|, |y'| \leq A$ for some constant
    $A$, we have that $\ell''(y, x) = e^{y}/(1 + e^y)^2 \leq e^{A}/(1 + e^A)^2 > 0$, and consequently $\ell(y, x)$ is strongly convex in $y$ on the domain
    $|y| \leq A$ for all $x \in \{0, 1\}$. As a result, we have the inequality
    \begin{equation*}
        \ell(y, x) \geq \ell(y', x) + (y - y')\ell'(y', x) +  \frac{ e^A}{2(1+e^A)^2} (y - y')^2
    \end{equation*}
    for $x \in \{0, 1\}$ and all $y, y'$ with $|y|, |y'| \leq A$. After multiplying the above inequality by the $\fnone$ and $\fnzero$ separately, adding the two inequalities together, and integrating, we obtain the inequality
    \begin{align*}
        \mcI_n[K] \geq  \mcI_n[K'] + \intsq \nabla \mcI_n[K'] &(K\llp - K'\llp) \, dl dl'  \\ 
        & + \frac{1}{2M} \intsq \Big( K\llp - K'\llp  \Big)^2 \, dl dl'
    \end{align*}
  for any $K, K' \in \mcZ^{\geq 0}(A)$ for which $\| K \|_{\infty}, \| K' \|_{\infty} < \infty$, where $M$ depends on the value of $C_M := \max\{ \| K \|_{\infty}, \| K '\|_{\infty} \}$; in particular, we have that $M = (e^{C_M}/(1+e^{C_M})^2)^{-1}$. Note that under our assumptions, the $K_n^*$ are uniformly bounded in $L^{\infty}([0, 1]^2)$, and consequently it follows that for any $K \in \mcZ^{\geq 0}(A)$ which is bounded in $L^{\infty}([0, 1]^2)$ that 
  \begingroup 
  \allowdisplaybreaks 
  \begin{align*}
    \big\{ &\mcI_n[K] +\xi_n \mcI_n^{\text{reg}}[K] \big\} - \big\{ \mcI_n[K_n^*]+\xi_n \mcI_n^{\text{reg}}[K_n^*] \big\} \\
    & \stackrel{(a)}{\geq}  \intsq \nabla \mcI_n[K_n^*] (K\llp - K_n^*\llp) \, dl dl'  + \frac{1}{2M} \intsq \big( K\llp - K_n^*\llp \big)^2 \, dl dl' \\
    & \qquad \qquad  + \xi_n \mcI_n^{\text{reg}}[K] - \xi_n \mcI_n^{\text{reg}}[K_n^*] \\ 
    & \stackrel{(b)}{\geq} \mathrm{Tr}( T_{\nabla}( T_K - T_{K_n^*})) + \xi_n \mathrm{Tr}(V^*(T_K - T_{K_n^*})) + \frac{1}{2M} \intsq  \big( K\llp - K_n^*\llp \big)^2 \, dl dl' \\ 
    & \stackrel{(c)}{\geq} \frac{1}{2M} \intsq \big(  K\llp - K_n^*\llp \big)^2 \, dl dl'.
  \end{align*}
  \endgroup
  To obtain this, in (a) we substituted in the bound on $ \mcI_n[K] - \mcI_n[K']$ stated above. In (b), we used the isometry between the trace inner product on operators and the corresponding inner product of the kernels, and the KKT conditions stating the existence of a bounded operator $V^*$ for which $\mathrm{Tr}(V^* T_{K_n^*}) = \mcI_n^{\text{reg}}[K_n^*]$ and $\| V^* \|_{\text{op}} \leq 1$; the latter property consequently implies that $\mcI_n^{\text{reg}}[K] \geq \mathrm{Tr}(V^* K)$ by the variational formulation of the trace. In (c), we then use the fact that $K_n^*$ is optimal provided that $\mathrm{Tr}( (T_{\nabla} + \xi_n V^*)( T_K - T_{K_n^*})) \geq 0$. 
\end{proof}
\section{Proof of additional theorems from Section~\ref{sec:theory:embed}}
\label{sec:app:extra}

In this section, we write $\mu_i(K)$ for either the $i$-th largest
eigenvalue of a symmetric matrix $K$, or the $i$-th largest eigenvalue
of a self-adjoint operator with
kernel $K$ (as introduced in the beginning of Appendix~\ref{sec:app:embed_conv}. We write $\sigma_i(K)$ for the corresponding singular values; recall that for a matrix $K \in \mathbb{R}^{n \times d}$, we have that $\sigma_r(K)^2 = \mu_r(KK^T)$ for any $r \leq \min\{n, d \}$, and that for a self-adjoint positive definite matrix or
operator $K$, we have that $\sigma_r(K) = \mu_r(K)$ for all $r$. 

Before proving Theorems~\ref{thm:procrustes}~and~\ref{thm:procrustes_2}, we require a brief lemma.

\begin{lemma}
    \label{app:other_results:kmat_lemma_1}
    Let $K: [0, 1]^2 \to \mathbb{R}$ be the kernel of
    a symmetric, positive operator
    which is either a) piecewise constant on a partition $\mcQ \times \mcQ$ where $\mcQ$ is a partition of $[0, 1]$ of size $\kappa$, or b) continuous.
    Suppose moreover that $K$ has rank exactly equal to $r$, where
    \begin{equation}
        K(x, y) = \sum_{i=1}^r \psi_i(x) \psi_i(y)
    \end{equation}
    for some non-zero, orthogonal functions $\phi_i: [0, 1] \to \mathbb{R}$
    which are piecewise continuous. Then if $\lambda_i$ are i.i.d $\mathrm{Unif}([0,1])$ and we define the random matrix $(K_{\lambda})_{ij} := K(\lambda_i, \lambda_j)$, then $K_{\lambda}$ is of rank $\leq r$, and with asymptotic
    probability $1$ as $n \to \infty$, has rank exactly equal to $r$.
\end{lemma}

\begin{proof}[Proof of Lemma~\ref{app:other_results:kmat_lemma_1}]
    Note that if we write $\Psi_{ij} = \psi_j(\lambda_i) \in \mathbb{R}^{n \times r}$, then as $K_{\lambda} = \Psi \Psi^T$, we know that the rank of $K_{\lambda}$ must be less than or equal to $r$. For the second part,
    we note that under the given conditions, we can apply Corollary~5.5 of \parencite{koltchinskii_random_2000} to the matrix $n^{-1} K_{\lambda}$;
    under a), the diagonal summability condition needed follows trivially, and under b), Mercer's
    theorem gives the diagonal summability condition needed, with
    the other conditions being satisfied as a result of $K$ being finite rank.
    Consequently we have that 
    \begin{gather*}
        \mu_r(n^{-1} K_{\lambda}) = \mu_r(K) + O_p(n^{-1/2}) > \frac{1}{2} \mu_r(K) \text{ with probability $\to 1$}.
    \end{gather*}
    In particular, with asymptotic probability $1$, $n^{-1} K_{\lambda}$ is
    of full rank, and therefore so is $K_{\lambda}$.
\end{proof}

\begin{proof}[Proof of Theorem~\ref{thm:procrustes}]
    To save on notation, we write $K$ for $K_n^*$, $K_{\lambda}$
    for the matrix $(K(\lambda_i, \lambda_j))_{ij}$, and $\phi_i(l)$ for
    the $\phi_{n,i}(l)$. We note 
    that Lemma~\ref{app:other_results:kmat_lemma_1} gives
    the guarantee that $K_{\lambda}$ is asymptotically of exact rank
    $r$. Writing $\Psi_{\lambda} \in \mathbb{R}^{n \times r}$ for the matrix $(\phi_j(\lambda_i))$ for $i \in [n]$ and $j \in [r]$, the same argument
    in Lemma~\ref{app:other_results:kmat_lemma_1}  guarantees that the singular value
    $\sigma_r(n^{-1/2} \Psi_{\lambda})^2 = \mu_r(n^{-1} K_{\lambda}) \geq \tfrac{1}{2} \mu_r(K) > 0$ with asymptotic probability $1$, and
    therefore we can work on an event where the $r$-th highest singular
    value of $n^{-1/2} \Psi_{\lambda}$ is uniformly bounded away from zero.
    
    With this, we can now apply Lemma~5.4 of \parencite{tu_low-rank_2016}, which states that for any matrices $U, V \in \mathbb{R}^{n \times r}$, we have that
    \begin{equation*}
        \min_{Q \in O(r)} \| U - V Q \|_F^2 \leq \frac{1}{2 (\sqrt{2} - 1) \sigma_r^2(V)} \| UU^T - VV^T \|_F^2,
    \end{equation*}
    where $\sigma_d(V)$ is the $d$-th largest singular value of the matrix $V$. We recall that $\sigma_r(V)^2 = \mu_r(VV^T)$. Applying this to $U = n^{-1/2} \bmomega$ and
    $V = n^{-1/2} \Psi_{\lambda}$, followed by the above remark, gives the
    desired result.
\end{proof}

\begin{proof}[Proof of Theorem~\ref{thm:procrustes_2}]
    For this, we begin by noting that as $\tilde{G}$ is defined to be
    a best rank $r$ approximation to the matrix $G$, $n^{-1} \tilde{G}$ is a
    best rank $r$ approximation to the matrix $n^{-1} G$, and consequently we have that
    \begin{equation}
        \label{eq:procrustes_2_1}
        n^{-2} \| \tilde{G} - G \|_F^2 = \| n^{-1} \tilde{G} - n^{-1} G \|_F^2 = \sum_{i=r+1}^d \mu_i( n^{-1} G)^2
    \end{equation}
    by the Eckart–Young–Mirsky theorem. To proceed, we then recall that as $G_{ij} = \langle \whomega_i, \whomega_j \rangle$, and moreover we have that $\mu_i(K_{\lambda}) = 0$ for $i \geq r + 1$ we have that
    \begin{equation}
        \label{eq:procrustes_2_2}
    \begin{aligned}
        \sum_{i=r+1}^d \mu_i( n^{-1} G)^2 & = \sum_{i=r+1}^d \big( \mu_i( n^{-1} G) - \mu_i( n^{-1} K_{\lambda}) \big)^2 \\ 
        & \leq \sum_{i=1}^d \big( \mu_i( n^{-1} G) - \mu_i( n^{-1} K_{\lambda}) \big)^2  \leq \| n^{-1} G - n^{-1} K_{\lambda} \|_F^2 = o_p(1)
    \end{aligned}
    \end{equation}
    where the last inequality follows by the Weilandt-Hoffman inequality \parencite{hoffman_variation_1953}, giving the first part of the theorem statement. The second part of the
    theorem statement then follows by applying the proof of Theorem~\ref{thm:procrustes} to the matrix $\tilde{\Omega}$, noting that
    \begin{equation*}
        n^{-2} \| \tilde{G} - K_{\lambda} \|_F^2 \leq n^{-2} \| \tilde{G} - G \|_F^2 + n^{-2} \| G - K_{\lambda} \|_F^2 \leq 2 n^{-2} \| G - K_{\lambda} \|_F^2 = o_p(1)
    \end{equation*}
    by the triangle inequality and by combining \eqref{eq:procrustes_2_1} and \eqref{eq:procrustes_2_2} together.
\end{proof}

Before proving Theorem~\ref{thm:calc_minima}, we require a lemma about
the eigenfunctions of an operator whose kernel is H\"{o}lder continuous.

\begin{lemma}
    \label{app:other_results:holder_eigenfunction}
    Suppose that $K: [0, 1]^2 \to \mathbb{R}$ is symmetric and H\"{o}lder$([0, 1]^2, \beta, L)$ continuous. Then the eigenfunctions of the associated
    operator $T_K$ are H\"{o}lder$([0, 1], \beta, L)$ continuous, 
    and moreover are uniformly bounded in $L^{\infty}([0, 1])$.
\end{lemma}

\begin{proof}[Proof of Lemma~\ref{app:other_results:holder_eigenfunction}]
    We begin by noting that for any function $f \in L^2([0, 1])$, we have
    that
    \begin{align*}
        \big| T_K[f](x) - T_K[f](y) \big| & \leq \int_0^1 | K(x, z) - K(y, z) | \cdot | f(z) | \, dx \\ 
        & \leq L \| f \|_1 | x - y |^{\beta} \leq L \| f \|_2 | x - y |^{\beta},
    \end{align*} 
    and therefore the image of the unit ball $\| f \|_2 = 1$ consists
    of H\"{o}lder$([0, 1], \beta, L)$ continuous functions; consequently,
    so are the eigenvectors. Moreover, we note that the image of such a ball
    gives functions which are uniformly bounded in $L^{\infty}([0, 1])$; indeed,
    writing $g = T_K[f]$, and picking any $x \in [0, 1]$, we have that
    \begin{equation*}
        | g(x) | \leq | g(x) - g(y) | + | g(y) | \text{ for all $y \in [0, 1]$}
    \end{equation*}
    and therefore by integrating against $y$ we end up with 
    \begin{align*}
        | g(x) |  \leq \int_0^1 | g(x) - g(y) | \, dy + \int_0^1 |g(y)| \, dy \leq L \int_0^1 |x - y |^{\beta} \, dy + \| g\|_1 \leq L + 1
    \end{align*}
    as $\| g\|_1 \leq \|g\|_2 = 1$, and $|x - y|^{\beta} \leq 1$ for all $x, y \in [0, 1]$. 
\end{proof}

\begin{proof}[Proof of Theorem~\ref{thm:calc_minima}]
    Without loss of generality, suppose that $c_1 = c_2 = 1$; otherwise, we can just rescale the regularization constant $\xi$ so that, up to constant scaling, the objective is the same as one with $c_1 = c_2 = 1$. Now,
    recall that by the spectral theorem for self-adjoint operators and
    Lemma~\ref{app:other_results:holder_eigenfunction}, we can write
    \begin{equation*}
        T_W[f] = \sum_{i=1}^{\infty} \mu_i(W) \langle f, \phi_i \rangle \phi_i \qquad \text{ and } \qquad W(l, l') = \sum_{i=1}^{\infty} \mu_i(W) \phi_i(l) \phi_i(l')
    \end{equation*}
    where the latter sum converges in $L^2$, the $\mu_i(W)$ are sorted in
    monotone decreasing absolute value, the $(\phi_i)_{i \geq 1}$ are
    orthonormal eigenfunctions which are H\"{o}lder($[0, 1], \beta, L$)
    and uniformly bounded in $L^\infty([0, 1])$. We now want to study
    the minimizer of the function
    \begin{equation*}
        \| T_W - T_L \|_{\mathrm{HS}}^2 + \xi \| T_L \|_1
    \end{equation*}
    over all positive kernels $L$, where we have phrased the problem
    entirely in terms of the associated operators. To do so, we begin
    by writing
    \begin{equation*}
        T_L = T_{L^{\parallel}} + T_{L^{\perp}} \qquad \text{ where } L^{\parallel}(x, y) = \sum_{n=1}^{\infty} \mu_i \phi_n(x) \phi_n(y), 
    \end{equation*}
    for some $\mu_i \geq 0$, where $L^{\perp}$ is symmetric, positive and orthogonal to $L^{\parallel}$ in that $L^{\perp}[\phi] = 0$ for any $\phi \in \mathrm{cl}\{ \mathrm{span}(\phi_1, \phi_2, \ldots) \}$. 
    We can then argue that any minimizer $L$ must have $L^{\perp} = 0$.
    Indeed, we have that by orthogonality of $T_{L^{\perp}}$ to both
    $T_W$ and $T_{L^{\parallel}}$, we get the decomposition
    \begin{equation*}
        \| T_W -  T_L \|_{\mathrm{HS}}^2 = \| T_W - T_{L^{\parallel}} \|_{\mathrm{HS}}^2 + \| T_{L^{\perp}} \|_{\mathrm{HS}}^2 
    \end{equation*}
    and so $ \| T_W - T_{L^{\parallel}} \|_{\text{HS}}^2 \leq \| T_W - T_L \|_{\text{HS}}^2$ with equality if and only if $T_{L^{\perp}} = 0$; and moreover $\| T_{L^{\parallel}} \|_1 \leq \| T_L \|_1$. As $T_{L^{\perp}} = 0$, we can then show that the objective function equals
    \begin{equation*}
        \sum_{i=1}^{\infty} ( \mu_i - \mu_i(W))^2 + \xi \sum_{i=1}^{\infty} \mu_i. 
    \end{equation*}
    To minimize this, we note that we can minimize each term in the sum
    over $\mu_i \geq 0$ by taking 
    $\widehat{\mu}_i = ( \mu_i(W) - \xi)_{+}$. In particular, as 
    the eigenvalues of $W$ decay as $O(i^{-(1/2 + \beta)})$ \parencite{reade_eigen-values_1983}, it follows that for 
    $i \geq N$ where $N = O(\xi^{-2/(2 + 2\beta)})$, we have 
    that $\widehat{\mu}_i = 0$. Consequently, we get that
    \begin{equation*}
        \widehat{L}(x, y) = \sum_{i=1}^N \widehat{\mu}_i \phi_i(x) \phi_i(y)
    \end{equation*}
    is the minimizing positive kernel. We now note that as the $\phi_i$
    are uniformly bounded in $L^{\infty}([0, 1])$, and the $\widehat{\mu}_i$
    are bounded above also, we can argue that $\widehat{L}$ will belong to
    some set $\mcZ_d^{\geq 0}(A)$ for some $A$ sufficiently large and 
    any $d \geq N$, and consequently $\widehat{L} \in \mcZ^{\geq 0}$ also.
    In particular, this means that $\widehat{L}$ is the minimizer of the
    objective function over the set $\mcZ^{\geq 0}$. Finally,
    we then note that as the eigenfunctions are H\"{o}lder and we
    have a finite sum of terms of the form $\widehat{\mu}_i \phi_i(x) \phi_i(y)$, this plus the boundedness of the eigenfunctions will imply that
    $\widehat{L}$ is H\"{o}lder of exponent $\beta$ also. 
\end{proof}
\section{Proof of results in Section~\ref{sec:samp}}
\label{sec:app:samp_proof}

In this section, given triangular arrays $(X_{ni})$ and $(Y_{ni})$ for $i \in I_n$, $n \geq 1$, we use the notation $X_{n, i} = (1 + O_p(r_n)) Y_{n, i}$ to be equivalent to saying that $\max_{i \in I_n} | X_{ni} / Y_{ni} - 1| = O_p(r_n)$. Before giving the proofs of Lemmas~\ref{sampling:psamp_formula},~\ref{sampling:unif_edge_uni_formula}~and~\ref{sampling:rw_uni_stat_formula}, we require the following result.

\begin{lemma}
    \label{app:formula:conc}
    Suppose that Assumption~2 holds. Let $g : [0, 1]^2 \to [0, \infty]$ be a bounded measurable function which is bounded away from zero, and define
    \begin{equation*}
        T_{n, i} = \frac{1}{n-1} \sum_{j \in [n] \setminus \{i \} } a_{ij} g(\lambda_j),  \qquad \text{ so } \qquad \mathbb{E}[ T_{n, i} \,|\, \lambda_i ] = \rho_n \int_0^1 W(\lambda_i, y) g(y) \, dy.
    \end{equation*}
    Then for all $t \geq 0$ we have that 
    \begin{equation*}
        \mathbb{P}\Big(  \Big| \frac{ T_{n, i} }{ \mathbb{E}[T_{n, i} \,|\, \lambda_i ] } - 1 \Big| \geq t \,|\, \lambda_i \Big) \leq 2\exp\Big( \frac{ - n \mathbb{E}[T_{n, i} \,|\, \lambda_i ]  t^2 }{ 8 \| g \|_{\infty} (1+t) } \Big)
    \end{equation*}
    and whence that $T_{n, i} = \mathbb{E}[T_{n, i} \,|, \lambda_i] (1 + O_p((\log n/n \rho_n )^{1/2} ) )$. Similarly, if we write 
    \begin{equation*}
        \widetilde{T}_{n, i} = \frac{1}{n-1} \sum_{j \in [n] \setminus \{i \} } (1 - a_{ij} ) g(\lambda_j), \text{ so } \mathbb{E}[ \widetilde{T}_{n, i} \,|\, \lambda_i ] = \int_0^1 (1 - \rho_n W(\lambda_i, y) ) g(y) \, dy, 
    \end{equation*}
    then $\widetilde{T}_{n, i} = \mathbb{E}[ \widetilde{T}_{n, i} \,|, \lambda_i] (1 + O_p((\log n/n)^{1/2} ) )$.
\end{lemma}

To prove this result we use the method of exchangeable pairs to derive a concentration inequality. Assuming that $(X, X')$ is an exchangeable pair of random variables, and $f$ is a measurable function with $\mathbb{E}[f(X)] = 0$, if we have an anti-symmetric function $F(X, X')$ satisfying 
\begin{equation*}
  \mathbb{E}[F(X, X') | X] = f(X), \qquad v(X) := \frac{1}{2} \mathbb{E}\big[ \big| \{ f(X) - f(X') \} F(X, X') \big| \,|\, X \big] \leq B f(X) + C,
\end{equation*}
for some constants $B, C \geq 0$, then we get the concentration inequality \parencite[Theorem~3.9]{chatterjee_concentration_2005}
\begin{equation}
  \label{eq:exchange}
  \mathbb{P}\big( |f(X)| \geq t) \leq 2 \exp\Big( -\frac{t^2}{2C + 2Bt} \Big).
\end{equation}

\begin{proof}[Proof of Lemma~\ref{app:formula:conc}]
    We begin by noting that as $g$ is assumed to be bounded away from zero, and by Assumption~2 we assume that $W$ is also bounded away from zero, there exists a constant $c > 0$ such that $\mathbb{E}[T_{n, i} | \lambda_i] \geq c \rho_n > 0$ for all $i \in [n], n \geq 1$. To derive the given bounds, we will use the method of exchangeable pairs, working conditional on the $\lambda_i$ at first in order to derive a concentration inequality. By then using the above lower bound on $\mathbb{E}[T_{n, i} | \lambda_i]$, we will be able to obtain a bound which holds unconditionally, and consequently get the claimed bound on $T_{n, i}$ holding uniformly across all the vertices. 

    To begin, let $\bm{A}_{n, i}$ denote the $i$-th row of the adjacency matrix $\bm{A}_n$, and $\bm{\lambda}_{n, -i} := (\lambda_j)_{ j \neq i}$. We construct an exchangeable pair $\big( (\bm{\lambda}_{n, -i}, \bm{A}_{n, i} ), (\tilde{\bm{\lambda}}_{n, -i}, \tilde{\bm{A}}_{n, i} ) \big)$ as follows: we select an index $J$ uniformly from $[n] \setminus \{i\}$, redraw $\tilde{\lambda}_{J} \sim U[0, 1]$ and $\tilde{a}_{iJ} \sim \mathrm{Bern}(W_n(\lambda_i, \tilde{\lambda}_{J} ) )$ but otherwise we keep the other entries of $\tilde{\bm{\lambda}}_n$ and $\tilde{\bm{A}}_{n,i}$ the same. With this, note that
    \begin{align*}
        \frac{1}{\mathbb{E}[T_{n, i} \,|\, \lambda_i] } \mathbb{E}\Big[ \sum_{j \in [n]\setminus\{i\} } & a_{ij} g(\lambda_j) - \sum_{j \in [n]\setminus\{i\} } \tilde{a}_{ij} g(\tilde{\lambda}_j) \,\Big|\, \lambda_i, \bm{\lambda}_{n, -i}, \bm{A}_{n, i} \Big]  = \frac{T_{n, i} }{ (n-1) \mathbb{E}[T_{n, i} \,|\, \lambda_i] } - 1,
    \end{align*}   
    and the associated variance term is of the form
    \begin{align*}
        v\big( \bm{\lambda}_{n, -i}, \bm{A}_{n, i} \big) & = \frac{1}{ (n-1) \mathbb{E}[T \,|\, \lambda_i]^2 } \mathbb{E}\Big[  \Big( \sum_{j \in [n] \setminus \{i\} }  \{ a_{ij} g(\lambda_j) - \tilde{a}_{ij} g( \tilde{\lambda}_j ) \} \Big)^2 \,\Big|\, \lambda_i, \bm{\lambda}_{n, -i}, \bm{A}_{n, i} \Big]  \\
        & = \frac{1}{(n-1)^2 \mathbb{E}[T \,|\, \lambda_i ]^2 } \sum_{j \in [n] \setminus \{ i \} } \mathbb{E}\Big[ \big\{ a_{ij} g(\lambda_j) - a_{ij}' g(\lambda_j') \big\}^2 \,\Big|\, \lambda_i, \bm{\lambda}_{n, -i}, \bm{A}_{n, i} \Big],
    \end{align*}
    where $(a'_{ij})_{ij}$ and $(\lambda_i')_{i \geq 1}$ are independent copies of $(a_{ij})_{ij}$ and $(\lambda_i)_{i \geq 1}$. To bound this last quantity, we write
    \begin{align*} 
        \mathbb{E}\Big[ \big\{ a_{ij} g(\lambda_j)  & - a_{ij}' g(\lambda_j') \big\}^2 \Big|\, \lambda_i, \bm{\lambda}_{n, -i}, \bm{A}_{n, i} \Big] \\
        & = \| g \|_{\infty}^2 \mathbb{E}\Big[ \Big\{ a_{ij} \frac{g(\lambda_j)}{ \|g \|_{\infty} } - a_{ij}' \frac{ g(\lambda_j') }{ \| g \|_{\infty}}\Big\}^2 \,\Big|\, \lambda_i, \bm{\lambda}_{n, -i}, \bm{A}_{n, i} \Big] \\
        & \leq 2 \| g \|_{\infty} \mathbb{E}\Big[ a_{ij} g(\lambda_j) + a_{ij}'  g(\lambda_j') \,\Big|\, \lambda_i, \bm{\lambda}_{n, -i}, \bm{A}_{n, i} \Big],
    \end{align*}
    where we used the fact the inequalities $(a- b)^2 \leq 2(a^2 + b^2) \leq 2(a+b)$ for $a, b \in [0, 1]$ to obtain the last line. It therefore follows that 
    \begin{equation*}
        v\big( \bm{\lambda}_{n, -i}, \bm{A}_{n, i} \big) \leq \frac{2 \|g \|_{\infty} }{ (n-1) \mathbb{E}[T_{n, i} | \lambda_i] } \Big( \frac{ T_{n, i} }{ (n-1) \mathbb{E}[T_{n, i} | \lambda_i ] } + 1 \Big)
    \end{equation*}
    from which we can apply the inequality stated in \eqref{eq:exchange} to get the stated concentration inequality. As $\mathbb{E}[T_{n, i}] | \lambda_i ] \geq c \rho_n$, we can conclude that  
    \begin{equation*}
        \mathbb{P}\Big( \Big| \frac{ T_{n, i} }{ \mathbb{E}[T_{n, i} \,|\, \lambda_i ] } - 1 \Big| \geq t \Big) \leq 2 \exp\Big( \frac{ -cn \rho_n t^2 }{8 \|g \|_{\infty} (1+t) } \Big)
    \end{equation*}
    for all $i \in [n]$, from which taking a union bound allows us to conclude that $T_{n,i} = \mathbb{E}[T_{n,i} \,|\, \lambda_i] (1 + O_p( (\log(n)/n \rho_n)^{1/2} ))$. The same style of argument gives the claimed result when $a_{ij} \to 1 - a_{ij}$, noting that in this case one can instead argue that $\mathbb{E}[\tilde{T}_{n, i} \,|\, \lambda_i] \geq c'$ for some constant $c' > 0$ for all $i \in [n]$, $n \geq 1$.
\end{proof}

\begin{proof}[Proof of Lemma~\ref{sampling:psamp_formula}]
    We note that a vertex $i$ is sampled with probability $k/n$, and any pair of vertices $(i, j)$ is sampled with probability $k(k-1)/n(n-1)$, so the claimed result follows immediately.
\end{proof}

\begin{proof}[Proof of Lemma~\ref{sampling:unif_edge_uni_formula}]
    The formulae for $f_n(l, l', 1)$ and $f_n(l, l', 0)$ are given in Proposition~72 of \parencite{davison_asymptotics_2021}. It remains to derive the formula for $\tilde{g}_n(\lambda_i)$. For convenience, we denote $\tilde{s}_n = (\log(n)/ n \rho_n)^{1/2}$. To continue, we note that in the proof of Proposition~72 of \parencite{davison_asymptotics_2021}, it is shown that
    \begin{gather*}
        \mathbb{P}\big( u \in \mcV(S_0(\mcG_n)) \,|\, \mcG_n \big) = \frac{2 k W(\lambda_u, \cdot) }{ \mcE_W n} \big( 1 + O_p\big( \tilde{s}_n \big) \big), \\
        \mathbb{P}\big( B( l, \mathrm{Ug}_{\alpha}( u \,|\, \mcG_n) \geq 1 \,|\, \mcG_n \big)  = \frac{ l W(\lambda_u, \cdot)^{\alpha} }{ n \mcE_W(\alpha) } \big( 1 + O_p(\tilde{s}_n)) \big), \\
        \mathbb{P}\big( u, v \in \mcV(S_0(\mcG_n)) \,|\, \mcG_n = \Big( \frac{2 k a_{uv} }{ n^2 \rho_n \mcE_W} + \frac{ 4k(k-1) W(\lambda_u, \cdot) W(\lambda_v, \cdot) }{ n^2 \mcE_W^2} \Big) \cdot \big( 1+ O_p(\tilde{s}_n) \big),
    \end{gather*}
    and as a particular consequence, it therefore follows that 
    \begin{equation*}
      \mathbb{P}\big( u \in \mcV(S_0(\mcG_n)) \,\big|\, v \in \mcV(S_0(\mcG_n)), \mcG_n \big) = \Big( \frac{ a_{uv} }{ n \rho_n W(\lambda_v, \cdot) } + \frac{ 2(k-1) W(\lambda_u, \cdot) }{ n \mcE_W} \Big) \cdot \big( 1 + O_p(\tilde{s}_n) \big).
    \end{equation*}
    To begin in finding the formula for $\tilde{g}_n(\lambda_i)$, we note that
    \begin{equation*}
        \mathbb{P}\big( u \in \mcV(S(\mcG_n)) \,|\, \mcG_n \big) = \mathbb{P}\big( \mcV(u \in S_0(\mcG_n)) \,|\, \mcG_n \big) + \mathbb{P}\big( u \in \mcV(S_{ns}(\mcG_n) \setminus S_0(\mcG_n)) \,|\, \mcG_n \big),
    \end{equation*}
    where the first term is given as above. The second term corresponds to the probability that the vertex arises only through the negative sampling process, and so 
    \begin{equation*}
      \mathbb{P}\big( u \in \mcV(S_{ns}(\mcG_n) \setminus S_0(\mcG_n)) \,|\, \mcG_n \big) = \mathbb{P}\Big( \bigcup_{v \in \mcV_n \setminus \{ u \} } A_v \,\Big|\, \mcG_n \Big) 
    \end{equation*}
    where $A_v = \{ v \in \mcV(S_0(\mcG_n)) , \text{$u$ selected via negative sampling from $v$} \}$. We then have that  
    \begin{align*}
        \Big| \mathbb{P}\Big( \bigcup_{v \in \mcV_n\setminus \{u\} } A_v \,\Big|\, \mcG_n \Big) & - \sum_{ v \in \mcV_n\setminus \{u\} } \mathbb{P}( A_v \,|\, \mcG_n ) \Big| \leq \frac{1}{2} \sum_{ \substack{v, v' \in \mcV_n\setminus \{u\} \\ v' \neq v} } \mathbb{P}\big( A_v \cap A_{v'} \,|\, \mcG_n \big) \\
        & \leq \sum_{v \in \mcV_n\setminus \{u\}} \mathbb{P}\big(A_v \,|\, \mcG_n \big) \cdot \max_{v' \in \mcV_n \setminus \{ u \} } \sum_{v \in \mcV_n \setminus \{v', u \} } \mathbb{P}\big( A_v \,|\, A_{v'}, \mcG_n \big).
    \end{align*}
    We begin by finding the asymptotic form of $\sum_{ v \in \mcV_n\setminus \{u\} } \mathbb{P}( A_v \,|\, \mcG_n )$, where we find that 
    \begin{align*} 
      \sum_{v \in \mcV_n\setminus \{u\}} \mathbb{P}\big( A_v \,|\, \mcG_n \big) & = \sum_{v \neq u} \mathbb{P}\big( v \in \mcV(S_0(\mcG_n)) \big) \mathbb{P}\big( B( l, \mathrm{Ug}_{\alpha}( u \,|\, \mcG_n) \geq 1 \,|\, \mcG_n \big) (1 - a_{uv} ) \\
      & = \big( 1 + O_p\big( \tilde{s}_n \big) \big)  \cdot \frac{ 2kl W(\lambda_u, \cdot)^{\alpha} }{ n \mcE_W(\alpha) \mcE_W } \cdot \frac{1}{n} \sum_{ v \in \mcV_n\setminus \{u\} } (1 - a_{uv} ) W(\lambda_v, \cdot) \\
      & = \big( 1 + O_p\big( \tilde{s}_n \big) \big)  \frac{ 2kl W(\lambda_u, \cdot)^{\alpha} }{ n \mcE_W(\alpha) \mcE_W } \cdot \int_0^1 (1 - \rho_n W(\lambda_u, y) ) W(y, \cdot) \, dy,
  \end{align*}
  where we have used the formulae quoted at the beginning of the proof and Lemma~\ref{app:formula:conc}. It remains to examine the term 
  \begin{equation*}
    \max_{v' \in \mcV_n \setminus \{ u \} } \sum_{v \in \mcV_n \setminus \{v', u \} } \mathbb{P}\big( A_v \,|\, A_{v'}, \mcG_n \big).
  \end{equation*}
  To do so, note that we can write 
    \begin{align*}
        \mathbb{P}\big( A_{v} \cap A_{v'} & \,|\, \mcG_n \big) \\
        & = \mathbb{P}\big( v, v' \in \mcV(S_0(\mcG_n)) \,|\, \mcG_n \big) \mathbb{P}\big( B( l, \mathrm{Ug}_{\alpha}( u \,|\, \mcG_n) ) \geq 1 \,|\, \mcG_n \big)^2 (1 - a_{uv} ) (1 - a_{uv'} ) 
    \end{align*}
  and so 
    \begin{align*}
        \mathbb{P}\big( A_v \,|\, & A_{v'}, \mcG_n \big) \\ 
        & = \mathbb{P}\big( v \in \mcV(S_0(\mcG_n)) \,|\,  v' \in \mcV(S_0(\mcG_n)), \mcG_n\big) \mathbb{P}\big( B( l, \mathrm{Ug}_{\alpha}( u \,|\, \mcG_n) ) \geq 1 \,|\, \mcG_n \big) (1 - a_{uv} ).
    \end{align*}
    It therefore follows that, using the results stated at the beginning of the proof, 
    \begin{align*}
      & \sum_{v \in \mcV_n \setminus \{v', u \} } \mathbb{P}\big( A_v \,|\, A_{v'}, \mcG_n \big) \\
      & \;= (1 + O_p(\tilde{s}_n)) \frac{ l W(\lambda_u, \cdot)^{\alpha} }{ n \mcE_W(\alpha) } \sum_{v \in \mcV_n \setminus \{v', u \} }  \mathbb{P}\big( v \in \mcV(S_0(\mcG_n)) \,|\,  v' \in \mcV(S_0(\mcG_n)), \mcG_n\big) (1 - a_{uv} ) \\
      & \;= (1 + O_p(\tilde{s}_n)) \frac{ l W(\lambda_u, \cdot)^{\alpha} }{ n^2 \mcE_W(\alpha) } \sum_{v \in \mcV_n \setminus \{v', u \} } (1 - a_{uv} ) \cdot \big( \frac{a_{v v'} }{ \rho_n W(\lambda_{v'}, \cdot) } + 2(k-1) \mcE_W^{-1} W(\lambda_v, \cdot) \big) \\ 
      & \;= O_p( n^{-1} ) 
    \end{align*}
    uniformly across all $v', u$, and therefore
    \begin{equation*}
        \max_{v' \in \mcV_n \setminus \{ u \} } \sum_{v \in \mcV_n \setminus \{v', u \} } \mathbb{P}\big( A_v \,|\, A_{v'}, \mcG_n \big) = O_p( n^{-1} ).
    \end{equation*}
    Combining all of the above together then gives that 
    \begin{align*}
      \mathbb{P}\big( u &\in \mcV(S(\mcG_n)) \,|\, \mcG_n \big) \\ 
      & = \frac{2k}{n \mcE_W }\Big( W(\lambda_u, \cdot) +  \frac{ l W(\lambda_u, \cdot)^{\alpha} }{ \mcE_W(\alpha)} \cdot \int_0^1 (1 - \rho_n W(\lambda_u, y) ) W(y, \cdot) \, dy \Big) \big( 1 + O_p\big( \tilde{s}_n \big) \big)  
    \end{align*}
    and so we get the stated formula for $g_n$ with $s_n = \tilde{s}_n$. 
\end{proof}

\begin{proof}[Proof of Lemma~\ref{sampling:rw_uni_stat_formula}]
  The formulae for $f_n(l, l', 1)$ and $f_n(l, l', 0)$ are given in Proposition~74 of \parencite{davison_asymptotics_2021}. We also note that
  within the proof of Proposition~74 of \parencite{davison_asymptotics_2021}, we have that 
  \begin{gather*}
    \mathbb{P}\big( u \in \mcV(S_0(\mcG_n)) \,|\, \mcG_n \big) = \frac{ k W(\lambda_u, \cdot) }{n \mcE_W } \big( 1 + O_p\big( \tilde{s}_n \big) \big), \\
    \mathbb{P}\big( \tilde{v}_i = u \,|\, \mcG_n \big)  = \frac{ W(\lambda_u, \cdot) }{ n \mcE_W } \big( 1 + O_p(\tilde{s}_n) \big), \\ 
    \mathbb{P}\big( u \text{ selected via negative sampling from } v \,|\, \mcG_n \big) = \frac{ l W(\lambda_u, \cdot)^{\alpha} (1 - a_{uv} ) }{ n \mcE_W(\alpha) } \big( 1 + O_p(\tilde{s}_n) \big),
  \end{gather*}
  where we again write $\tilde{s}_n = (\log(n)/n \rho_n)^{1/2}$. To derive the corresponding formula for $\tilde{g}_n(l)$, we begin by noting 
  \begin{equation*}
    \mathbb{P}\big( u \in \mcV(S(\mcG_n)) \,|\, \mcG_n \big) = \mathbb{P}\big( u \in \mcV(S_0(\mcG_n)) \,|\, \mcG_n \big) + \mathbb{P}\big( u \in \mcV(S_{ns}(\mcG_n) \setminus S_0(\mcG_n)) \,|\, \mcG_n \big).
  \end{equation*}
  The first term is given above, so we focus on the second. Denoting $A_i(u) = \{ \tilde{v}_i = u \}$ for $i \leq k + 1$ and $u \in \mcV_n$, and $B_i(v|u) = \{ v \text{ selected via negative sampling from } u \}$, we know that 
  \begin{equation*}
    \mathbb{P}\big( u \in \mcV(S(\mcG_n) \setminus S_0(\mcG_n)) \,|\, \mcG_n \big) = \mathbb{P}\Big( \bigcup_{i=1}^{k + 1} \bigcup_{v \in [n] \setminus \{ u \} } A_i(v) \cap B_i(u | v) \,|\, \mcG_n \Big).
  \end{equation*}
  Letting $C_i = \cup_{v \in [n] \setminus \{u\} } \{ A_i(v) \cap B_i(u | v) \}$, we note that
  \begin{equation*}
    \Big( \sum_{i=1}^{k+1} 1[C_i] \Big) - 1\Big[ \bigcup_{j=1}^{k+1} C_j \Big] = \sum_{i = 1}^k 1\Big[ C_i \cap \cup_{j > i} C_j \Big],
  \end{equation*}
  and moreover that as the $A_i(v)$ are disjoint across all $v \in \mcV_n$ for each $i$ fixed, we have that 
  \begin{equation*}
    \sum_{i=1}^{k+1} \mathbb{P}\Big( \bigcup_{v \in [n] \setminus \{ u \} } A_i(v) \cap B_i(u | v) \,|\, \mcG_n \Big) = \sum_{i=1}^{k+1} \sum_{v \in [n] \setminus \{ u \} } \mathbb{P}\Big(  A_i(v) \cap B_i(u | v) \,|\, \mcG_n \Big).
  \end{equation*}
  Combining the above two facts therefore gives 
  \begingroup 
  \allowdisplaybreaks
  \begin{align*}
    \Big| \mathbb{P}\Big( \bigcup_{i=1}^{k + 1} & \bigcup_{v \in [n] \setminus \{ u \} } A_i(v) \cap B_i(u | v) \,|\, \mcG_n \Big) - \sum_{i=1}^{k+1} \sum_{v \in [n] \setminus \{ u \} } \mathbb{P}\big( A_i(v) \cap B_i(u | v) \,|\, \mcG_n \big) \Big| \\
    & \leq \sum_{i=1}^{k} \mathbb{P}\Big(  \Big\{  \bigcup_{v \in [n] \setminus \{ u \} } A_i(v) \cap B_i(u | v) \Big\} \cap \Big\{ \bigcup_{j > i} \bigcup_{v' \in [n] \setminus \{ u \} } A_j(v') \cap B_j(u | v') \Big\} \,|\, \mcG_n \Big) \\ 
    & \leq \sum_{i=1}^{k} \sum_{j > i } \sum_{v, v' \in [n] \setminus \{ u \} } \mathbb{P}\big( A_j(v') \cap B_j(u | v') \cap A_i(v) \cap B_i(u | v) \,|\, \mcG_n).
  \end{align*}
  \endgroup
  To handle the intersection probabilities, we note that we can write (using the above formulae), for indices $i < j$ and $v, v'$, that
  \begingroup 
  \allowdisplaybreaks
  \begin{align*}
        \mathbb{P}\big( A_i(v) & \cap B_i(u | v) \cap A_j(v') \cap B_j( u | v') \,|\, \mcG_n \big) \\
        & =  \big( 1 + O_p\big( \tilde{s}_n \big) \big) \cdot \frac{ l^2 W(\lambda_u, \cdot)^{2\alpha} }{n^2 \mcE_W(\alpha)^2 } \mathbb{P}\big( A_j(v') \,|\, A_i(v) \mcG_n \big) \mathbb{P}\big( A_i(v) \,|\, \mcG_n \big)  (1 - a_{uv} )(1 - a_{uv'} ) \\ 
        & = \big( 1 + O_p\big( \tilde{s}_n \big) \big) \cdot \frac{ l^2 W(\lambda_u, \cdot)^{2\alpha} W(\lambda_v, \cdot) }{n^3 \mcE_W(\alpha)^2  \mcE_W } \mathbb{P}\big( A_j(v') \,|\, A_i(v), \mcG_n \big)  (1 - a_{uv} )(1 - a_{uv'} ).
  \end{align*}
  \endgroup
  Write $E_n$ for the number of edges in $\mcG_n$ and $\mathrm{deg}_n(u)$ for the degree of the vertex $u$ in $\mcG_n$; then by Proposition~61 of \parencite{davison_asymptotics_2021} we have that 
  \begin{equation*} 
    \max_{u \in [n] } \frac{1}{ \mathrm{deg}_n(u) } = O_p( (n \rho_n)^{-1} ), \qquad \mathrm{deg}_n(u) = n \rho_n W(\lambda_u, \cdot) (1 + O_p(\tilde{s}_n) ), 
  \end{equation*}
  and we note that by Assumption~\ref{assume:reg} that $W(\lambda_u, \cdot)$ is bounded below away from zero (and above by one) uniformly over all $\lambda_u \in [0, 1]$. To handle the $\mathbb{P}(A_j(v') \,|\, A_i(v), \, \mcG_n)$ term, we note that by stationarity of the random walk and the Markov property that, when $j > i + 1$
  \begingroup 
  \allowdisplaybreaks
  \begin{align*}
    \mathbb{P}( A_j(v') \,|\, A_i(v), \, \mcG_n) & = \mathbb{P}( \tilde{v}_{j-i+1} = v' \,|\, \tilde{v}_1 = v,  \, \mcG_n) \\ 
    & = \sum_{u \,:\, a_{uv} = 1 } \mathbb{P}( \tilde{v}_{j-i+1} = v' \,|\, \tilde{v}_2 = u, \mcG_n ) \mathbb{P}( \tilde{v}_2 = u \,|\, \tilde{v}_1 = v, \mcG_n ) \\ 
    & = \sum_{u \,:\, a_{uv} = 1 } \frac{2 E_n}{\mathrm{deg}_n(u) \mathrm{deg}_n(v) } \mathbb{P}( \tilde{v}_{j-i+1} = v' \,|\, \tilde{v}_2 = u, \mcG_n ) \mathbb{P}( \tilde{v}_2 = u \,|\, \mcG_n ) \\ 
    & \leq \sum_{u \in [n] } \frac{2 E_n}{\mathrm{deg}_n(u) \mathrm{deg}_n(v) } \mathbb{P}( \tilde{v}_{j-i+1} = v' \,|\, \tilde{v}_2 = u, \mcG_n ) \mathbb{P}( \tilde{v}_2 = u \,|\, \mcG_n ) \\
    & \leq \frac{ \mathrm{deg}_n(v') }{ \mathrm{deg}_n(v) } \cdot \max_{u \in [n] } \frac{1}{ \mathrm{deg}_n(u) } = O_p( (n \rho_n)^{-1} ),
  \end{align*}
  \endgroup
  and when $j = i+1$ we have 
  \begin{equation*}
    \mathbb{P}( A_{i+1}(v') \,|\, A_i(v), \, \mcG_n ) = \frac{ a_{v v'} }{ \mathrm{deg}_n(v) } = O_p( (n \rho_n)^{-1} ).
  \end{equation*}
  Consequently we have 
  \begin{align*}
    \mathbb{P}\big( A_i(v) \cap B_i(u | v) &\cap A_j(v') \cap B_j( u | v') \,|\, \mcG_n \big) \\ 
    & \leq \big(1 + O_p(\tilde{s}_n) \big) \frac{ l^2 W(\lambda_u, \cdot)^{2\alpha} W(\lambda_v, \cdot) }{n^4 \rho_n \mcE_W(\alpha)^2  \mcE_W }  (1 - a_{uv} )(1 - a_{uv'} )
  \end{align*}
  and therefore 
  \begingroup 
  \allowdisplaybreaks
  \begin{align*}
    \sum_{i=1}^{k} \sum_{j > i } &\sum_{v, v' \in [n] \setminus \{ u \} }  \mathbb{P}\big( A_j(v') \cap B_j(u | v') \cap A_i(v) \cap B_i(u | v) \,|\, \mcG_n)  \\
    & \leq (1 + O_p(\tilde{s}_n) ) \cdot \sum_{i = 1}^k \sum_{v \in [n] \setminus \{u \} } \frac{k l^2 W(\lambda_u, \cdot)^{2 \alpha} W(\lambda_v, \cdot)  }{ n^3 \rho_n \mcE_W(\alpha)^2 \mcE_W } (1 - a_{uv} ) \\ 
    & \leq (1 + O_p( \tilde{s}_n ) ) \cdot \frac{k l W(\lambda_u, \cdot)^{\alpha} }{n \rho_n \mcE_W(\alpha )} \cdot \frac{ k l W(\lambda_u, \cdot )^{\alpha} }{n \mcE_W(\alpha) \mcE_W } \cdot \frac{1}{n} \sum_{v \in [n] \setminus \{u \} } W(\lambda_v, \cdot) (1 - a_{uv} ) \\ 
    & =  (1 + O_p( \tilde{s}_n ) ) \cdot \frac{ k l W(\lambda_u, \cdot)^{\alpha} }{n \rho_n \mcE_W(\alpha )} \cdot \frac{ k l W(\lambda_u, \cdot )}{n \mcE_W(\alpha) \mcE_W } \cdot \int_0^1 (1 - \rho_n W(\lambda_u, y)) W(y, \cdot) \, dy,
  \end{align*}
  \endgroup
  where in the last line we have used Lemma~\ref{app:formula:conc}. We then note that as we have 
  \begin{align*}
    \sum_{i=1}^{k+1} &\sum_{v \in [n] \setminus \{ u \} } \mathbb{P}\big( A_i(v) \cap B_i(u | v) \,|\, \mcG_n \big) \\
    &  = \big( 1 + O_p\big( \tilde{s}_n(\gamma) \big) \big) \cdot \frac{ (k+1) l W(\lambda_u, \cdot)^{\alpha} }{ n \mcE_W(\alpha) \mcE_W } \cdot \int_0^1 (1 - \rho_n W(\lambda_u, y) ) W(y, \cdot) \, dy,
  \end{align*}
  by the formulae stated above and Lemma~\ref{app:formula:conc}, we can therefore conclude by combining the above bounds that 
  \begin{align*}
    \mathbb{P}\big( u &\in \mcV(S_{ns}(\mcG_n) \setminus S_0(\mcG_n)) \,|\, \mcG_n \big) \\ 
    & = \big( 1 + O_p( \tilde{s}_n ) \big) \cdot \frac{ (k+1)l W(\lambda_u, \cdot)^{\alpha} }{ \mcE_W(\alpha) \mcE_W } \int_0^1 (1 - \rho_n W(\lambda_u, y) ) W(y, \cdot) \, dy.
  \end{align*}
  Consequently $\mathbb{P}\big( u \in S(\mcG_n) \,|\, \mcG_n \big)$ equals
  \begin{equation*}
      \big( 1 + O_p\big( \tilde{s}_n(\gamma) \big) \big) \cdot \frac{1}{n} \Big\{ \frac{ k W(\lambda_u, \cdot) }{\mcE_W} +  \frac{ (k+1)l W(\lambda_u, \cdot)^{\alpha} }{ \mcE_W(\alpha) \mcE_W } \int_0^1 (1 - \rho_n W(\lambda_u, y) ) W(y, \cdot) \, dy \Big\}
  \end{equation*}
  as desired.
\end{proof}

\begin{proof}[Proof of Theorem~\ref{thm:example}]
    We begin by highlighting that for the given model, we have that
    \begin{equation}
        W(\lambda, \cdot) = \frac{1}{k} \big( p + (k-1) q \big) = \mcE_W, \qquad W(\lambda, \cdot)^{\alpha} = \mcE_W(\alpha)
    \end{equation}
    for $\lambda \in [0, 1]$, and therefore we have that
    \begin{align*}
        f_n(\lambda_i, \lambda_j, 1) & = \frac{2k}{\kappa^{-1}(p + (\kappa - 1)q)} = 2k c_1, \\
        f_n(\lambda_i, \lambda_j, 0) & = 2 l (k+1), \\ 
        \tilde{g}_n(\lambda_i) & = k + l (k+1) \cdot \big( 1 - \rho_n \kappa^{-1}(p + (\kappa - 1) q ) \big) = c_2. 
    \end{align*}
    Consequently, as a result of Proposition~\ref{app:min_exist} viii), we know that we can obtain the minimizing kernel $K_n^*$ which appears in the convergence theorem Theorem~\ref{thm:embed_conv} as follows: we obtain a matrix $\tilde{K} \in \mathbb{R}^{k \times k}$ obtained via minimizing the convex function
    \begin{align*}
        - \frac{1}{\kappa^2} & \sum_{i, j} \big\{ 2k c_1 \cdot (p \delta_{ij} + q (1 - \delta_{ij} )) \log \sigma( \tilde{K}_{ij}) + 2l (k+1) \log \sigma(-\tilde{K}_{ij}) \big\} + \xi c_2 \sum_{i=1}^{\kappa} \tilde{K}_{ii} \\ 
        & = - \frac{1}{\kappa^2} \sum_{i, j} \big\{ 2k c_1 \cdot (p \delta_{ij} + q (1 - \delta_{ij} )) \log \sigma( \tilde{K}_{ij}) + 2l (k+1) \log \sigma(-\tilde{K}_{ij}) \big\}  + \xi c_2 \| \tilde{K} \|_{*}
    \end{align*}
    over all positive semi-definite matrices $\tilde{K}$. The desired convergence then follows by applying Theorem~\ref{thm:embed_conv}.
\end{proof}
\section{Additional experimental details}
\label{sec:app:exps}

\begin{table}
    \centering
    \caption{Summary statistics of Cora, CiteSeer and PubMedDiabetes datasets.}
    \begin{tabular}{@{}ccccc@{}}
    \toprule
    Dataset & Nodes & Edges & Features & Classes \\ \midrule
    Cora & 2708 & 5429 & 1433 & 7 \\
    CiteSeer & 3312 & 4732 & 3703 & 6 \\
    PubMed & 19717 & 44338 & 500 & 3 \\ \bottomrule
    \end{tabular}
    \label{tab:datasets}
\end{table}

We now describe the hyperparameter and training details of each of the methods used in the experiments; for all the methods, we used the Stellargraph\footnote{We highlight that the Stellargraph package is licsensed under the Apache License 2.0.} implementation of the architecture \parencite{data61_stellargraph_2018}. The code used to run the experiments is available on GitHub\footnote{\url{https://github.com/aday651/embed-reg}}. Experiments were run on a cluster, using for each experiment 4 cores of a Intel Xeon Gold 6126 2.6 Ghz CPU, and a variable amount of RAM depending on the method and dataset used. In total, the experiments carried out used approximately 100k CPU hours in total (including preliminary experiments). Table~\ref{tab:datasets} gives a summary of the 
features of the Cora, CiteSeer and
PubMedDiabetes datasets used in the experiments. 

\textbf{node2vec} - We train node2vec with $p=q=1$ for 5 epochs, using 50 random walks of length 5 per node to form as subsamples, and train using batch sizes of 64. We use an Adam algorithm with learning rate $10^{-3}$. For the regularization, we use tf.regularizers.l2 with the specified regularization weight as the embeddings regularizer argument to the Embeddings layer used in the node2vec implementation. 

Unlike as reported in \cite{hamilton_inductive_2017}, we found that using the Adam algorithm with learning rate $10^{-3}$ lead to far better performance than stochastic gradient descent with rates of either this magnitude or those suggested in e.g. the experiments performed within the GraphSAGE paper (of $0.2$, $0.4$ or $0.8$). In our preliminary experiments, we generally found that varying the learning rates of the Adam method rarely lead to significant changes in performance and kept any observed trends relatively stable, and so we did not vary these significantly throughout.

\textbf{GraphSAGE} - For GraphSAGE, we used a two layer mean-pooling rule with neighbourhood sampling sizes of 10 and 5 respectively; we note that using 25 and 10 samples as suggested in \parencite{hamilton_inductive_2017} were computationally prohibitive for all the experiments we wished to carry out. Otherwise, we use the node2vec loss with 10 random walks of length 5 per node, use a batch size of 256 for training, and train for 10 epochs.

\textbf{GCN} - To train a GCN in an unsupervised fashion, we parameterize the embeddings in the usual node2vec loss through a two layer GCN with ReLU activations, with intermediate layer sizes 256 and 256. For the node2vec loss, we instead use 10 random walks of length 5 per node, use a batch size of 256 during training, and train the loss for 10 epochs. 

\textbf{DGI} - For DGI, we use the same parameters as specified in \parencite{velickovic_deep_2018}. We use 256 dimensional embeddings only; we found in our preliminary experiments the performance change in using a 512 dimensional embedding was negligible, and that on the PubMedDiabetes dataset, the memory usage required to learn a 256 dimensional embedding was substantial (above 32GB of RAM needed). Otherwise, we train a one dimensional GCN with ReLU activation using the DGI methodology, for 100 epochs with an early stopping rule with a tolerance of 20 epochs, a batch size of 256, and used Adam with learning rate $10^{-3}$.

\textbf{Classifier details} - Given the embeddings learned in an unsupervised fashion, there is then the need to build a classifier for both the node classification and link prediction tasks. To do so, we use logistic regression, namely the LogisticRegressionCV method from the scikit-learn Python package. The cross validation was set to use 5-folds and a `one vs rest' classification scheme. Otherwise we used the default settings, except for a larger tolerance of the number of iterations for the BFGS optimization scheme.

\end{document}